\documentclass[journal,twocolumn]{IEEEtran}
\usepackage{amssymb}
\usepackage{mathrsfs}
\usepackage{amsfonts}
\usepackage{amsmath}
\usepackage{graphicx}
\usepackage{multirow}
\usepackage{caption2}
\usepackage{float}
\usepackage{booktabs}
\usepackage{algorithm}
\usepackage{algorithmic}
\usepackage{subfigure}
\usepackage{psfrag}
\usepackage{txfonts}
\usepackage{latexsym,bm}
\usepackage{color}
\usepackage[table]{xcolor}
\usepackage{multicol}
\usepackage{epstopdf}
\usepackage{url}

\newtheorem{theorem}{Theorem}

\newtheorem{lemma}{Lemma}

\newtheorem{proposition}{Proposition}

\newtheorem{definition}{Definition}

\newenvironment{proof}[1][Proof]{\textbf{#1. }}{\ \rule{0.5em}{0.5em}}%

\ifCLASSINFOpdf
\else
\fi

\begin{document}

\title{Component-based  Sketching for Deep ReLU Nets}

\author{Di Wang, Shao-Bo Lin, Deyu Meng, Feilong Cao
\IEEEcompsocitemizethanks{\IEEEcompsocthanksitem  D. Wang and S. B. Lin are with the Center for Intelligent Decision-Making and Machine Learning, School of Management, Xi'an Jiaotong University. Deyu Meng is with the school of Mathematics and Statistics, Xi'an Jiaotong University. Feilong Cao is with the School of Science, China Jiliang University. (Corresponding author: Shao-Bo Lin, Email: sblin1983@gmail.com)}}

\IEEEcompsoctitleabstractindextext{%
\begin{abstract}
Deep learning has made profound impacts in the domains of data mining and AI, distinguished by the groundbreaking achievements in numerous real-world applications and the innovative algorithm design philosophy. However, it suffers from the inconsistency issue between optimization and generalization, as achieving good generalization, guided by the bias-variance trade-off principle, favors under-parameterized networks, whereas ensuring effective convergence of gradient-based algorithms demands over-parameterized networks. To address this issue, we develop a novel sketching scheme based on deep net components for various tasks. Specifically, we use deep net components with specific efficacy to build a sketching basis that embodies the advantages of deep networks. Subsequently, we transform deep net training into a linear empirical risk minimization problem based on the constructed basis, successfully avoiding the complicated convergence analysis of iterative algorithms.
The efficacy of the proposed component-based sketching is validated through both theoretical analysis and numerical experiments.
Theoretically, we show that the proposed component-based sketching provides almost optimal rates in approximating saturated functions for shallow nets and also achieves almost optimal generalization error bounds.
Numerically, we demonstrate that, compared with the existing gradient-based training methods,  component-based sketching possesses superior generalization performance with reduced training costs.
\end{abstract}


}

\maketitle

\IEEEdisplaynotcompsoctitleabstractindextext

\IEEEpeerreviewmaketitle
\section{Introduction}

Deep learning \cite{lecun2015deep} has made a profound impact in the fields of data mining and artificial intelligence, including not only the breakthrough in numerous applications such as image processing \cite{krizhevsky2017imagenet}, go games \cite{silver2016mastering} and speech recognition \cite{kamath2019deep}, but also the philosophy behind algorithm designs in the sense that labor-intensive and professional featured engineering techniques can be replaced by tuning parameters for certain deep neural networks (deep nets). Until now, we have witnessed the explosive developments of deep learning in the pursuit of new network structures \cite{zhou2020universality}, novel optimization algorithms \cite{kingma2014adam}, and further applications \cite{choudhary2022recent}.

Stimulated by the avid research activities in practical fields, understanding deep learning has become a recent theoretical focus. This includes demonstrating the power of depth in approximation theory, deriving optimal generalization error bounds in learning theory, and analyzing the convergence of iterative algorithms in optimization theory. From the function approximation (or representation) perspective, the advantages of deep nets over shallow nets have been demonstrated in providing localized approximation \cite{chui2020realization}, sparse approximation in the spatial domain \cite{lin2018generalization}, sparse approximation in the frequency domain \cite{schwab2019deep}, rotation-invariant approximation \cite{chui2019deep}, and manifold structure representation \cite{shaham2018provable}. From the learning theory perspective, optimal generalization errors of deep nets have been verified in learning hierarchical target functions \cite{kohler2016nonparametric},  non-smooth functions \cite{imaizumi2019deep}, composite functions \cite{schmidt2020nonparametric}, feature-group functions \cite{han2020depth}, and translation-equivalent functions \cite{han2023learning}, which are beyond the capability of shallow learning approaches such as shallow nets \cite{anthony1999neural} and kernel methods \cite{shawe2004kernel}. From the non-convex optimization perspective, the global convergence of stochastic gradient descent (SGD) \cite{allen2019convergence}, block coordinate descent (BCD) \cite{zeng2019global}, and alternating direction method of multipliers (ADMM) \cite{zeng2021admm} for deep nets has been verified. Additionally, the landscapes of deep nets with different activation functions have been investigated in \cite{sun2020global} to show the absence of local minima under certain restrictions on the width. The magic behind deep learning seems to be unlocked by unifying these interesting works. However, the problem is that the requirements for the structures, widths, and depths of deep nets vary significantly across different perspectives, implying that the existing results cannot be directly combined to explain the great success of deep learning.

\subsection{Tug of war between optimization and generalization}
\begin{figure}[!t]
\centering
\centering
\includegraphics*[width=8.5cm,height=4cm]{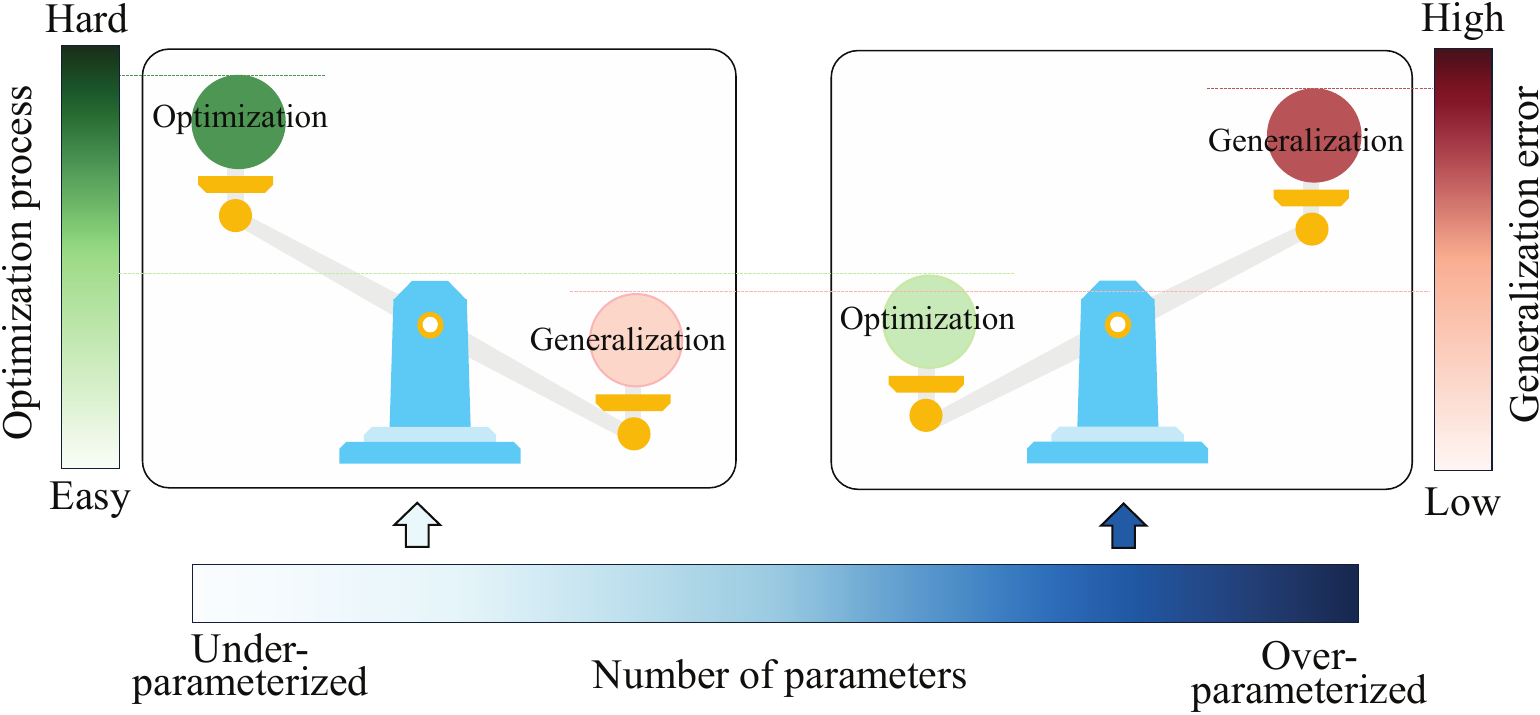}
\caption{Tug of war between optimization and generalization}
\label{Figure:tug-of-war}
\end{figure}

The tug of war between optimization and generalization, as shown in Figure \ref{Figure:tug-of-war}, yields a critical inconsistency in deep learning. This inconsistency arises because good generalization based on the bias-variance trade-off principle \cite{gyorfi2002distribution} requires the networks to be under-parameterized, while good convergence of gradient-type algorithms requires deep nets to be over-parameterized. Here and hereafter, an over-(or under-) parameterized deep net means that the number of free parameters in the network is more (or less) than the size of data. As discussed in \cite{li2022benefit}, the landscape of under-parameterized deep nets involves numerous global minima, saddle points, and even flat regions, and it remains an open question to design provable iterative algorithms that can equip the networks in \cite{kohler2016nonparametric,imaizumi2019deep,schmidt2020nonparametric,han2020depth,han2023learning} to fully realize their theoretical gains. The remedy is either to use over-parameterized deep nets to guarantee the existence of convergent algorithms or to avoid iterative algorithms directly.

Motivated by the nice numerical results in \cite{zhang2021understanding}, some theoretical efforts have been made to demonstrate the benign over-fitting phenomenon of over-parameterized deep nets. For example, \cite{cao2020generalization} succeeded in deriving an estimator with provable generalization error by applying gradient-based algorithms to over-parameterized deep nets under some slightly strict conditions on the data;  \cite{chen2020much} quantitatively characterized the degree of over-parameterization in training a deep net with provable generalization error under similar conditions as \cite{cao2020generalization};
\cite{lin2021generalization} proved that there exist perfect global minima possessing optimal generalization error bounds for implementing empirical risk minimization (ERM) on over-parameterized deep nets; \cite{zhou2024learning} derived a tight generalization bound for some interpolating deep convolutional neural networks. However, the conditions presented in \cite{cao2020generalization,chen2020much} are difficult to verify in practice, while the estimators derived in \cite{lin2021generalization,zhou2024learning} are challenging to obtain, indicating that understanding the generalization performance of over-parameterized deep nets is still in its infancy and the existing results fail to provide a perfect solution to the tug of war between optimization and generalization.


Due to the excellent approximation performances of constructive neural networks \cite{herrmann2022constructive},
constructing and sketching of (possibly shallow) neural networks are two popular schemes to replace iterative algorithms for under-parameterized deep nets. For instance, \cite{lin2018constructive} constructed a two-hidden-layer neural network to achieve the optimal generalization error of deep nets in learning smooth functions; \cite{liu2022construction} constructed a two-hidden-layer neural network to embody the advantages of deep nets in capturing spatial sparseness; \cite{fang2020learning} proposed a deterministic sketching strategy for training shallow nets; \cite{wang2020random} developed a random sketching scheme for shallow ReLU nets to realize the advantages of shallow nets in learning smooth functions. Despite these solid theoretical analysis, both constructing and sketching in the literature suffer from the well-known saturation phenomenon \cite{yarotsky2017error,petersen2018optimal} in the sense that the derived estimators cannot realize high-order smoothness, even when the smoothness of the target functions is known in practice. Taking \cite{wang2020random} as an example, the proposed random sketching scheme only succeeds in learning one-order smooth functions well, and its learning performance cannot be improved when higher-order smoothness is presented.
The main reason is the lack of constructing or sketching strategies for deep nets with sufficiently many layers to circumvent the aforementioned saturation phenomenon. In a nutshell, learning schemes based on both over-parameterized and under-parameterized deep nets are insufficient to settle the tug of war between optimization and generalization, making the full underlying machinery of deep learning still a mystery.

\subsection{Motivations and road-map}
Designing a feasible  sketching scheme for deep nets is much more difficult than
that for shallow nets. The difficulty lies in not only  embodying the provable advantages of deep nets over shallow nets, but also reducing the computational complexities of the classical optimization-based algorithms \cite{kingma2014adam,zeng2021admm}. Two crucial stepping stones for sketching, presented in this paper, are: 1) a component basis constructed based on locality-component \cite{chui2020realization} and product-component \cite{yarotsky2017error} for deep nets, and 2) a dimension leverage approach based on spherical minimum energy points \cite{petrushev1998approximation,wang2020random}. The work is motivated by the following three interesting observations.

$\bullet$ {\bf Neuron-based training versus component-based training:}
Neuron-based training that devotes to tuning parameters on each neuron is the most popular training scheme in deep learning.
This scheme dominates in providing different combinations of neurons to settle different learning tasks, but is not adept at explaining the running mechanisms of the training process. More importantly, neuron-based training frequently requires huge computations to determine all parameters of an over-parameterized model to guarantee the convergence \cite{li2022benefit}. In contrast, component-based training is frequently employed in theory to show the power of depth. As shown in Figure \ref{Figure:component-based-training}, component-based training starts with specific deep nets to represent some simple but important feature units such as the square function \cite{yarotsky2017error}, product function \cite{yarotsky2017error}, indicator function \cite{chui2020realization} and logarithmic function \cite{zhang2024classification},
and then finds suitable combinations or compositions of these deep net components to demonstrate the power of depth over shallow nets. Although such a component-based training scheme is rarely used in practice, almost all theoretical verifications in the literature \cite{schwab2019deep,han2020depth,liu2022construction,yarotsky2017error,petersen2018optimal,lin2022universal} are done according to it. These solid theoretical analysis tools inevitably imply a novel and promising training scheme based on deep net components  to realize  theoretical advantages of deep learning and provide interpretability, simultaneously.

\begin{figure}[!t]
\centering
\centering
\includegraphics*[width=8cm,height=3cm]{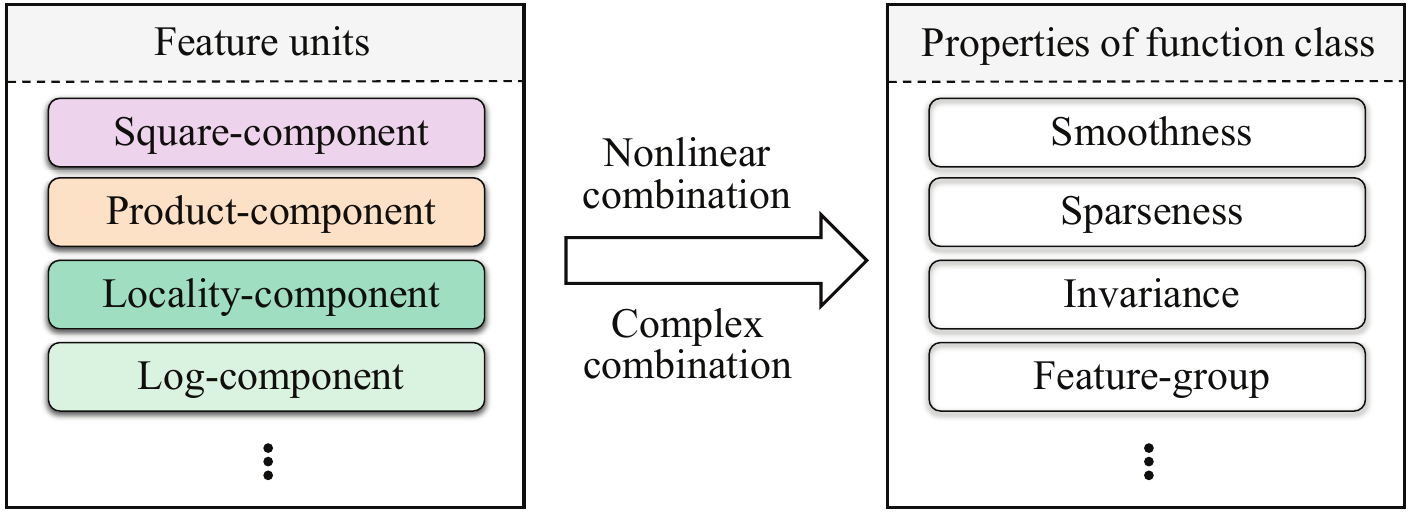}
\caption{Flow of component-based training schemes}
\label{Figure:component-based-training}
\end{figure}

\begin{figure}[!t]
\centering
\centering
\includegraphics*[width=8cm,height=3cm]{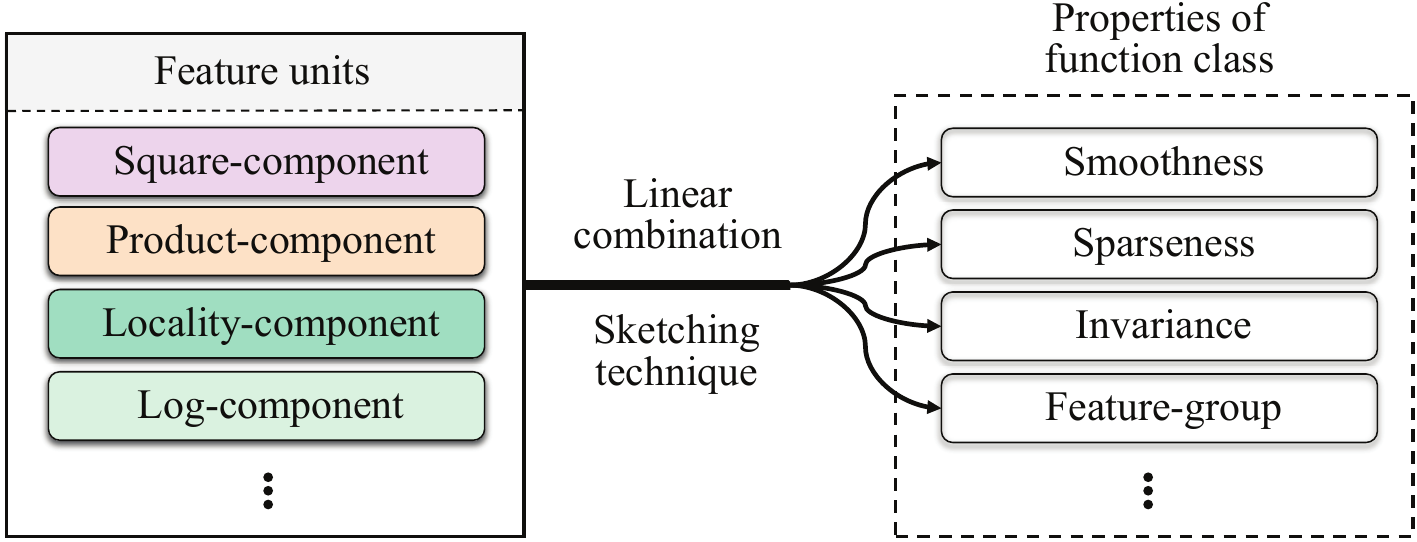}
\caption{Flow of component-based sketching schemes}
\label{Figure:component-based-sketching}
\end{figure}

$\bullet$ {\bf Component-based training versus component-based sketching:}
It was shown in \cite{schwab2019deep,han2020depth,liu2022construction,yarotsky2017error,petersen2018optimal,lin2022universal} that  component-based training is sufficient to embody the excellent performance of deep nets, but the training process still requires solving nonlinear non-convex optimization problems and therefore suffers from the optimization and generalization conflict depicted in Figure \ref{Figure:tug-of-war}. Recalling that the versatility and theoretical advantages of deep nets \cite{guo2019realizing} are built upon the a-priori information  exhibited in the specific learning tasks, we are able to formulate the  data features at first and then sketch suitable deep net components to find an appropriate combination to extract these features.
To be detailed,   figures are frequently associated with local similarity commonly referring to piece-wise smooth features, which can be well represented by combining the locality-component, square-component, and product-component \cite{petersen2018optimal}; gene data, which generally involves some manifold structures, possess  local manifold features that have been successfully derived by combining locality-component, product-component, and manifold-component \cite{shaham2018provable}; applications in computer vision concerning sparseness are related to spatially sparse features that can be tackled by combining locality-component and square-component \cite{chui2020realization}.   All these show that, once a learning task associated with some data features is given, it is not difficult to select suitable deep net components and find a combination of these components directly to settle the task. Such an approach, called
 component-based sketching, is presented in Figure \ref{Figure:component-based-sketching}. Comparing to component-based training, the pros of component-based sketching are low computational costs and high interpretability, while the cons are its functional singularity in the sense that different learning tasks require different sketching schemes.

$\bullet$ {\bf Construction of components in high dimension versus dimension-leverage:}
Training and construction are two feasible approaches to obtain deep net components. The construction approach presented in the literature \cite{yarotsky2017error,chui2020realization,zhang2024classification} is efficient in the sense that it succeeds in doing the same task with shallower and narrower networks. The main flaw of the construction, as shown in \cite{chui2020realization,yarotsky2017error,zhang2024classification}, is the curse of dimensionality. Taking the locality-component for example, the approach developed in \cite{chui2020realization} requires a cubic partition that is extremely difficult to obtain for high dimensional data. The training method, however, as shown in our Simulation 1 in Section 5.1 below, suffers from the optimization-generalization conflict,  resulting in the derived deep nets performing not so well, though it can overcome the curse of dimensionality.
The dimension leverage approach, based on ridge function representation and Kolmogorov-Arnold's extension theorem \cite{petrushev1998approximation,fang2020learning},  comes to our attention. The key idea of dimension leverage is to construct deep net components for univariate input at first, and then generate a set of minimal energy points as the weights of ridge functions. Taking the  locality-component for example again, the dimension leverage approach transforms the cubic partition for the construction approach into the cone-type partition, just as Figure \ref{Figure:cone-partition} purports to show.

\begin{figure}[t]
  \centering
  \begin{tabular}{m{3.5cm}<{\centering} m{0.5cm}<{\centering} m{3.9cm}<{\centering}}
  \includegraphics[width=3.5cm,height=3.4cm]{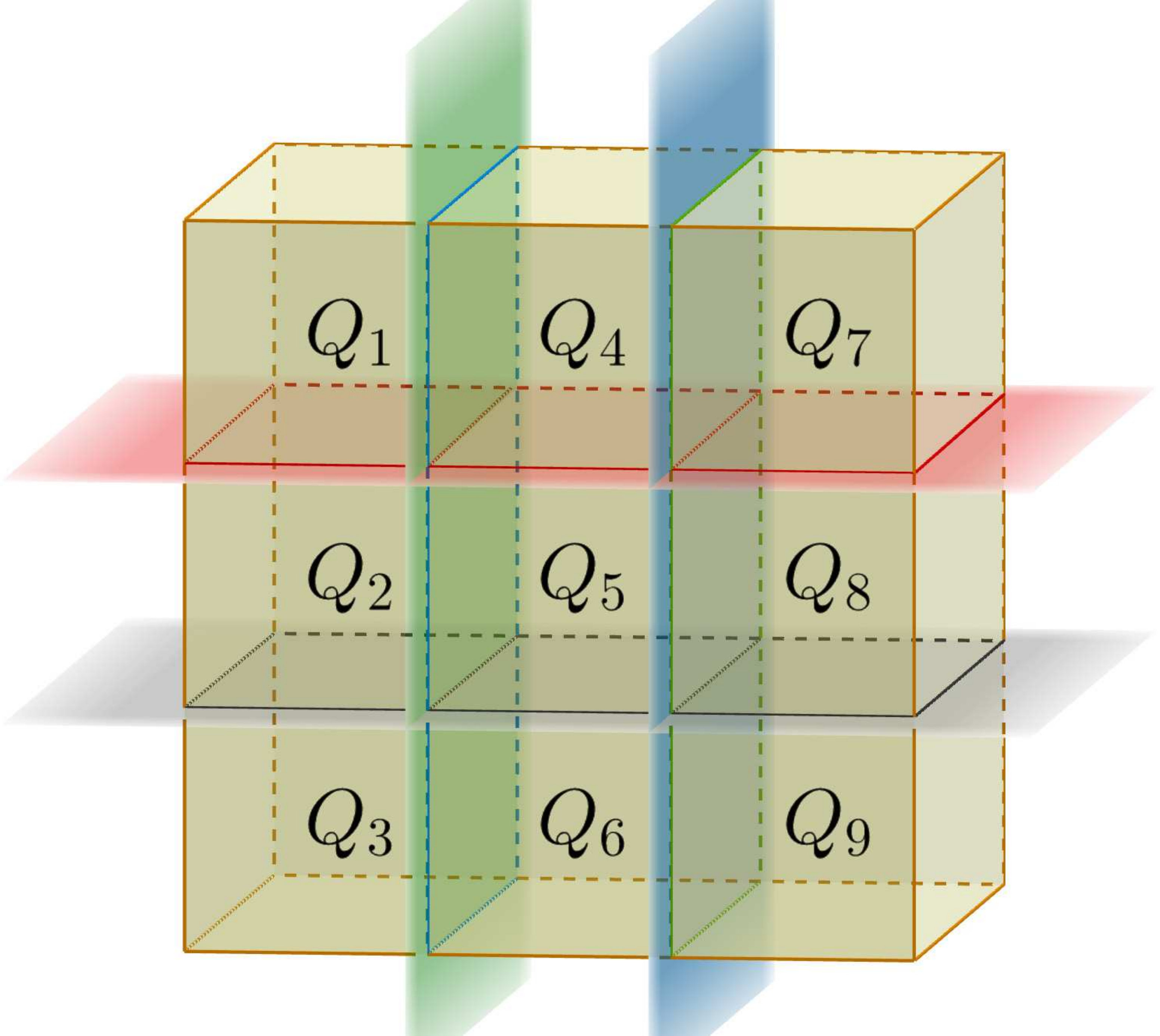}
  &\Large{$\Rightarrow$}
  &\includegraphics[width=3.9cm,height=3.8cm]{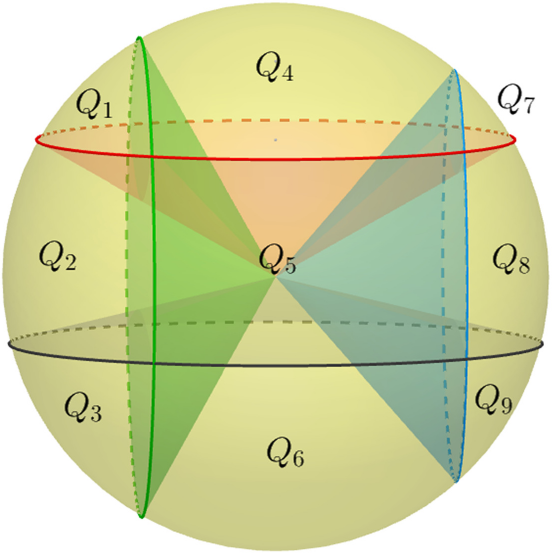}
  \end{tabular}
  \caption{Cubic locality versus cone-type locality}\label{Figure:cone-partition}
\end{figure}

As shown in Figures \ref{Figure:component-based-training} and \ref{Figure:component-based-sketching}, it is difficult to design a sketching scheme that is universal for different learning tasks. As the smoothness of high order is the key limitation of shallow nets \cite{yarotsky2017error,petersen2018optimal}, we aim to develop a feasible and efficient component-based sketching scheme to embody the high order smoothness and therefore circumvent the saturation phenomenon presented in the classical sketching or construction approaches in \cite{fang2020learning,wang2020random,liu2022construction}. Our road-map shown in Figure \ref{Figure:roadmap} contains mainly the following three steps.

\textbf{Step 1. Constructing univariate deep nets}:
Based on the delicate construction of the locality-component \cite{chui2020realization}, square-component \cite{yarotsky2017error}, and product-component \cite{han2020depth}, we construct a univariate deep net to extract smoothness features of different orders by introducing several parameters to control the locality, accuracy of product computation, and smoothness.

\textbf{Step 2. Dimension leverage}:
Motivated by the dimension leverage approach in \cite{petrushev1998approximation}, we
generate a set of  minimum energy points $\{\xi_\ell\}_{\ell=1}^N$ for some $N\in\mathbb N$ on the unit sphere and use $\xi_\ell \cdot x$  as the input of the constructed univariate deep nets and succeed in building the sketching basis.

\textbf{Step 3. ERM on sketching basis}:
Fed with the sketching basis, the nonlinear neural network training is transformed into linear problems, and
the classical empirical risk minimization (ERM) is  efficient for deriving the weights.
We then   determine the weights by solving a linear ERM.

\begin{figure}[t]
\centering
\centering
\includegraphics*[width=6cm,height=5.5cm]{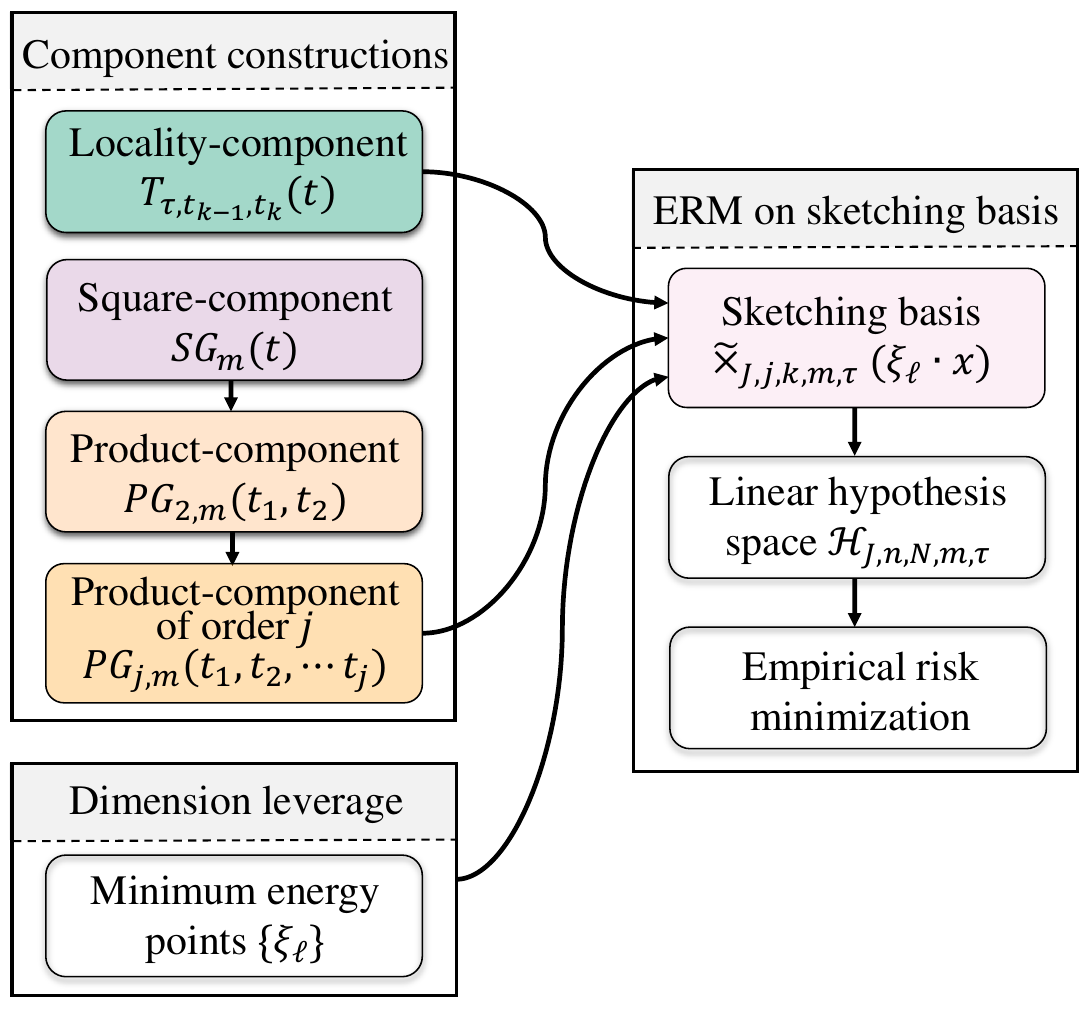}
\caption{Road-map of the component-based sketching schemes for deep nets}
\label{Figure:roadmap}
\end{figure}

 \subsection{Main contributions}


This paper proposes a novel component-based sketching scheme to equip deep nets based on different deep net components for specific tasks, and successfully circumvents the saturation phenomenon  exhibited in existing neuron-based construction approaches \cite{liu2022construction} and neuron-based sketching approaches \cite{wang2020random}. Additionally, it also avoids strict restrictions on data distributions for the training of neuron-based over-parameterized deep nets \cite{cao2020generalization}. The main contributions can be summarized as follows.

$\bullet$  Methodology novelty: We develop a novel sketching scheme based on deep net components for different tasks to obtain a deep net estimate. The basic idea is novel in two folds. On one hand, rather than operating on a single neuron as in classical neural networks training, our sketching scheme starts with deep net components that have specific efficacy and finds a delicate combination of them to build the sketching basis, thereby embodying the advantages of deep nets. On the other hand,  based on the constructed sketching basis, a nonlinear and non-convex deep net training problem is transformed into a linear ERM problem, which not only significantly reduces the computational burden but also avoids the complicated convergence analysis for existing iterative algorithms.

$\bullet$ Theoretical novelty: We conduct expressivity analysis and optimal generation error verifications for the proposed sketching scheme. For expressivity analysis, we rigorously prove that the span of the sketching basis is sufficient to avoid the saturation phenomenon. This implies that complicated nonlinear and non-convex training of deep nets is not necessary in the corresponding learning tasks.
For generalization error analysis, we show that the proposed component-based sketching scheme for deep nets achieves almost optimal generalization errors, provided the associated hyper-parameters are appropriately determined. Furthermore, we show that the well-known cross-validation approaches \cite{gyorfi2002distribution} succeed in determining these hyper-parameters, while providing the same almost optimal generalization error rates.

$\bullet$ Numerical novelty: We evaluate the proposed component-based sketching scheme against classical neuron-based training schemes for both shallow and deep nets,  and the random sketching scheme for shallow nets \cite{wang2020random},  in the context of toy simulations and real-world data experiments. The experimental results reveal that, in most cases, the proposed method achieves superior generalization performance, while its training time is comparable to, or even less than, that of shallow networks. All the numerical results validate our theoretical assertions and confirm that the component-based sketching technique is both feasible and efficient.

The rest of the paper is organized as follows. In the next section, we present the detailed construction of deep net components. In Section \ref{Sec.algorithm}, we introduce the component-based sketching scheme, as well as the construction of the sketching basis. In Section \ref{Sec.theory}, we study the theoretical behaviors of the component-based sketching scheme, including the expressivity of the sketching basis and almost optimal generalization error bounds. Section \ref{Sec.experiments} provides numerical experiments to show the
efficiency of the proposed method. In Section \ref{Sec.Proofs}, we present proofs of our theoretical assertions.

\section{Construction of Deep Net Components}
In this section, we aim to  construct  several deep net components to embody the advantage of deep nets.

\subsection{Power of depth of deep nets}
Let $L\in\mathbb N$ be the depth of a deep net, and $d_\ell \in \mathbb{N}$ be the width of the $\ell$-th hidden layer for $\ell=1,\dots,L$, with $d_0=d$. Denote the affine operator $\mathcal J_\ell:\mathbb R^{d_{\ell-1}}\rightarrow\mathbb R^{d_\ell}$ by $\mathcal J_\ell(x):=W_\ell  x+{\bf b}_\ell$, where $W_\ell$ is a $d_\ell\times d_{\ell-1}$ weight matrix and ${\bf b}_\ell\in\mathbb R^{d_\ell}$ is a bias vector. For the ReLU function $\sigma(t):=\max\{t,0\}$, write  $\sigma(x)=(\sigma(x^{(1)}),\dots,\sigma(x^{(d)}))^T$ for $x=(x^{(1)},\dots,x^{(d)})^T$.
Define  an $L$-layer deep net  by
\begin{equation}\label{deep-net}
     \mathcal N_{d_1,\dots,d_L}(x)
     = {\bf a}\cdot \sigma\circ \mathcal J_L \circ \sigma\circ \mathcal J_{L-1} \circ \dots \circ \sigma\circ\mathcal J_1(x),
\end{equation}
where ${\bf a} \in\mathbb R^{d_L}$. The structure of $ \mathcal N_{d_1,\dots,d_L}$ is determined by weight matrices $W_\ell$ and bias vectors ${\bf b}_\ell$, $\ell=1,\dots,L$. Denote by $\mathcal H_{L,n}$ the set of $\mathcal N_{d_1,\dots,d_L}$ with $L$ layers and $n$ free parameters.




The power of depth has been widely studied  \cite{schmidt2020nonparametric,han2020depth}, showing that the generalization error of deep learning is much smaller than that of the classical shallow learning in the framework of learning theory \cite{gyorfi2002distribution,cucker2007learning}, in which the samples in $D=\{(x_i,y_i)\}_{i=1}^{|D|}$ with $x_i\in\mathcal X$ and $y_i\in\mathcal Y$, are assumed to be drawn independently and identically according to an unknown but definite distribution $\rho=\rho_X\times\rho(y|x)$ with $\rho_X$   the marginal distribution and $\rho(y|x)$  the conditional distribution. The distribution $\rho_X$ corresponds to a  $\rho_X$ square integrable space $L_{\rho_X}^2$ endowed with the norm $\|\cdot\|_\rho$, while the distribution $\rho(y|x)$ is associated with the well-known regression function $f_{\rho }(x):=\int_{\mathcal Y}yd\rho (y|x)$ \cite{cucker2007learning} that minimizes the generalization error
$
             \mathcal{E}(f):=\int_{\mathcal Z}(f(x)-y)^{2}d\rho,
$
where $\mathcal Z:=\mathcal X\times\mathcal Y$. Since $\rho$ is unknown, $f_\rho$ cannot be achieved in practice, and the quality of any estimate  $f$ is measured by \cite{cucker2007learning}
\begin{equation}\label{equality}
          \mathcal{E}(f)-\mathcal{E}(f_{\rho })=\Vert f-f_{\rho }\Vert _{\rho
           }^{2}.
\end{equation}
Without loss of generality, we assume that the input space is $\mathcal X=\mathcal B^d_{1/2}$, the ball centered at the origin of radius $1/2$, and the output space $\mathcal Y\in [-M,M]$ for some $M>0$.

Denote by $\mathcal M(\Theta,\Lambda)$ the set of all distributions satisfying $\rho_X\in\Lambda$ and $f_\rho\in \Theta\subseteq L_{\rho_X}^2$, where $\Lambda$ is some set of marginal distributions and $\Theta$ is the set of regression functions. We enter into a competition over $\mathcal U_D$,  which denotes the class of all functions derived from the data set $D$, and define
\begin{eqnarray*}
          e(\Theta,\Lambda)
           := \sup_{\rho\in \mathcal M(\Theta,\Lambda)}\inf_{f_D\in \mathcal U_D}\mathbf E(\|f_\rho-f_{D}\|^2_\rho).
\end{eqnarray*}
Due to the definition, $e(\Theta,\Lambda)$ represents the optimal generalization performance that a learning scheme based on $D$ can achieve for $\mathcal M(\Theta,\Lambda)$,  and thus it provides a baseline to analyze the generalization performance of a learning scheme. For a  given estimate $f_D$, denote
$$
    \mathcal V_{\Theta,\Lambda}(f_D):=\sup_{\rho\in \mathcal M(\Theta,\Lambda)}\mathbf E(\|f_\rho-f_{D}\|^2_\rho)
$$
as its generalization error for $\mathcal M(\Theta,\Lambda)$. We then present the following definition of (almost) optimal estimators.
\begin{definition}\label{Def:optimal-algorithm}
Let  $\mathcal M(\Theta,\Lambda)$ be  the set of all distributions
satisfying  $\rho_X\in\Lambda$ and $f_\rho\in \Theta$, and $f_D$ be an estimate derived based on $D=\{(x_i,y_i)\}_{i=1}^{|D|}$ that are i.i.d. drawn according to $\rho=\rho_X\times\rho(y|x)$.
If
\begin{equation}\label{rate-optimal-algorithm}
     \mathcal V_{\Theta,\Lambda}(f_D)\sim e(\Theta,\Lambda),
\end{equation}
then $f_D$ is said to be a rate-optimal estimator (ROE) for $\mathcal M(\Theta,\Lambda)$.
Denote by $\mathcal Ropt(\Theta,\Lambda)$ the set of all ROE.
If
\begin{equation}\label{almost-optimal-algorithm}
    e(\Theta,\Lambda)\leq  \mathcal V_{\Theta,\Lambda}(f_D)\leq \tilde{c}e(\Theta,\Lambda)\log^v(|D|)
\end{equation}
for $\tilde{c},v>0$, then $f_D$ is said to be an almost rate-optimal estimator (AROE) for $\mathcal M(\Theta,\Lambda)$,
and we denote by $\mathcal Opt^{\tilde{c},v}(\Theta,\Lambda)$ the set of all AROEs with coefficient $\tilde{c}$ and exponent $v$.
\end{definition}


It should be highlighted that ROE or AROE is defined for specific learning tasks, i.e., $\mathcal M(\Theta,\Lambda)$, without which the optimality is meaningless. For instance, the random sketching scheme developed in \cite{wang2020random} serves as an AROE for one-order smooth regression functions but falls short of optimality for high-order smooth regression functions.
The excellent learning performance of deep learning has been verified in vast literature by
demonstrating that the derived estimators of deep learning are either ROE or AROE for numerous learning tasks.  Table \ref{Tab:existing} presents some typical results and shows that
deep learning succeeds in deriving AROE for different learning tasks, provided that the structure, depth, and width are appropriately determined. However,
since diverse learning tasks often necessitate distinct network structures,
the architecture of the network has recently become a key focus for neuron-based training of
deep learning \cite{zhou2020universality}, making network design labor-intensive and the training process time-consuming.
 \begin{table}[t]
    \begin{center}
	\begin{tabular}{|c|c|c|c|c|}
		\hline
		Ref. & Targets & Depth  & Structure &Quality \\
		\hline
		\cite{schmidt2020nonparametric} &Composite functions& $\log |D|$ &  fully connected & AROE\\
		\hline
		\cite{chui2019deep} &radial function &  $\mathcal O(d)$ & Tree structure &  AROE\\
		\hline
		\cite{chui2020realization} & sparseness & $\mathcal O(d)$ &    sparsely connected & AROE\\
		\hline
		\cite{han2020depth} & feature-group & $\mathcal O(d)$ &   sparsely connected & AROE \\
		\hline
		\cite{han2023learning} & translation-invariance & $Poly(|D|)$ & convolution &AROE   \\\hline
  \cite{lin2023optimal} & smoothness & $Poly(|D|)$ & convolution &AROE   \\
		\hline		
    \end{tabular}
    \end{center}
      \caption{Theoretical verifications of the power of depth in learning theory}\label{Tab:existing}
\end{table}

\subsection{Construction of  deep net components}

We call the deep net that is used to approximate the square function, product function, and indicator functions as the square-component, product-component, and locality-component, respectively. As the aim of this paper is to show the power of depth in learning smooth functions to circumvent the saturation phenomenon  \cite{liu2022construction,wang2020random}, we provide detailed constructions for the square-component, product-component, and locality-component. With the help of the  dimension leveraging scheme, we focus solely on deep net components for univariate inputs, although obtaining components for high-dimensional inputs is also straightforward \cite{chui2020realization,yarotsky2017error,zhang2024classification}.



We first construct a univariate locality recognizer using neural networks.
Let $\Lambda_{n}:=\{t_0,t_1,\dots,t_{n}\}$ with $t_k=-1/2+k/n$ for $k=0,1,\dots,n$.
Define a trapezoid-shaped function \cite{chui2020realization,wang2020random}
$T_{\tau,t_{k-1},t_k}$ with a parameter $0<\tau\leq 1$ by
\begin{eqnarray}\label{Localized-identifier}
    T_{\tau,t_{k-1},t_k}(t)
     &:= & \frac1\tau \{\sigma(t-t_{k-1}+\tau)-\sigma(t-t_{k-1})-\sigma(t-t_k) \nonumber\\
    &+ &\sigma(t-t_k-\tau) \}.
\end{eqnarray}
It is easy to check that the constructed network is of depth 1 and width 4. The following proposition directly follows from the definition of $\sigma$.

\begin{proposition}\label{Prop:localized}
Let $\sigma$ be the ReLU function, $t_k=-1/2+k/n$, and $T_{\tau,t_{k-1},t_k}$ be defined in \eqref{Localized-identifier}. Then,
\begin{equation}\label{Deteailed trapezoid}
    T_{\tau,t_{k-1},t_{k}}(t)=\left\{\begin{array}{cc}
   1,&\mbox{if}\ t_{k-1 }\leq t\leq t_{k},\\
   0,& \mbox{if}\ t\geq t_{k}+\tau,\ \mbox{or}\ t\leq t_{k-1}-\tau,\\
   \frac{t_{k}+\tau-t}{\tau}, &\mbox{if}\ t_{k}<t<t_{k}+\tau,\\
   \frac{t-t_{k-1}+\tau}{\tau}, & \mbox{if}\ t_{k-1}-\tau<t<t_{k-1}.
   \end{array}
   \right.
\end{equation}
\end{proposition}

Proposition \ref{Prop:localized} shows that, with sufficiently small $\tau$, the constructed network $T_{\tau,t_{k-1},t_k}$ approximates the indicator function $\mathcal I_{[t_{k-1},t_k]}$ on the interval $[t_{k-1},t_k]$, implying that $T_{\tau,t_{k-1},t_k}$ is capable of realizing the position information of the input.
As shown in \eqref{Deteailed trapezoid} and Figure \ref{Figure:localized-1}, the locality of $T_{\tau,t_{k-1},t_{k}}$ depends heavily on $\tau$ in the sense that smaller $\tau$ implies better locality.
 \begin{figure}[!t]
\centering
\centering
\includegraphics*[width=7cm,height=5cm]{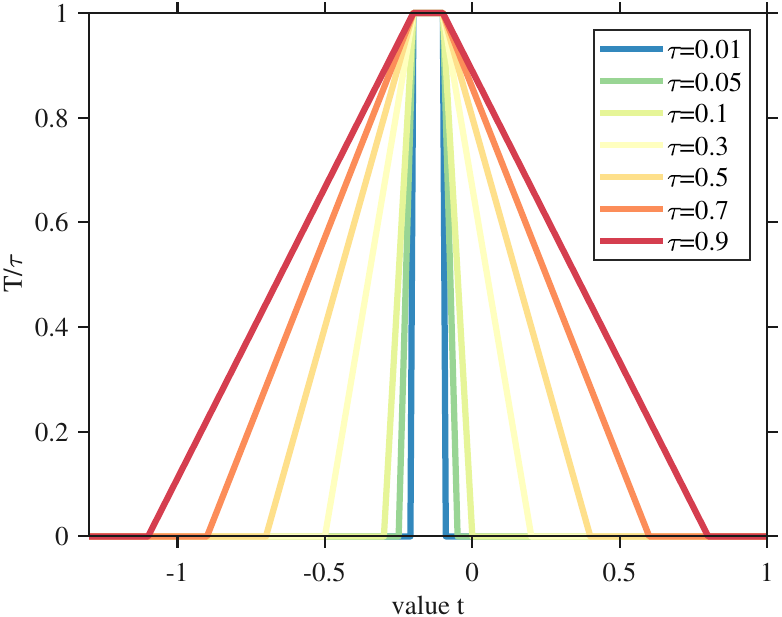}
\caption{Locality for $T_{\tau,t_k,t_{k+1}}$ with different $\tau$}
\label{Figure:localized-1}
\end{figure}

We then construct a deep net to act as the square-component. For $s\in\mathbb N$, define $g(t)=2\sigma(t)-4\sigma(t-1/2)$ and
$$
    g_s(t)=g\circ g\circ\dots\circ g(t).
$$
For $m\in\mathbb N$, define ${SG}_0(t)=t$ and
\begin{equation}\label{Def.square-component}
    {SG}_m(t)=t-\sum_{s=1}^m\frac{g_s(t)}{2^{2s}}.
\end{equation}
Recalling $t=\sigma(t)-\sigma(-t)$, it follows from the definition of $g_s$ that ${SG}_m$ is a deep net with $m$ layers and $\mathcal O(m^2)$ free parameters. The following proposition, whose proof is a slight modification of \cite[Lemma 1]{yarotsky2017error}, shows that ${SG}_m(t)$ can approximate $t^2$ well. Figure \ref{Figure:func_fm} exhibits the network architecture of the square-component.

\begin{proposition}\label{Prop:square-component}
For any $m\in\mathbb N$, $a,b\in\mathbb R$, let $SG_m$ be defined by \eqref{Def.square-component}, then
\begin{equation}\label{square-component-bound}
     |{SG}_m(t)-t^2|\leq c_12^{-2m},\qquad t\in[a,b],
\end{equation}
where $c_1$ is a constant depending only on $a,b$.
\end{proposition}


\begin{figure}[!t]
\centering
\centering
\includegraphics*[width=8cm,height=4.8cm]{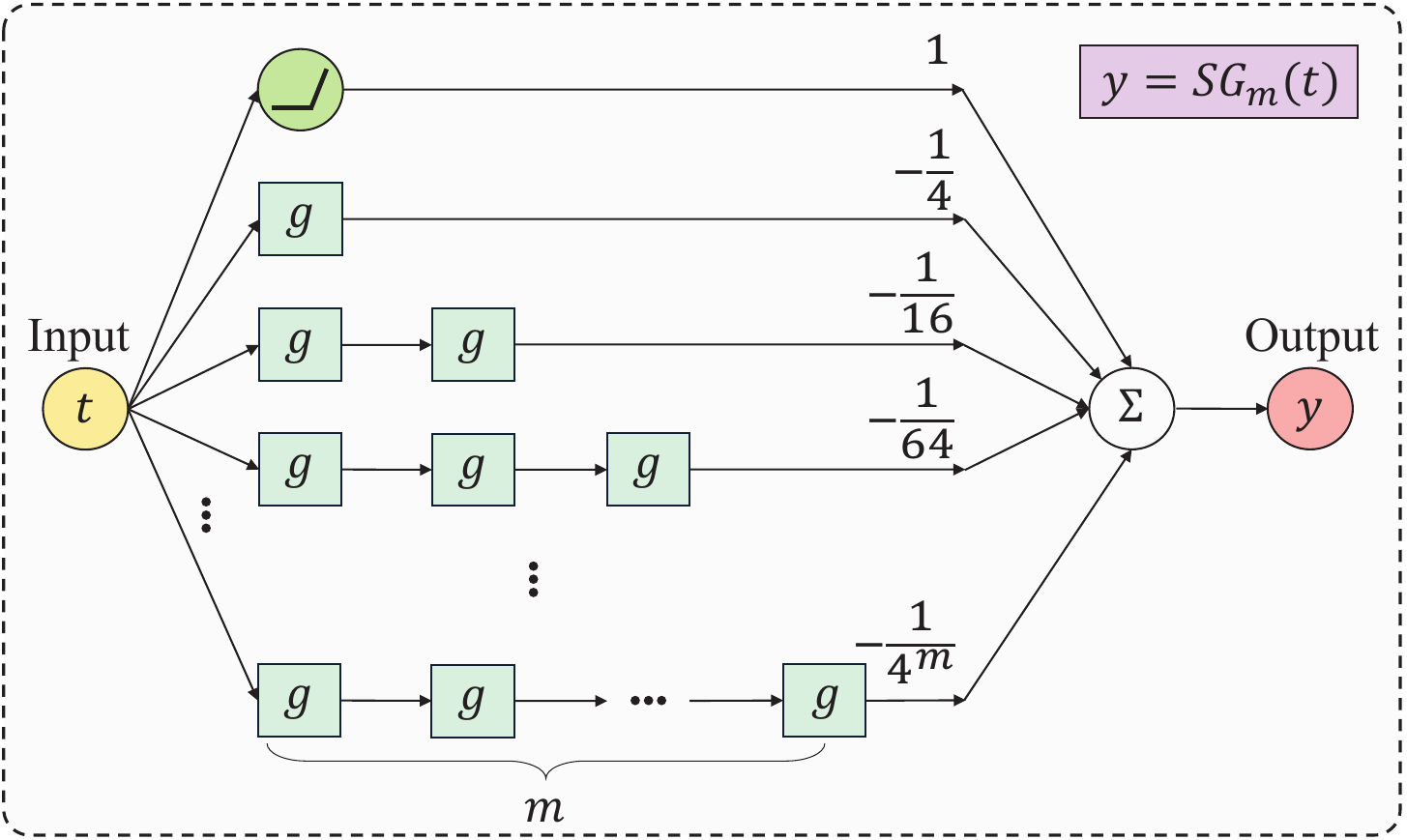}
\caption{Network architecture of the function $SG_m(t)=t-\sum\limits_{s=1}^m\frac{g_s(t)}{4^s}$}
\label{Figure:func_fm}
\end{figure}

The final construction in our proof is the product-component, which utilizes a deep neural network to approximately encode the product of two or more real numbers.
Since $t_1t_2=\frac{(t_1+t_2)^2-t_1^2-t_2^2}{2}$, we can use the square-component defined in \eqref{Def.square-component} to get the product component directly. For any $m\in\mathbb N$, define
\begin{equation}\label{Def.product-component-2}
    {PG}_{2,m}(t_1,t_2):=\frac{{SG}_m(t_1+t_2)-{SG}_{m}(t_1)-{SG}_m(t_2)}{2}.
\end{equation}
It is obvious that ${PG}_{2,m}$ is a deep net of depth $m$ and $\mathcal O(m^2)$ free parameters. Furthermore, for $j=3,4,\dots$, we can define the product-component of order $j$ as
 \begin{equation}\label{Def.product-component-j}
     PG_{j,m}(t_1,\dots,t_j)
     := \overbrace{PG_{2,m}(PG_{2,m}(
     \cdots PG_{2,m}}^{j-1}(t_1,t_2),\cdots,t_{j-1}),t_{j}).
\end{equation}
As shown in Figure \ref{Figure:func_pro2}, $ {PG}_{j,m}$ is a deep net of depth $m(j-1)$, with $\mathcal O(m^2)$ free parameters. The following proposition, which can be derived directly from Proposition \ref{Prop:square-component}, demonstrates the effectiveness of the product-component defined in \eqref{Def.product-component-j}.
\begin{figure}[!t]
\centering
\centering
\includegraphics*[width=4.2cm,height=4.5cm]{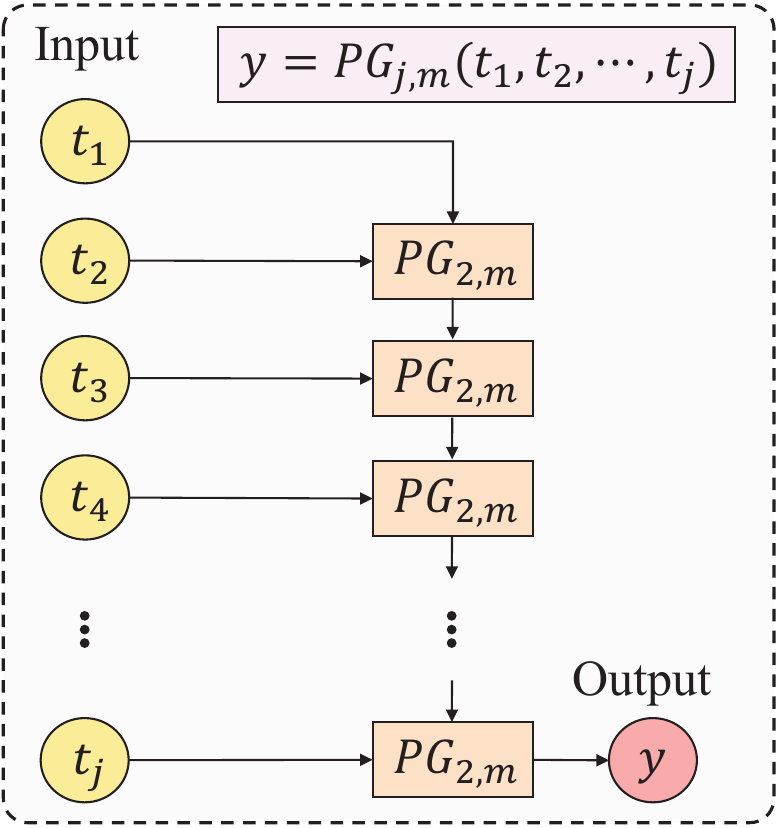}
\caption{Network architecture of the function $PG_{j,m}(t_1,\dots,t_j)$}
\label{Figure:func_pro2}
\end{figure}

\begin{proposition}\label{Prop:product-component-bound}
   Let $j,m\in\mathbb N$, $a,b\in\mathbb R$, and ${PG}_{j,m}$ be given in \eqref{Def.product-component-j}. Then for any $t_1,\dots,t_j\in[a,b]$, there holds
\begin{equation}\label{product-component-bound}
    |t_1\cdots t_j-{PG}_{j,m}(t_1,\dots,t_j)|\leq c_2j2^{-m},
\end{equation}
where $c_2$ is a constant depending only on $a,b$.
\end{proposition}

\section{Component-based Sketching for Deep ReLU Nets}\label{Sec.algorithm}
In this section, we present a component-based sketching scheme  to equip   deep ReLU nets.





\subsection{Construction of sketching basis via dimension leverage}

Given a partition of the interval $\mathbb J:=[-1/2,1/2]$,
we first use  $T_{\tau,t_{k-1},t_{k}}(t)$ in \eqref{Localized-identifier} to recognize the position of the univariate inputs. That is, $T_{\tau,t_{k-1},t_{k}}(t)$  reflects the locality of $t$ in the spatial domain.
Then, we introduce an additional index $J\in\mathbb N$ to capture the locality in the frequency domain \cite{han2020depth,yarotsky2017error} to avoid the saturation phenomenon. Given the tolerance parameter $m$, locality parameter $\tau$, and frequency parameter $J$, we define
\begin{equation}\label{basis-for-simpl}
    \tilde{\times}_{J,j,k,m,\tau}(t):=PG_{J,m}(T_{\tau,t_k,t_{k+1}}(t),\overbrace{t,\dots,t}^{j},\overbrace{1,\dots,1}^{J-1-j}).
\end{equation}


\begin{figure*}[!t]
\centering
\centering
\includegraphics*[width=16cm,height=9cm]{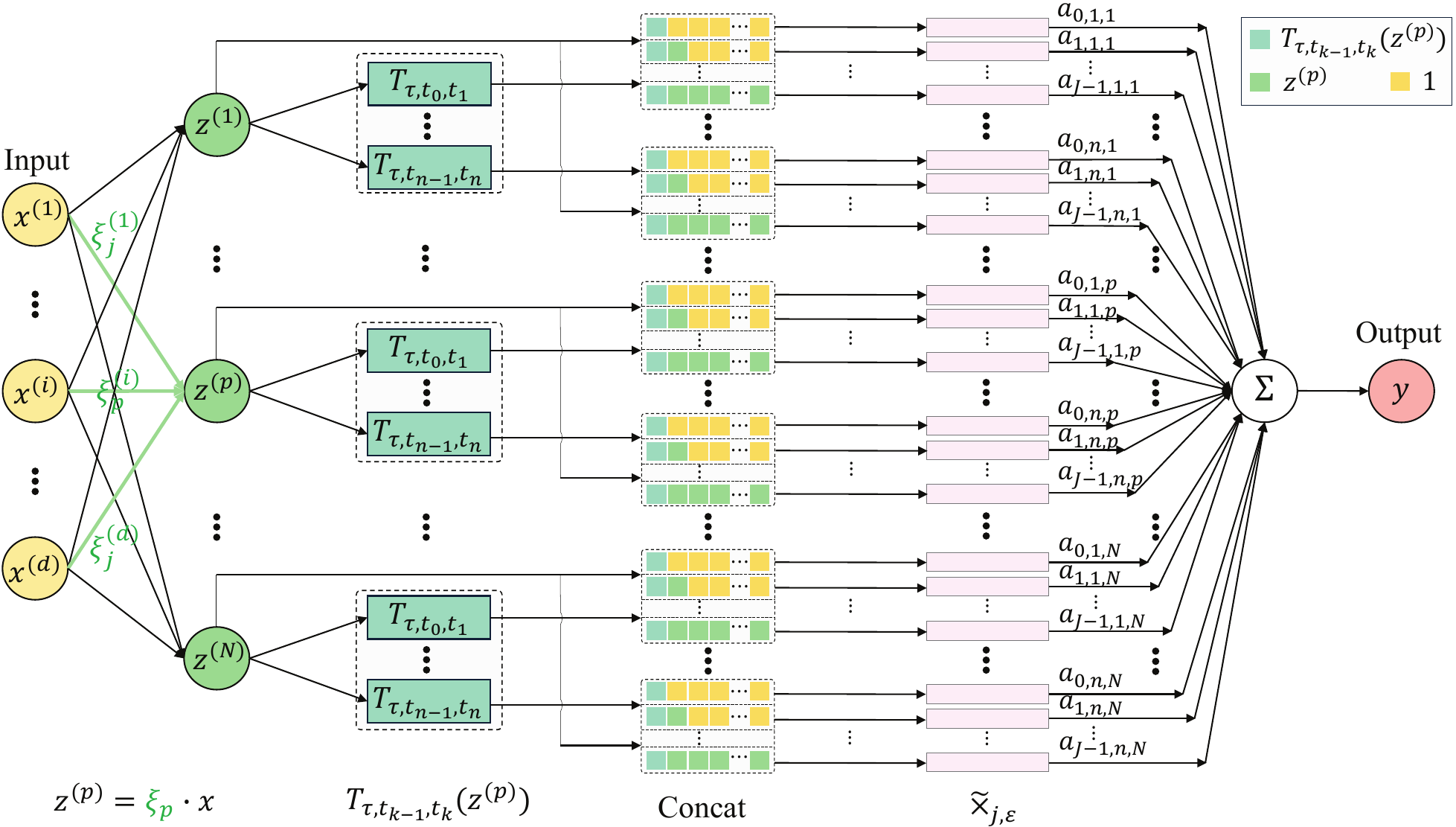}
\caption{Network architecture of the proposed method}
\label{Figure:proposed_net}
\end{figure*}

It should be mentioned that $J$ is not very large in practice \cite{lin2017does} and $\tilde{\times}_{J,j,k,m,\tau}$ reduces to $T_{\tau,t_k,t_{k+1}}$ when $J=1$. We then introduce the dimension leveraging scheme based on
 the minimum energy points on spheres.
Let $\mathbb S^{d-1}$ be the unit sphere embedded in $d$-dimensional Euclidean space $\mathbb R^d$.
The Riesz $\mu $-energy $(\mu \geq 0)$ associated with $\Xi _{N}=\{\xi_i\}_{i=1}^N$, denoted by
$A_{\mu }(\Xi _{N})$, is defined as \cite{brauchart2015distributing}
$$
           A_{\mu }(\Xi _{N}):=\left\{
          \begin{array}{cc}
           \sum_{i\neq j}|\xi_{i}-\xi_{j}|^{-\mu }, & if\ \mu >0, \\
            \sum_{i\neq j}-\log |\xi_{i}-\xi_{j}|, & if\ \mu =0.%
           \end{array}%
           \right.
$$
 We use $ \mathcal{R}_{\mu
}(\mathbb S^{d-1},N)$ to denote the $N$-point minimal $\mu $-energy
over $\mathbb{S}^{d-1}$, that is,
\begin{equation}\label{minimum energy points set}
\mathcal{R}_{\mu}(\mathbb{S}^{d-1},N):=\min_{\Xi _{N}\in \mathbb{S}%
^{d-1}}A_{\mu }(\Xi _{N}),
\end{equation}
where the minimization is taken over all $N$-point configurations of $%
\mathbb{S}^{d-1}$. If $\Xi _{N}^{\ast }\subset \mathbb{S}^{d-1}$ is
a minimizer of (\ref{minimum energy points set}), i.e.,
$$
            A_{\mu }(\Xi _{N}^{\ast })=\mathcal{R}_{\mu
              }(\mathbb{S}^{d-1},N),
$$
then $\Xi _{N}^{\ast }$ is called a minimal $\mu $-energy configuration of $%
\mathbb{S}^{d-1}$, and the points in $\Xi _{N}^{\ast }$ are called
the minimal $\mu $-energy points.
Numerous feasible approaches have been developed to approximately solve the
minimal $\mu$-energy configuration problems, among which a widely used
procedure is Leopardi's recursive zonal sphere partitioning
procedure \cite{leopardi2006partition}. It has been justified that
this procedure can approximately generate the minimal $\mu$-energy points of $\mathbb{S}
^{d-1} $ for some $\mu$ with a ``cheap'' computational cost,
more precisely, with an $\mathcal{O}(N\log N)$ asymptotic time
complexity \cite{leopardi2006partition}. Let $\{\xi_\ell\}_{\ell=1}^N$ denote the set of minimal $\mu$-energy  points.  We then  set $t=\xi_\ell\cdot x$ in \eqref{basis-for-simpl} to leverage the dimension.
Due to the locality property (\ref{Deteailed trapezoid}) of $T_{\tau,t_{k-1},t_k}$, the basis $\{T_{\tau,t_{k-1},t_k}(\xi_\ell\cdot x)\}_{\ell,k}^{N,n}$ presents a cone-type partition of the ball and thus reflects some position information of the input $x$, just as Figure \ref{Figure:cone-partition} exhibits. Differently from the classical cubic partition, the cone-type partition produced by the dimension leveraging scheme  is user-friendly when $d$ is large.

Given $J,n,N,m$,$\tau$, and $\{\xi_\ell\}_{\ell=1}^N$, we obtain a set of sketching basis $\{\tilde{\times}_{J,j,k,m,\tau}(\xi_\ell\cdot x)\}_{\ell,k,j}^{N,n,J-1}$ and construct a
  parameterized   hypothesis space
\begin{equation}\label{Hypothesis-space-deep}
             \mathcal{H}_{J,n,N,m,\tau}:=\left\{
             \sum_{j=0}^{J-1}\sum_{k=1}^{n}\sum_{\ell=1}^{N}a_{jk \ell} \tilde{\times}_{J,j,k,m,\tau}(\xi_\ell\cdot x):a_{ jk\ell}\in \mathbb{R}\right\}.
\end{equation}
The position information of $x$ and the almost equally-spaced property of  minimal $\mu$-energy configurations \cite{brauchart2015distributing} are
crucial to guarantee the representation capability of the linear space (\ref{Hypothesis-space-deep}) \cite{wang2020random}.
Figure \ref{Figure:proposed_net} presents the structure of deep nets  in \eqref{Hypothesis-space-deep}. It can be found in Figure \ref{Figure:proposed_net} that elements in the sketching basis possess a special structure that is different from the classical deep fully connected neural networks. Actually, their structures depend heavily on the deep net components rather than neurons. As the deep net components are specified before the training process, provided the parameters  $J,n,N,m,\tau$ are given, $ \mathcal{H}_{J,n,N,m,\tau}$ defined by \eqref{Hypothesis-space-deep} is actually a linear space of dimension $nJN$ and can be generated off-line.

\subsection{Component-based sketching for deep ReLU nets}
In this subsection, we propose a component-based sketching scheme for deep ReLU nets.
Given the training data $D=\{(x_i,y_i)\}_{i=1}^{|D|}$ with $x_i\in\mathbb B_{1/2}$ and $y_i\in[-M,M]$ for some $M>0$, and generated hypothesis space $ \mathcal{H}_{J,n,N,m,\tau}$,
the deep net  estimate is defined by
\begin{equation}\label{RLS-deep-final}
      f_{D,J,n,N,m,\tau}(x):= \pi_M f^*_{D,J,n,N,m,\tau}(x),
\end{equation}
where $\pi_M(t)=\mbox{sign}(t)\min\{|t|,M\}$ is the well known truncation operator \cite{gyorfi2002distribution} and
\begin{equation}\label{RLS-deep}
          f^*_{D,J,n,N,m,\tau}:={\arg \min}_{f\in  \mathcal{H}_{J,n,m,N,\tau }}\frac1{|D|}\sum_{(x_i,y_i)\in D}%
             |f(x_{i})-y_{i}|^{2}.
\end{equation}
The detailed implementation of the sketching scheme for deep ReLU nets can be found in Algorithm \ref{Algorithm:sketching-deep}.

\begin{algorithm}[t]\caption{Component-based sketching for deep  ReLU nets}\label{Algorithm:sketching-deep}
\begin{algorithmic}
\STATE { \textbf{Input:} Training data $D=\{(x_i,y_i)\}_{i=1}^{|D|}$, input for test data $D^*=\{x_i^*\}_{i=1}^{|D^*|}$, and parameters $n,N,J\in\mathbb N$,
$\tau\in (0,1)$, and $\varepsilon\in (0,1)$. }
\STATE {{ Step 1 (Initialization)}: Normalize $D_x=\{x_{i}\}_{i=1}^{|D|}$ and $D^*$ to $\mathbb
B_{1/2}^d$ and compute $M=\max_{1\leq i\leq |D|}|y_{i}|.$}
\STATE {{ Step 2 (Off-line: Product-component)}: For given $J\in\mathbb N$ and sufficiently small $\varepsilon>0$, construct a deep net
$PG_{J,m}$ defined by \eqref{Def.product-component-j} such that $|t_1\cdots t_J-PG_{J,m}(t_1,\dots,t_J)|\leq \varepsilon$.}
\STATE {{ Step 3 (Off-line:Sketching)}: For given $n,N\in\mathbb N$, generate a set of minimal $\mu$-energy points  $\{\xi_j\}_{\ell=1}^N$ with $\mu>d-1$, and a set of equally spaced points on $G_n:=\{t_k\}_{k=0}^n$ with $t_k=-1/2+k/n$.}
\STATE {{ Step 4 (Off-line:Spatial locality identifier)}: For $\tau\in(0,1)$, generate a set of spatial locality identifiers $\{T_{\tau,t_{k-1},t_k}(\xi_\ell\cdot x)\}_{\ell=1,k=1}^{N,n}$ via (\ref{Localized-identifier}).
 \STATE{ { Step 5 (Off-line: Generating Hypothesis space)}:
Generate the hypothesis space (\ref{Hypothesis-space-deep}).}
 \STATE{ { Step 6 (On-line: Deriving the estimate)}: Define the estimator by
$$
          f_{D,J,n,N,m,\tau}(x):= \pi_M f^*_{D,J,n,N,m,\tau}(x),
$$
where $f^*_{D,J,n,N,\tau,\varepsilon}$ is given by (\ref{RLS-deep}).}
\STATE{\textbf{Output:}  Prediction of input for test data $\{f_{D,J,n,N,m,\tau}(x^*_i)\}_{i=1}^{|D^*|}$. }
}
\end{algorithmic}
\end{algorithm}



At the first glance, there are totally six hyper-parameters involved in the sketching scheme: $N, n, J, m, \tau, \mu$. Note that $\tau, \mu$, and $m$ can be specified before the learning process. In particular, we can use the approaches in \cite{leopardi2006partition} to generate a set of minimal $\mu$-energy points and make $m$ sufficiently large. In this way, $m$ and $\mu$ are not practically involved. Furthermore, it can be found in Figure \ref{Figure:localized-1} that small $\tau$ leads to good locality. We can select $\tau$ to be small enough. The selection of $N,n,J$ is more important than the other three hyper-parameters. Our theoretical analysis shows that the size of the minimal $\mu$-energy configuration should satisfy $N\sim n^{d-1}$. All these demonstrate that the crucial hyper-parameter of the proposed component-based sketching scheme is $n$, which reflects the well-known bias-variance trade-off in machine learning \cite{cucker2007learning}, where large $n$ results in small bias but large variance, while small $n$ implies small variance but large bias. The perfect generalization performance relies on the selection of a suitable $n$. Besides $n$, $J$ is another crucial parameter to realize the advantage of deep nets. In fact, if $J=1$, the hypothesis space $\mathcal{H}_{J,n,m,N,\tau }$ is actually a shallow net.  The component-based sketching scheme is embodied by tuning $J$, which involves fixed structures of deep nets for special purposes. Fortunately, in most applications, $J$ is not large \cite{lin2017does} and cannot be larger than $5$ in practice, making the parameter selection  not so difficult.

As shown in Algorithm \ref{Algorithm:sketching-deep}, there are mainly three novelties of the component-based sketching scheme for deep ReLU nets. At first, it transforms the complicate non-convex and non-linear optimization for deep learning into a linear least squares problem, which successfully avoids the conflict between optimization and generalization. That is, it is easy to find a global minimum \eqref{RLS-deep} by using some basic gradient-based strategy or analytic methods. Then, mathematically speaking, the proposed component-based sketching scheme
requires the computational complexity of order
$\mathcal{O}(|D|Nd+|D|J^2Nn+J^3N^3n^3)$.
Our theoretical  and numerical results below will show that $JNn$ is much smaller than $|D|$, demonstrating that Algorithm \ref{Algorithm:sketching-deep} significantly reduces the computational burden compared to classical deep learning training \cite{goodfellow2016deep}. Finally, we will  show that such a component-based sketching scheme succeeds in breaking through the mentioned saturation problem of shallow ReLU nets \cite{yarotsky2017error,wang2020random} and therefore demonstrates the power of depth in learning extremely smooth functions.

\section{Theoretical Behaviors}\label{Sec.theory}
In this section, we focus on studying theoretical behaviors of the proposed  algorithm. In particular, we are interested in the expressivity of the hypothesis space spanned by  the sketching basis and the generalization performance of the component-based sketching algorithm.

\subsection{Expressivity of the sketching basis}
Different from the classical deep fully connected nets, the hypothesis space $\mathcal{H}_{J,n,N,m,\tau}$ we adopted in the component-based sketching scheme is a linear space of $JnN$ dimension. The first and most important question is whether such a linearization degrades the expressivity of deep nets. The answer to this question is plausibly  positive, since the nonlinear deep net hypothesis space $\mathcal H_{L,n}$ can be regarded as infinite unions of linear deep net hypothesis spaces when the linear and nonlinear deep nets have the same depth and width. However, due to the special construction of $\mathcal{H}_{J,n,N,m,\tau}$ that possesses special structures via combining deep net components and generally cannot be achieved by training a deep fully connected net, we show in this subsection that $\mathcal{H}_{J,n,N,m,\tau}$ does not essentially shrink the capacity of nonlinear deep nets.

There are several quantities that have been employed to measure the capacity of deep nets, such as the covering number \cite{guo2019realizing}, pseudo-dimension \cite{lin2022universal},  and VC-dimension \cite{bartlett2019nearly}. To be detailed, let $\mathcal X\subseteq\mathbb R^d$, $\mathcal B$ be a Banach space, and $\mathcal V$ be a compact set
in $\mathcal B$. The quantity $H_\varepsilon(\mathcal V,\mathcal B)=\log_2\mathcal N_\varepsilon(\mathcal V,\mathcal B)$, where $\mathcal N_\varepsilon(\mathcal V,\mathcal B)$ is the number of elements in the least $\varepsilon$-net of $\mathcal V$, is called the $\varepsilon$-entropy of $\mathcal V$ in  $\mathcal B$. The quantity $\mathcal N_\varepsilon(\mathcal V,\mathcal B)$ is called the
 $\varepsilon$-covering number of  $\mathcal V$. For any $t\in\mathbb
R$, define
$$
        \mbox{sgn} (t):=\left\{\begin{array}{cc}
                       1, & \mbox{if} \ t\geq0,\\
                       -1, & \mbox{if}\ t<0.
                       \end{array}
                       \right.
$$
If a vector ${\bf t}=(t_1,\dots,t_n)$ belongs to $\mathbb R^n$, then we denote by $\mbox{sgn} ({\bf t})$ the vector $(\mbox{sgn} (t_1),$ $\dots$,$\mbox{sgn} (t_n))$. The VC dimension of $\mathcal V$ over $\mathcal X$, denoted by $VCdim(\mathcal V)$, is defined as the maximal natural number $l$ for which there exists a collection $(\eta_1,\dots,\eta_l)$ in  $\mathcal X$ such that the cardinality of the sgn-vectors set
$$
              S=\{(\mbox{sgn} (v(\eta_1)),\dots,\mbox{sgn} (v(\eta_l))):v\in \mathcal V\}
$$
equals $2^l$, that is, the set $S$ coincides with the set of all vertexes of the unit cube in  $\mathbb R^l$. Here and hereafter, $\mbox{sgn}(t)$ for $t\in\mathbb R$ denotes the indicator function.
The quantity
$$
                Pdim(\mathcal V):=\max_{g} VCdim(\mathcal V+g),
$$
is called the pseudo-dimension of the set $\mathcal V$ over $X$,
where $g$ ranges over all functions defined on $X$ and $\mathcal
V+g=\{v+g:v\in \mathcal V\}$.

In this paper, we use the pseudo-dimension to quantify the capacity, mainly due to the following   important properties exhibited in \cite{mendelson2003entropy}.

\begin{lemma}\label{Lemma:CAPACITY-RELATION}
Let $G$ be a $k$-dimensional vector space of functions from a set
$\mathbb A$ into $\mathbb R$. Then $Pdim(G )=k$.
If $\mathcal V_R$ is a class of functions which consists of
all functions $f\in \mathcal V$ satisfying $|f(x)|\leq R$ for all
$x\in \mathbb A$, then
$$
             VC(\mathcal V)\leq {Pdim}(\mathcal V),
$$
and
$$
            H_\varepsilon(\mathcal V_R,L^2(\mathbb A))\leq c
            {Pdim}(\mathcal V_R)\log_2\frac{R}\varepsilon,
$$
where $c$ is an absolute positive constant.
\end{lemma}

The above lemma demonstrates that deriving a tight bound of the pseudo-dimension is sufficient to provide upper bounds of both the VC-dimension and covering numbers of deep nets. Moreover, Lemma \ref{Lemma:CAPACITY-RELATION} shows that the pseudo-dimension is a good measurement of the capacities of linear and nonlinear spaces, since the pseudo-dimension of a linear space equals the dimension.
Furthermore, the following upper bounds on the pseudo-dimension of deep nets were derived in
\cite{lin2022universal,bartlett2019nearly}.
\begin{lemma}\label{Lemma:Pseudo-for-DCNN}
Let $\mathcal H_{n^*,L,d_{\max}}$ be the set of deep nets with $n^*$ free parameters, depth $L$, and width $d_{\max}$. Then
there exists an absolute constant $C_0$ such that
\begin{equation}\label{bound-for-pseudo}
    Pdim(\mathcal H_{n^*,L,d_{\max}})\leq C_0Ln^*\log d_{\max}.
\end{equation}
\end{lemma}

Lemma \ref{Lemma:Pseudo-for-DCNN} shows that for the set of deep ReLU nets with $n^*$ free parameters, depth $L$, and width $d_{\max}$, the pseudo-dimension cannot be larger than $Ln^*\log d_{\max}$. Recalling the construction of $\mathcal{H}_{J,n,N,m,\tau}$, which is a deep net with $\mathcal O(JnNm^2)$ free parameters, depth $\mathcal O(mJ)$, and width $\mathcal O(JnN(m+J))$, we can derive from Lemma \ref{Lemma:Pseudo-for-DCNN} and Lemma \ref{Lemma:CAPACITY-RELATION} the following proposition.
\begin{proposition}\label{Prop:pseudo-dimension}
Let $\mathcal{H}_{J,n,N,m,\tau}$ be given by \eqref{Hypothesis-space-deep} with $0<\tau<1$ and $m\sim  \log n$. Then, there holds
\begin{equation}\label{Pseudo-dimen-sketching}
    JnN\leq  Pdim( \mathcal{H}_{J,n,N,m,\tau})
     \leq C_1  J nN\log^3 (JnN),
\end{equation}
where $C_1$ is a constant depending only on $d$.
\end{proposition}

Proposition \ref{Prop:pseudo-dimension} demonstrates that up to a logarithmic factor, linearization  does not essentially degrade  the expressivity of deep nets.
In the following, we show that $m\sim \mathcal O( \log n)$ as stated in Proposition \ref{Prop:pseudo-dimension} is feasible for the sketching basis, allowing us to achieve almost optimal approximation rates. To this end, we first introduce the well known Sobolev space. For $r>0$, let $\mathbb A\subseteq \mathbb R^d$ be a compact set and $W^r(L^2(\mathbb A))$ be the
Sobolev space for the domain $\mathbb A$. When $r=u$ is an
integer, a function $f\in L^2(\mathbb A)$ belongs to
$W^r(L^2(\mathbb A))$ if and only if its distributional
derivatives $D^u f$ of order $u$ are in $L^2(\mathbb A)$,
and the seminorm for $W^u(L^2(\mathbb A))$ is given by
$$
        |f|_{W^u(L^2(\mathbb
        A))}^2=\sum_{|\nu|=u}\|D^\nu f\|_{L^2(\mathbb
        A)}^2.
$$
The norm
for $W^u(L^2(\mathbb A))$ is obtained by adding $\|
f\|_{L^2(\mathbb
        A)}$ to $|f|_{W^u(L^2(\mathbb
        A))}$. For other values of $r$, we obtain
$W^r(L^2(\mathbb A))$ as the interpolation space \cite{devore1993constructive}
$$
     W^r(L^2(\mathbb A))=(L^2(\mathbb A),W^u(L^2(\mathbb
     A)))_{\theta,2},\quad \theta=r/u,0<r<u.
$$
The following theorem quantifies the distance between $W^r(L^2(\mathbb B_{1/2}^d))$ and $ \mathcal{H}_{J,n,N,m,\tau}$ for some $r>0$, showing that $m\sim \mathcal O(\log n)$ and linearization does not degrade the approximation performance of deep nets.

\begin{theorem}\label{Theorem:approximation}
Let   $d\geq 2$, $r>0$,   $J,n,N\in N$ satisfy $N\sim n^{d-1}$ and $J\geq r$. If $f\in W^{r+(d-1)/2}(L^2(\mathbb B^d_{1/2}))$,
$0<\tau\leq n^{-4J-1}$, and $m\sim \log n$, then
\begin{eqnarray}\label{approximation-rate-deep}
       && C'n^{-r-(d-1)/2}\|f\|_{ W^{r+(d-1)/2}(L^2(\mathbb B^d_{1/2}))} \nonumber\\
       &\leq& \mbox{dist}\left( W^r(L^2(\mathbb B_{1/2}^d)),\mathcal{H}_{J,n,N,m,\tau}\right) \nonumber\\
        &\leq& Cn^{-r-(d-1)/2}\|f\|_{ W^{r+(d-1)/2}(L^2(\mathbb B^d_{1/2}))},
\end{eqnarray}
where $C,C'$ are constants depending only on $r$ and $d$, and
\begin{equation}\label{distance}
    \mbox{dist}(\mathbb A,\mathbb B):=\sup_{f\in\mathbb A}\mbox{dist}(f,\mathbb B):=\sup_{f\in\mathbb A}\inf_{g\in\mathbb B}\|f-g\|_{L^2(\mathbb B^d_{1/2})}
\end{equation}
for two sets $\mathbb A,\mathbb B\in L^2(\mathbb B_{1/2}^d)$.
\end{theorem}

As the ReLU function $\sigma$ is non-differentiable, it is impossible for shallow nets with a finite number of neurons to approximate smooth function well \cite{yarotsky2017error}. In particular, \cite[Theorem 6]{yarotsky2017error}  proves that any twice-differentiable nonlinear function defined on $\mathbb \mathbb B_{1,2}^d$ cannot be $\epsilon$-approximated by ReLU networks of fixed depth $L$ with the number of free parameters less than $c\epsilon^{-1/(2(L-2))}$, where $c$ is a  positive constant depending only on $d$. Deepening the network is necessary to break through such a saturation problem. In our construction, the parameter $J$ is introduced to address the saturation problem. As $J$ increases, Theorem \ref{Theorem:approximation} allows for large $r$   and therefore avoids the saturation problem. Recalling \cite{han2020depth,yarotsky2017error,petersen2018optimal} where the derived approximation rates are also almost optimal for all deep nets with the same sizes,
Theorem \ref{Theorem:approximation} together with Proposition \ref{Prop:pseudo-dimension} illustrates that linearization does not essentially shrink the capacity of the deep nets hypothesis spaces, implying that the proposed component-based sketching scheme is capable of maintaining the excellent learning performances of deep learning while successfully circumventing the conflict between optimization and generalization.

\subsection{Generalization error for the component-based sketching algorithm}

Our previous analysis quantifies the representation capability of the constructed deep ReLU nets (\ref{Hypothesis-space-deep}), showing that the component-based sketching scheme for deep ReLU nets succeeds in breaking through the saturation problem of shallow nets in approximating smooth functions. In this part, we show that the component-based sketching scheme is one of the most efficient learners by demonstrating that it is AROE given in Definition \ref{Def:optimal-algorithm} .

As shown in \cite[Chapter 3]{gyorfi2002distribution}, there are no provable learning algorithms that can learn any data well. Therefore, some a-priori information should be specified for analysis. In addition to the smoothness assumption for the regression function, i.e., $f_\rho\in W^{r+(d-1)/2}(L^2(\mathbb B^d_{1/2}))$, we also need a widely-used distortion assumption for the marginal distribution $\rho_X$ \cite{chui2020realization}.
Let $S_\rho$ be the identity mapping
$$
               L^2(\mathbb B_{1/2}^d) ~~ {\stackrel{S_\rho}{\longrightarrow}}~~ L^2_{\rho_X},
$$
and $D_{\rho _{X}}=$ $\Vert S_\rho\Vert$.  $D_{\rho _{X}}$  is called the distortion of $\rho _{X}$, which measures how much $\rho_{X}$ distorts the Lebesgue measure.  The distortion assumption, $D_{\rho _{X}}<\infty$, which obviously holds for the uniform distribution, is necessary since our analysis is built upon the locality of deep ReLU nets \cite{chui2020realization} and has been widely used in the literature \cite{shi2011concentration,lin2018generalization,chui2020realization,wang2020random}.  Define $\Lambda_{dis}$ as the set of $\rho_X$ satisfying $D_{\rho _{X}}<\infty.$ We then present our generalization error estimates for the component-based sketching scheme.



\begin{theorem}\label{Theorem:gene-deep}
Let $r>0$, $d\geq 2$, and $J\in\mathbb N$ with $J\geq r$. If $n,N\in N$ satisfy $N\sim n^{d-1}$, $n\sim|D|^\frac{1}{2r+2d-1}$,
$0<\tau\leq n^{-4J-1}$, and $m\sim\log n$, then there holds
\begin{eqnarray}\label{Generalization-deep}
  && C'_1 |D|^{\frac{-2r-d+1}{2r+2d-1}} \leq
  e(W^{r+(d-1)/2}(L^2(\mathbb B^d_{1/2})),\Lambda_{dis}) \nonumber\\
  &\leq&
  \mathcal V_{W^{r+(d-1)/2}(L^2(\mathbb B^d_{1/2})),\Lambda_{dis}}\nonumber\\
  &\leq&
  C_2'e(W^{r+(d-1)/2}(L^2(\mathbb B^d_{1/2})),\Lambda_{dis})\log |D|\nonumber\\
  &\leq& C_2'   |D|^{ \frac{-2r-d+1}{2r+2d-1}}\log |D|,
\end{eqnarray}
where $C_1'$ and $C_2'$ are constants depending only on $M$, $r$, $d$, $D_{\rho_X}$, and $\|f_\rho\|_{ W^{r+(d-1)/2}(L^2(\mathbb B^d_{1/2}))}$.
\end{theorem}

Theorem \ref{Theorem:gene-deep} shows that, up to a logarithmic factor, the proposed component-based sketching scheme reaches the optimal learning rates that the best theoretical learning algorithms based on $D$ can achieve, implying that it is an AROE in learning smooth functions according to Definition \ref{Def:optimal-algorithm}.
Recalling shallow nets saturate at smoothness with $r=1$  \cite{wang2020random,yarotsky2017error}, Theorem \ref{Theorem:gene-deep} theoretically demonstrates the power of depth in generalization.

Despite the optimal learning rates presented in \eqref{Generalization-deep}, it remains to determine the parameters $n$ since the smoothness index $r$ is frequently not accessible in practice. To fix $n$, we can use the
so-called ``cross-validation'' method \cite[Chapter 8]{gyorfi2002distribution}
or the ``hold-out'' method \cite[Chapter 7]{gyorfi2002distribution}. For the sake of
brevity, we only provide the theoretical assessment of the latter one.
There are three steps to implement the ``hold-outstrategy '': (i) Divide the set of samples into two
independent subsets, $\tilde{D}_1$ with cardinality $|\tilde{D}_1|$ and $\tilde{D}_2$ with cardinality $|\tilde{D}_2|$; (ii) use $\tilde{D}_1$ to construct the sequence $\{f_{\tilde{D}_1,J,n,N,m,\tau}\}_{n=1}^{|\tilde{D}_1|^{\frac{1}{2d-1}}}$; and (iii) use $\tilde{D}_2$ to select an appropriate value of $n$, thus obtaining the final
estimator $f_{D,J,\hat{n},\hat{N},m,\tau}$ with $\hat{N}\sim \hat{n}^{d-1}$.

Let us explain how $\tilde{D}_2$ is used for the selection of $n$. We
consider the set
$$
        \Xi=\{1,2,\dots,\lceil |D|^{\frac{1}{2d-1}}\rceil\}.
$$
Then, the adaptive selection for the kernel parameter is given by
\begin{equation}\label{holdout}
       \hat{n}=\arg\min_{n\in\Xi}\frac1{|\tilde{D}_2|}\sum_{(x_i,y_i)\in\tilde{D}_2}(f_{\tilde{D}_1,J,n,N,m,\tau}\}(x_i)-y_i)^2,
\end{equation}
with $N\sim n^{d-1}$. The following Theorem \ref{THEOREM: adaption selection} illustrates the generalization capability of the component-based sketching algorithm equipped with the proposed parametric selection strategy, and it shows that the proposed parameter selection strategy is optimal.

\begin{theorem}\label{THEOREM: adaption selection}
Let $d\geq 2$ and  $J\in\mathbb N$. If
$0<\tau\leq n^{-4J-1}$ and $m\sim\log n$, then
\begin{eqnarray}\label{gener-para-sele}
  && C'_1 |D|^{\frac{-2r-d+1}{2r+2d-1}} \leq
  \sup_{f_\rho\in W^{r+(d-1)/2}(L^2(\mathbb B^d_{1/2}))}
  \mathbf E \left\{\|f_{\tilde{D}_1,J,\hat{n},\hat{N},m,\tau}-f_\rho\|_\rho^2\right\} \nonumber\\
 &\leq& C_3' (J+D_{\rho_X}^2) |D|^{\frac{-2r-d+1}{2r+2d-1}}\log |D|,
\end{eqnarray}
where $C_3'$ is a constant  depending only on $M$, $r$, $d$, and $\|f_\rho\|_{ W^{r+(d-1)/2}(L^2(\mathbb B^d_{1/2}))}$.
\end{theorem}




\begin{figure*}[t]
	\centering
	\subfigure{
        \rotatebox{90}{\scriptsize{~~~~~~~~~~~~Testing RMSE}}
    	\includegraphics[width=5.5cm,height=2.5cm]{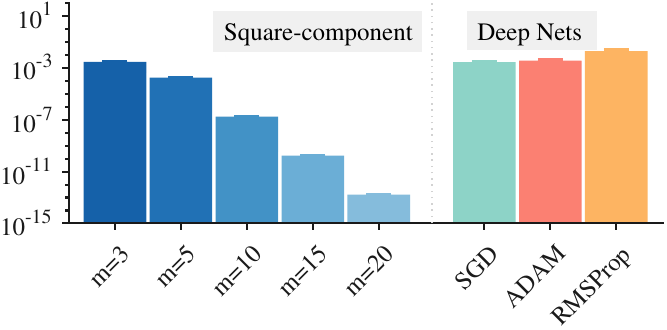}
	}
		\subfigure{
		\includegraphics[width=5.5cm,height=2.5cm]{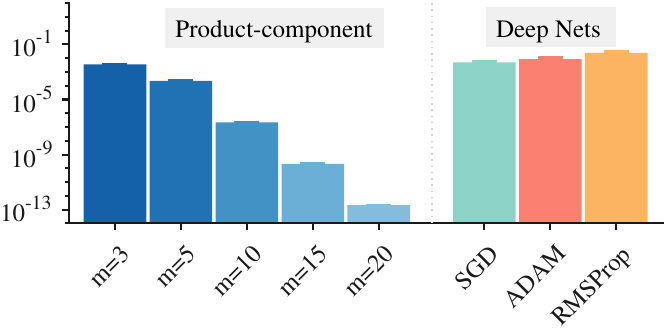}
	}
		\subfigure{
		\includegraphics[width=5.5cm,height=2.5cm]{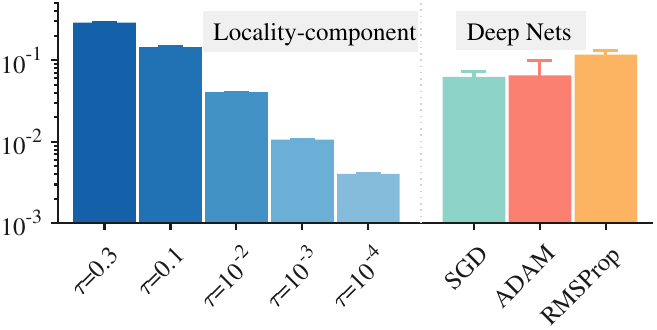}
	}
	\vspace{-3mm}
	\setcounter{subfigure}{0}
    \subfigure[Square function]{
		\rotatebox{90}{\scriptsize{~~~~~~~~~~~Testing time(sec)}}
		\includegraphics[width=5.5cm,height=2.5cm]{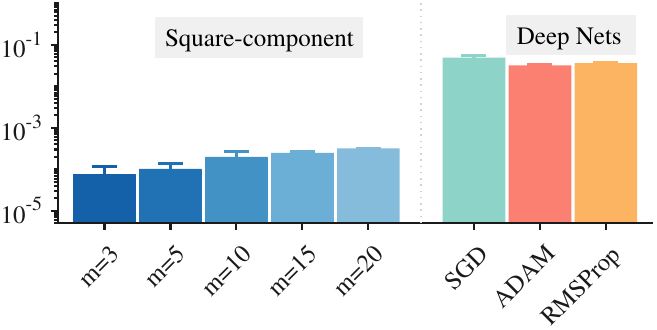}
	}
	\subfigure[Product function]{
		\includegraphics[width=5.5cm,height=2.5cm]{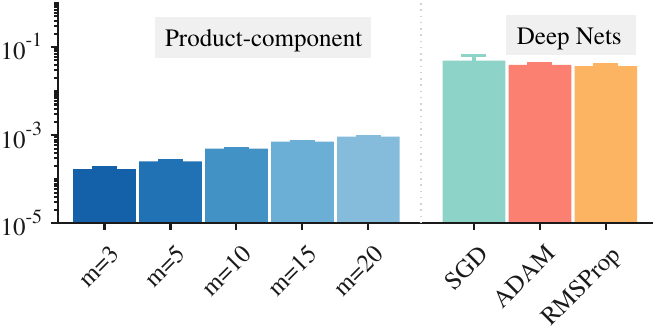}
	}
	\subfigure[Indicator function]{
		\includegraphics[width=5.5cm,height=2.5cm]{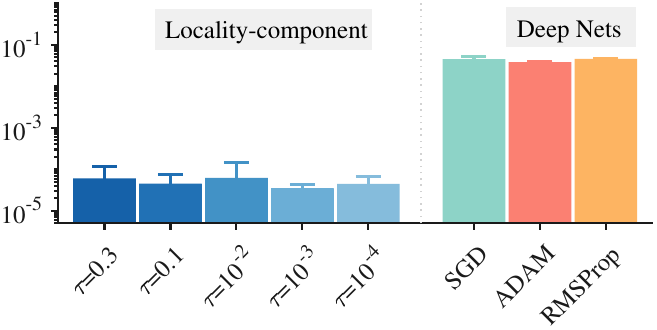}
	}
	\caption{Comparison of the testing RMSE and testing time between the constructed components and fully connected networks for approximating the square function, product function, and indicator function}
	\label{constructVStraining}
\end{figure*}

\begin{figure}[!t]
\centering
\centering
\includegraphics*[width=6cm,height=3cm]{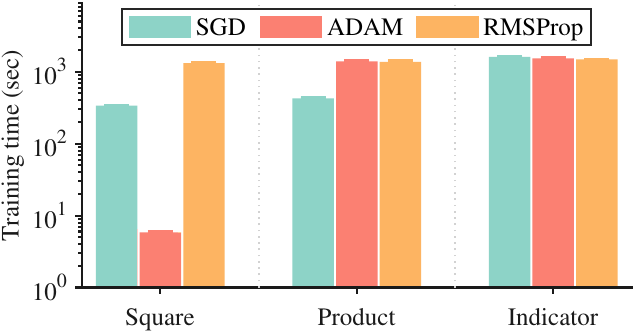}
\caption{The training time required for fully connected networks to approximate the square function, product function, and indicator function}
\label{DeepnetTrtime}
\end{figure}

\begin{figure*}[t]
    \centering
    \subfigure{\includegraphics[width=4.64cm,height=2.8cm]{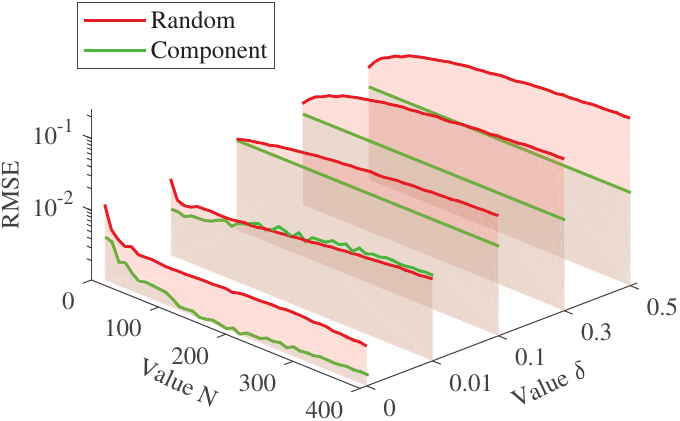}}\hspace{-0.1in}
    \subfigure{\includegraphics[width=4.64cm,height=2.8cm]{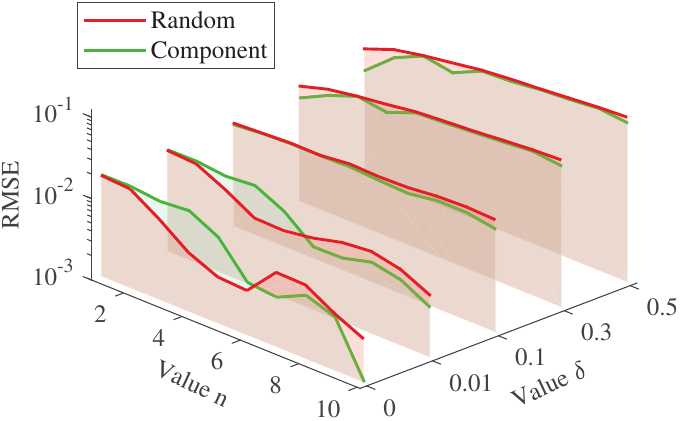}}\hspace{-0.1in}
    \subfigure{\includegraphics[width=4.64cm,height=2.8cm]{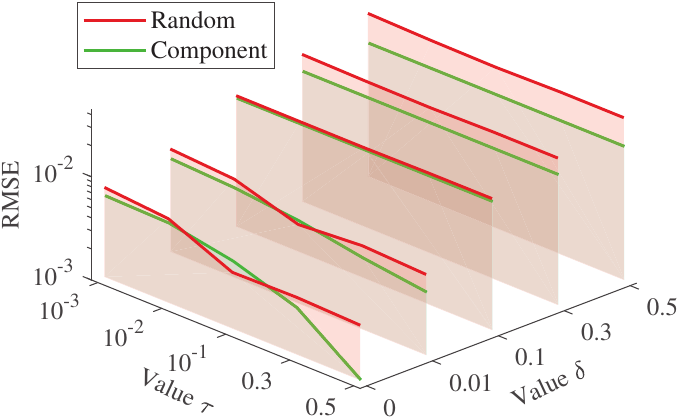}}\hspace{-0.1in}
    \subfigure{\includegraphics[width=4.64cm,height=2.8cm]{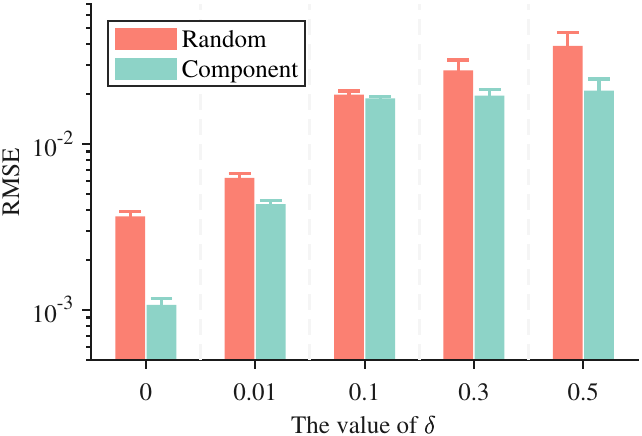}}\\
    \subfigure{\includegraphics[width=4.64cm,height=2.8cm]{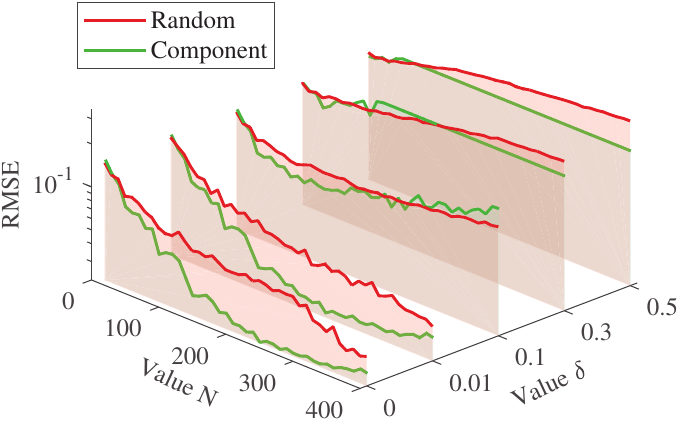}}\hspace{-0.1in}
    \subfigure{\includegraphics[width=4.64cm,height=2.8cm]{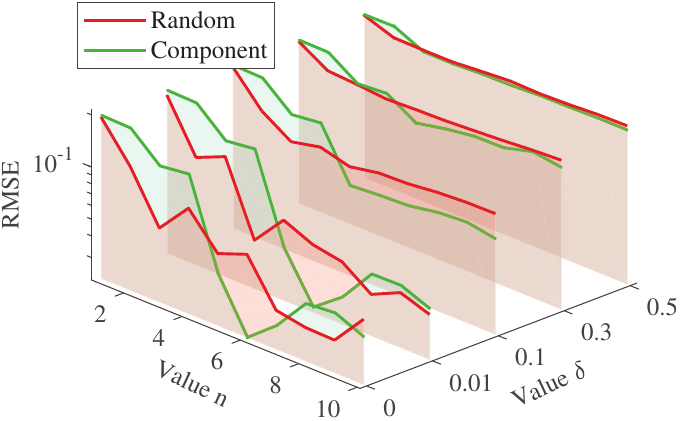}}\hspace{-0.1in}
    \subfigure{\includegraphics[width=4.64cm,height=2.8cm]{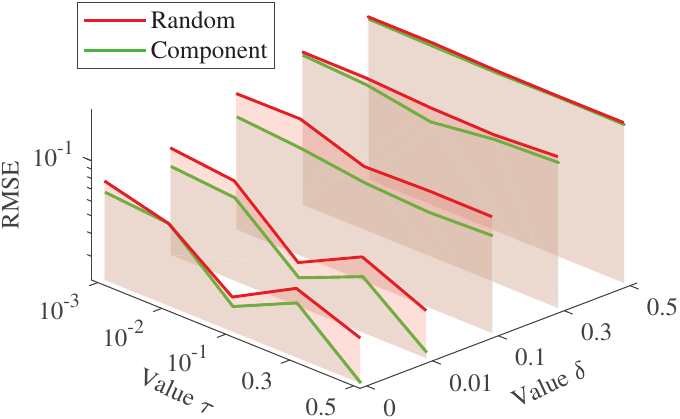}}\hspace{-0.1in}
    \subfigure{\includegraphics[width=4.64cm,height=2.8cm]{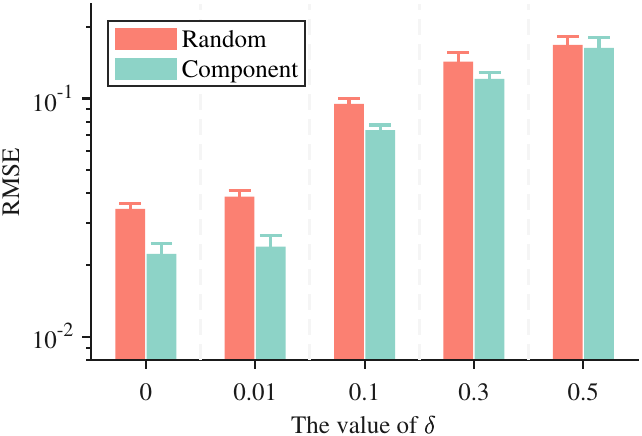}}
    \caption{Comparison of the testing RMSE for the two sketching techniques under varying parameter values and noise levels. The subfigures in the top row represent the results for the data generated by$f_1$, while the subfigures in the bottom row correspond to the data generated by $f_2$}\label{fig_sketching_techs}
    \vspace{-0.1in}
\end{figure*}

\begin{figure*}[t]
    \centering
    \subfigure{\includegraphics[width=3.86cm,height=2.5cm]{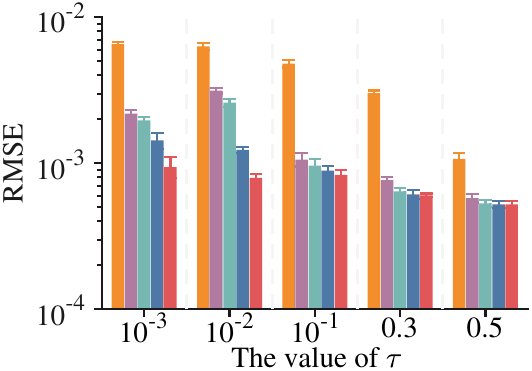}}\hspace{-0.045in}
    \subfigure{\includegraphics[width=3.56cm,height=2.5cm]{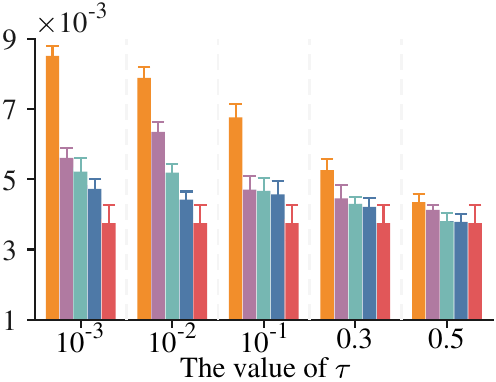}}\hspace{-0.045in}
    \subfigure{\includegraphics[width=3.56cm,height=2.5cm]{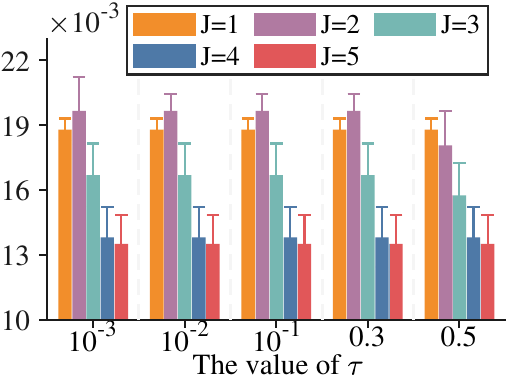}}\hspace{-0.045in}
    \subfigure{\includegraphics[width=3.56cm,height=2.5cm]{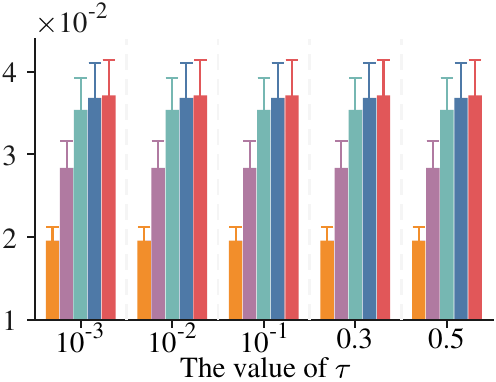}}\hspace{-0.045in}
    \subfigure{\includegraphics[width=3.56cm,height=2.5cm]{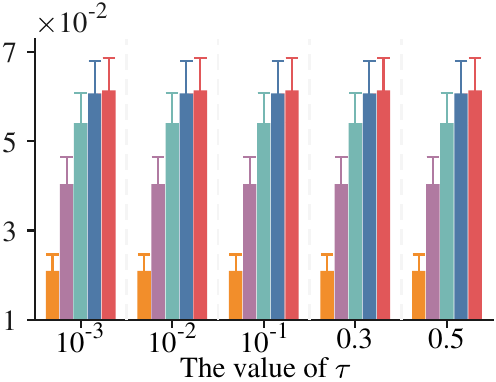}}\\
    \vspace{-0.08in}
    \setcounter{subfigure}{0}
    \subfigure[$\delta=0$]{\includegraphics[width=3.86cm,height=2.5cm]{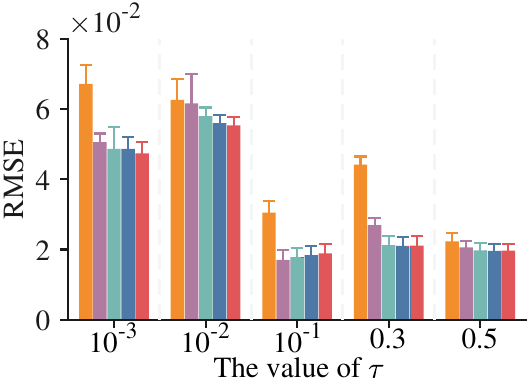}}\hspace{-0.045in}
    \subfigure[$\delta=0.01$]{\includegraphics[width=3.56cm,height=2.5cm]{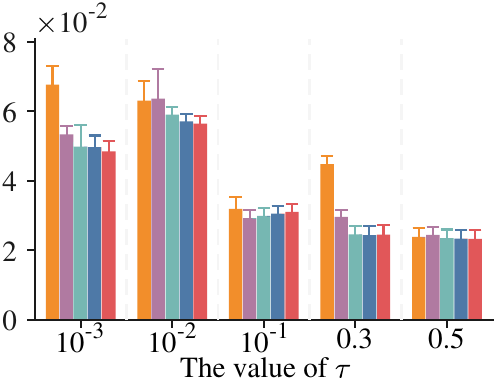}}\hspace{-0.045in}
    \subfigure[$\delta=0.1$]{\includegraphics[width=3.56cm,height=2.5cm]{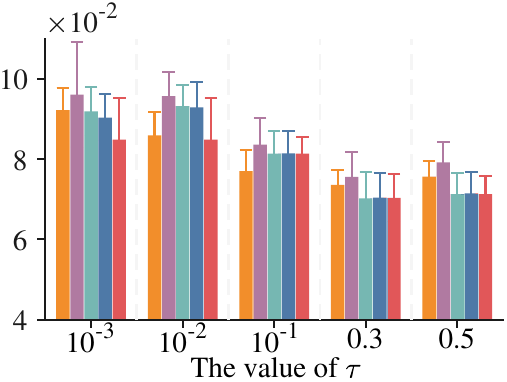}}\hspace{-0.045in}
    \subfigure[$\delta=0.3$]{\includegraphics[width=3.56cm,height=2.5cm]{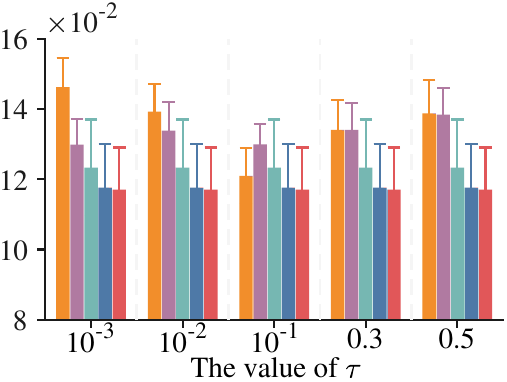}}\hspace{-0.045in}
    \subfigure[$\delta=0.5$]{\includegraphics[width=3.56cm,height=2.5cm]{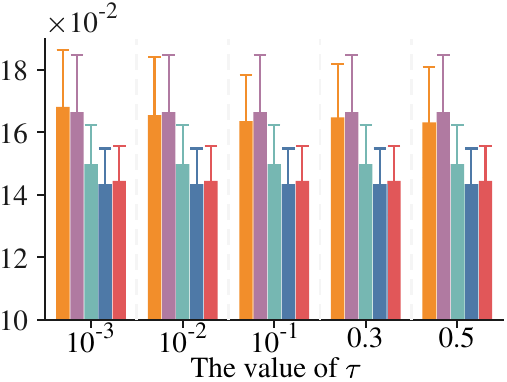}}
    \caption{Relationship between testing RMSE and the value $J$ of the frequency parameter under varying $\tau$ values and noise levels. The subfigures in the top row represent the results for the data generated by $f_1$, while the subfigures in the bottom row correspond to the data generated by $f_2$}\label{fig_freq_role}
    \vspace{-0.1in}
\end{figure*}

\section{Experimental Results}\label{Sec.experiments}
In this section, we present both toy simulations and real-world experiments to demonstrate the efficiency and effectiveness of the proposed method.
The comparison methods include: 1) 13 BP-type algorithms from MATLAB's ``Deep Learning Toolbox" \cite{MathWorks2020} for shallow fully connected
neural networks, specifically comprising 5 gradient descent-related algorithms such as gradient descent (GD),  GD with adaptive learning rate
(GDA), GD with momentum (GDM), GDM and adaptive learning rate (GDX), and resilient BP (RP), as well as 8 conjugate gradient descent-related
algorithms such as Levenberg-Marquardt (LM), Bayesian regularization (BR), BFGS quasi-Newton (BFG), CG-BP with Powell-Beale (CGB) restarts,  CG-BP
  with Fletcher-Reeves (CGF) updates, CG-BP with Polak-Ribi\'{e}re (CGP) updates, one-step secant (OSS), and scaled CG (SCG); 2) deep fully
  connected neural networks with optimizers such as stochastic gradient descent (SGD) \cite{goodfellow2016deep} and adaptive moment estimation
  (Adam) \cite{kingma2014adam}; 3) random sketching for neural networks (RSNN) \cite{wang2020random}. For all fully connected neural networks, we
  use the ReLU activation function. For deep networks, we set the widths of the hidden layers to be equal. All simulations were run on a desktop
  workstation equipped with an Intel(R) Core(TM) i9-10980XE 3.00 GHz CPU, an Nvidia GeForce RTX 3090 GPU, 128 GB of RAM, and Windows 10. The
  results are recorded by averaging the results of multiple individual trials, utilizing the best parameters identified through the grid search.

\subsection{Synthetic Results}\label{Syn_ex}

In this part, we conduct four simulations to validate our theoretical assertions. The first simulation is designed to illustrate the prominent benefits of component construction as compared to neuron-based training. The second simulation aims to demonstrate the superiority of component-based sketching over random sketching. The third simulation is performed to illustrate the importance of incorporating locality information in the frequency domain, which is tied to the frequency parameter $J$. The last simulation highlights the strength of the proposed method, particularly in terms of generalization performance and training time, compared to the aforementioned methods.

Before conducting simulations, we describe the process of generating synthetic data. The inputs $\{\bm{x}_i\}_{i=1}^p$ of the training set are independently drawn from a uniform distribution on the ball $\mathbb{B}_{1/2}^d$ with $d=3$ or $d=4$. The corresponding outputs
$\{y_i\}_{i=1}^N$ are generated according to the regression models in the form of $y_i=f_j(\bm{x}_i)=g_j(r_j(\bm{x}_i))+\varepsilon_i$ for $i=1,2,\dots,N$ and $j=1,2$, where
$\varepsilon_i$ is independent Gaussian noise from $\mathcal{N}(0,\delta^2)$, the function $g_1(r_1)$ is defined as
\begin{equation*}
g_1(r_1)=(r_1-0.1)(r_1-0.5)(r_1-0.9)
\end{equation*}
with $r_1(\bm{x})=2\|\bm{x}\|_2$
for 3-dimensional data,
and $g_2(r_2)$ is defined as
\begin{equation*}
g_2(r_2)=\left\{
\begin{array}{ll}
(1-r_2)^6(35r_2^2+18r_2+3), & \quad\mbox{if} ~ r_2\leq 1 \\
0, & \quad \mbox{otherwise}
\end{array}
\right.
\end{equation*}
with $r_2(\bm{x})=2.2\|\bm{x}\|_2$ for 4-dimensional data. Generalization abilities are evaluated on the testing set $\{(\bm{x}_i',y_i')\}_{i=1}^{p'}$, which is generated similarly to the training set but with the assurance that $y_{i}'=g_j(r_j(\bm{x}_i'))$.
In these simulations, the sizes of the training and testing sets are set to 2000 and 1000, respectively, with the exception of Simulation 1.

{\bf Simulation 1.} In this simulation, we compare the proposed component constructions with neuron-based training of fully connected networks for approximating the square function $y(t)=t^2$, the product function $y(t_1,t_2)=t_1t_2$, and the indicator function
$$
 y(t_1,t_2)=\left\{\begin{array}{ll}
   1, &\mbox{if}\ 1/2\leq t_1\leq 5/8 ~\mbox{and} ~ 3/8\leq t_2\leq 1/2,\\
   0, &\mbox{otherwise}.
   \end{array}
   \right.
$$
As this simulation focuses on comparing the approximation capabilities for the three functions, no noise is introduced into either the training or the testing data. For each function, we generate 10,000 training samples and 1,000 test samples.
In fully connected neural networks, the depth and width are selected from the sets $\{10, 20, \dots, 200\}$ and $\{1,2,\dots,5\}$, respectively. The maximum number of epochs is set to 20,000. Additionally, the piecewise learning rate decay method is utilized, where the initial learning rate, the factor for reducing the learning rate, and the epoch interval for learning rate decay are set to 0.05, 0.99, and 100, respectively. Here, we employ the optimizers of SGD \cite{goodfellow2016deep}, Adam \cite{kingma2014adam}, and RMSProp \cite{Ruder2016}. It is noteworthy that the proposed component constructions do not require training and can be directly evaluated. A comparison of the testing RMSE and testing time between the constructed components and fully connected networks for approximating the three functions is illustrated in Figure \ref{constructVStraining}, and the training time of the fully connected networks with the three optimizers is shown in Figure \ref{DeepnetTrtime}. In the figures, the blue bars on the left represent the results of the component constructs, while the three colored bars correspond to the results of fully connected neural networks utilizing the three optimizers. Based on these findings, the following conclusions can be deduced: 1) For the square and product functions, the approximation errors of the component constructs exhibit an exponential decline as the parameter $m$ increases, significantly outperforming fully connected neural networks. This observation is consistent with the conclusions drawn in Proposition \ref{Prop:square-component} and Proposition \ref{Prop:product-component-bound}. For the indicator function, the approximation errors of the component constructs decrease substantially as the parameter $\tau$ decreases, achieving much lower errors compared to fully connected neural networks when $\tau$ is sufficiently small. This is in accordance with the conclusion of Proposition \ref{Prop:localized}. 2) Although the testing time of the component constructs for the square and product functions gradually increases with the increase of the parameter $m$, it remains notably shorter than that of fully connected neural networks. Consequently, compared to fully connected neural networks, the component constructs drastically reduce the time required to construct the linear hypothesis space \eqref{Hypothesis-space-deep}. More importantly, the component constructs necessitate no training, thereby eliminating the cumbersome parameter tuning process and potentially thousands of seconds of training time associated with fully connected neural networks, as illustrated in Figure \ref{DeepnetTrtime}.
In summary, the component constructs comprehensively outperform fully connected neural networks in both training and testing time, as well as in approximation accuracy. Therefore, we use the component constructs in this paper.

{\bf Simulation 2.} This simulation examines the impact of the sketching techniques employed for $t_k$ and $\xi_j$ in \eqref{Hypothesis-space-deep} on the generalization performance. We compare the following two sketching techniques:
\begin{itemize}
    \item Random sketching: Set $t_0=-0.5$, and randomly sketch $n$ values from a uniform distribution within the interval $(-0.5, 0.5)$ and arrange them in ascending order to obtain the sequence $t_1 < \dots < t_n$. Randomly sketch $N$ inner weights $\{\xi_j\}_{j=1}^N$ according to the uniform distribution in the unit sphere $\mathbb{S}^{d-1}$.
    \item Component-based sketching: the values of $\{t_k\}_{k=0}^n$ are sampled at equal intervals from the interval $[-0.5, 0.5]$, and the inner weights $\{\xi_j\}_{j=1}^N$ are the center points of the regions of equal area partitioned by Leopardi's recursive zonal sphere partitioning procedure \cite{leopardi2006partition} \footnote{http://eqsp.sourceforge.net}.
\end{itemize}

In this simulation, the standard deviation of Gaussian noise $\delta$ varies from the set $\{0, 0.01, 0.1, 0.3, 0.5\}$. We fix the frequency parameter $J$ as 1, and vary the parameters $N\in \{10,20,30,\dots,400\}$, $n\in \{1,2,3,\dots,10\}$, and $\tau \in \{0.001,0.01,0.1,0.3,0.5\}$.
Figure \ref{fig_sketching_techs} illustrates a comparison of the two sketching techniques on testing RMSE under varying parameter values and noise levels. Specifically, the subfigures in the first three columns depict the relationship between RMSE and one parameter, while keeping the other two parameters of $(N, n, \tau)$ fixed at their optimal values. The subfigures in the last column compare the RMSE between the two sketching techniques when all parameters are set to their optimal values.
From the above results, we draw the following conclusions: 1) Under various noise levels, the relationship between the testing RMSE of the two sketching techniques and the respective parameter values consistently exhibits a similar overarching trend. Specifically, in scenarios with high noise, optimal models tend to favor smaller values for $N$ and $n$, as models of greater complexity can suffer from diminished generalization performance resulting from overfitting the noise. Furthermore, in high-noise levels, the influence of the parameter $\tau$ on the generalization error is comparatively minor.
2) When a single parameter undergoes variation, the testing RMSE of the component-based sketching is notably lower than that of the random sketching in most cases. Moreover, upon optimizing all three parameters, the component-based sketching exhibits a significantly lower RMSE, along with a reduced standard deviation, when processing training data across varying noise levels. These findings unequivocally underscore the superior effectiveness of component-based sketching compared to random sketching.

\begin{table}[t]
    \renewcommand\arraystretch{1.0}
 \begin{center}
		\caption{Comparisons of component-based sketching with varying frequency parameter $J$ on testing RMSE under different noise levels\label{tab_freq_role}}
  \vspace{0.02in}
\renewcommand\arraystretch{1.1}
		{\footnotesize
\begin{tabular}{
                        *{1}{p{0.85cm}<{\centering}}
                        *{1}{|p{0.5cm}<{\centering}}
                        *{1}{||p{0.75cm}<{\centering}}
 				    *{1}{p{0.88cm}<{\centering}}
                        *{1}{p{0.88cm}<{\centering}}
                        *{1}{p{0.88cm}<{\centering}}
                        *{1}{p{0.88cm}<{\centering}}
			           }
\toprule[1.5pt]
\multirow{2}{*}{Data} &  \multirow{2}{*}{\# $J$ } & \multicolumn{5}{c}{Standard deviation of noise} \\
\cline{3-7}
& & $\delta\hspace{-0.02in}=\hspace{-0.02in}0$ & $\delta\hspace{-0.02in}=\hspace{-0.02in}0.01$ &$\delta\hspace{-0.02in}=\hspace{-0.02in}0.1$ &$\delta\hspace{-0.02in}=\hspace{-0.02in}0.3$ &$\delta\hspace{-0.02in}=\hspace{-0.02in}0.5$\\
\hline
\hline
\multirow{10}{*}{\shortstack{$f_1$\\($\times10^{-3}$)}}
& \multirow{2}{*}{$J\hspace{-0.02in}=\hspace{-0.02in}1$}   & 1.07 &  4.35 & 18.79 & \textbf{19.56} & \textbf{20.98}  \\
&                             & (0.10) & (0.24) & (0.52) & (1.62) & (3.64)\\
& \multirow{2}{*}{$J\hspace{-0.02in}=\hspace{-0.02in}2$}   & 0.58 & 4.13 & 18.06 & 28.37 & 40.37 \\
&                             & (0.04) & (0.13) & (1.61) & (3.25) & (6.18) \\
& \multirow{2}{*}{$J\hspace{-0.02in}=\hspace{-0.02in}3$}   & 0.53 & 3.81 & 15.76 & 35.39 & 54.07  \\
&                             & (0.03) & (0.24) & (1.50) & (3.85) & (6.82)\\
& \multirow{2}{*}{$J\hspace{-0.02in}=\hspace{-0.02in}4$}   & \textbf{0.52} & 3.78 & 13.82 & 36.82 & 60.71 \\
&                             & (0.03) & (0.24) & (1.40) & (4.27) & (7.17) \\
& \multirow{2}{*}{$J\hspace{-0.02in}=\hspace{-0.02in}5$}   & \textbf{0.52} & \textbf{3.75} & \textbf{13.51} & 37.12 & 61.39  \\
&                             & (0.03) & (0.50) & (1.33) & (4.31) & (7.30)\\
\hline
\hline
\multirow{10}{*}{\shortstack{$f_2$\\($\times10^{-2}$)}}
& \multirow{2}{*}{$J\hspace{-0.02in}=\hspace{-0.02in}1$}   & 2.28 & 2.39 & 7.36 & 12.10 & 16.31 \\
&                             & (0.24) & (0.26) & (0.38) & (0.80) & (1.78)\\
& \multirow{2}{*}{$J\hspace{-0.02in}=\hspace{-0.02in}2$}   & \textbf{1.71} & 2.44 & 7.56 & 12.98 & 16.65 \\
&                             & (0.28) & (0.24) & (0.62) & (0.73) & (1.84) \\
& \multirow{2}{*}{$J\hspace{-0.02in}=\hspace{-0.02in}3$}   & 1.78 & 2.35 & \textbf{7.02} & 12.33 & 14.98  \\
&                             & (0.26) & (0.25) & (0.65) & (1.37) & (1.25)\\
& \multirow{2}{*}{$J\hspace{-0.02in}=\hspace{-0.02in}4$}   & 1.84 & 2.34 & 7.04 & 11.76 & \textbf{14.34} \\
&                             & (0.26) & (0.25) & (0.62) & (1.25) & (1.14) \\
& \multirow{2}{*}{$J\hspace{-0.02in}=\hspace{-0.02in}5$}   & 1.89 & \textbf{2.33} & 7.03 & \textbf{11.70} & 14.44  \\
&                             & (0.26) & (0.25) & (0.59) & (1.20) & (1.11)\\
\bottomrule[1.5pt]
\end{tabular}}
 \end{center}
 \vspace{-0.1in}
\end{table}

\begin{table*}[t]
\renewcommand\arraystretch{1.1}
 \begin{center}
 \caption{Comparisons of the testing RMSE for the seventeen methods under different noise levels\label{Comparison_RMSE_syn}}
{\footnotesize
\begin{tabular}{
                        *{1}{p{0.65cm}<{\centering}}
                        *{1}{|p{0.67cm}<{\centering}}
                        *{1}{||p{0.44cm}<{\centering}}
 				    *{4}{p{0.46cm}<{\centering}}
                        *{1}{p{0.001cm}<{\centering}}
                        *{1}{p{0.47cm}<{\centering}}
 				    *{7}{p{0.46cm}<{\centering}}
                        *{1}{p{0.001cm}<{\centering}}
                        *{1}{p{0.3cm}<{\centering}}
                        *{1}{p{0.6cm}<{\centering}}
 				    *{2}{p{0.5cm}<{\centering}}
			           }
\toprule[1.5pt]
\multirow{2}{*}{Data} &  \multirow{2}{*}{\shortstack{Noise\\level}} & \multicolumn{5}{c}{Gradient descent-related algorithms} &  & \multicolumn{8}{c}{Conjugate gradient descent-related algorithms} &  & \multicolumn{4}{c}{Deep nets}\\
\cline{3-7} \cline{9-16} \cline{18-21}
& & GD & GDA & GDM & GDX & RP & & LM & BR & BFG & CGB & CGF & CGP & OSS & SCG & & SGD & ADAM & RSNN & DSNN\\
\hline
\hline
\multirow{10}{*}{\shortstack{$f_1$\\(\scriptsize{$\times10^{-3}$})}}
& \multirow{2}{*}{$\delta\hspace{-0.02in}=\hspace{-0.02in}0$}   & 18.26 & ~5.92  & 17.90 & ~3.59 & ~3.59  &   & ~2.94  & ~7.57 & ~\textcolor{green}{\textbf{2.48}} & ~3.72 & ~2.74 & ~3.13 & ~3.06  & ~2.92 & & 17.87 & ~~18.69 & ~\textcolor{blue}{\textbf{1.70}} & ~\textcolor{red}{\textbf{0.52}}  \\
&                            & (1.09) & (0.53)  & (0.72) & (0.26) & (0.35)  &   & (0.13)  & (4.02) & (0.17) & (1.25) & (0.18) & (0.20) & (0.18)  & (0.26) &  & (0.72) & ~~(0.49) & (0.23) & (0.03) \\
\cline{2-21}
& \multirow{2}{*}{$\delta\hspace{-0.02in}=\hspace{-0.02in}0.01$}  & 16.59 & ~7.03  & 16.09 & ~5.72 &  ~6.16 &  &  ~5.46 & ~7.39  & ~\textcolor{blue}{\textbf{5.07}} & ~5.19 & ~\textcolor{green}{\textbf{5.09}} & ~6.94 & ~5.53 &  ~5.49 &  & 18.05 & ~~18.76 & ~6.55 & ~\textcolor{red}{\textbf{3.75}}  \\
&                            & (1.36) & (0.69)  & (0.78) & (0.40) & (0.49)  &   & (0.28)  & (2.07) & (0.26) & (0.42) & (0.25) & (5.48) & (0.35)  & (0.37) &  & (0.47) & ~~(0.42) & (0.70) & (0.50) \\
\cline{2-21}
& \multirow{2}{*}{$\delta\hspace{-0.02in}=\hspace{-0.02in}0.1$}  & 19.77 & 22.45  & 20.22 & 19.97 &  21.16 &   &  20.96 & \textcolor{green}{\textbf{18.96}} & 20.65 & 20.42 & 21.72 & 20.47 & 19.78 & 20.59 &  & 19.35 & ~~\textcolor{blue}{\textbf{18.81}} & 19.11 & \textcolor{red}{\textbf{13.51}} \\
&                            & (1.31) & (2.07)  & (1.21) & (1.04) & (1.34)  &   & (1.64)  & (0.55) & (1.48) & (1.20) & (1.55) & (1.06) & (1.17)  & (1.22) &  & (0.51) & ~~(0.53) & (2.14) & (1.33) \\
\cline{2-21}
& \multirow{2}{*}{$\delta\hspace{-0.02in}=\hspace{-0.02in}0.3$}  & 29.45 & 58.31 & 33.38 & 36.09 & 35.15  &   & 37.78  & 20.08 & 36.72 & 37.40 & 33.67 & 37.15 & 33.74  & 34.62 &  & 20.53 & ~~\textcolor{green}{\textbf{19.60}} & \textcolor{red}{\textbf{19.24}} & \textcolor{blue}{\textbf{19.56}} \\
&                            & (7.21) & (14.32)  & (6.90) & (7.51) & (2.79)  &   & (6.48)  & (1.87) & (6.33) & (5.41) & (7.20) & (4.49) & (4.71)  & (6.56) &  & (1.55) & ~~(1.69) & (0.88) & (1.62) \\
\cline{2-21}
& \multirow{2}{*}{$\delta\hspace{-0.02in}=\hspace{-0.02in}0.5$}   & \textcolor{green}{\textbf{21.05}} &  67.93 & 46.05 & 58.89 & 55.30  &   & 65.99  & 23.63 & 57.65 & 50.17 & 55.91 & 51.12 &  47.90 & 46.33 &  & 22.35 & ~~26.91 & \textcolor{red}{\textbf{19.85}} & \textcolor{blue}{\textbf{20.98}} \\
&                            & (3.64) & (11.56)  & (19.12) & (9.88) & (11.31)  &   & (8.09)  & (5.56) & (16.05) & (14.63) & (15.47) & (7.88) & (10.32)  & (9.25) &  & (3.37) & ~~(9.65) & (2.06) & (3.64) \\
\hline
\hline
\multirow{10}{*}{\shortstack{$f_2$\\(\scriptsize{$\times10^{-2}$)}}}
& \multirow{2}{*}{$\delta\hspace{-0.02in}=\hspace{-0.02in}0$}   & 10.10 & ~9.22  & ~9.47 & ~8.75 & ~8.53  & & ~6.75  & 14.92  & ~6.96 & 10.63 & ~6.95 & ~7.40 &  ~7.44 & ~7.68 &  & ~\textcolor{blue}{\textbf{2.88}} & ~~16.58 & ~\textcolor{green}{\textbf{4.79}} & ~\textcolor{red}{\textbf{1.71}} \\
&                            & (1.48) & (1.12)  & (1.79) & (0.71) & (0.38)  &   & (0.45)  & (1.65) & (0.52) & (8.59) & (0.54) & (0.37) & (0.45)  & (0.61) &  & (0.98) & ~~(5.99) & (7.63) & (0.28) \\
\cline{2-21}
& \multirow{2}{*}{$\delta\hspace{-0.02in}=\hspace{-0.02in}0.01$}   & 10.48 & ~9.91  & ~9.86 & ~8.70 & ~8.37  &   & ~6.85  & 14.91 & ~7.08 & ~7.84 & ~6.92 & ~7.61 & ~7.70  & ~7.58 & & ~\textcolor{blue}{\textbf{2.87}} & ~~16.29 & ~\textcolor{green}{\textbf{5.37}} & ~\textcolor{red}{\textbf{2.33}} \\
&                            & (1.99) & (1.09)  & (1.78) & (0.61) & (0.67)  &   & (0.43)  & (1.66) & (0.50) & (0.71) & (0.34) & (0.58) & (0.61)  & (0.60) &  & (1.00) & ~~(6.21) & (8.44) & (0.25) \\
\cline{2-21}
& \multirow{2}{*}{$\delta\hspace{-0.02in}=\hspace{-0.02in}0.1$}   & 11.71 & 10.85  & 10.97 & 10.31 & 11.19  & & 10.24  & 14.95  & ~9.93 & 10.16 & ~\textcolor{green}{\textbf{9.76}} & ~9.90 & 10.65 & 10.24  & & ~\textcolor{red}{\textbf{3.45}} & ~~14.84 & 10.93 & ~\textcolor{blue}{\textbf{7.02}}  \\
&                            & (1.77) & (1.05)  & (1.55) & (0.75) & (0.90)  &   & (0.74)  & (1.69) & (0.72) & (0.73) & (0.88) & (0.66) & (1.96)  & (0.67) &  & (0.98) & ~~(6.75) & (1.47) & (0.64) \\
\cline{2-21}
& \multirow{2}{*}{$\delta\hspace{-0.02in}=\hspace{-0.02in}0.3$}   & 15.39 & \textcolor{green}{\textbf{14.32}}  & 16.36 & 15.14 & 16.29  &   & 16.35  & 15.30 & 15.36 & 15.38 & 15.03 & 15.39 & 15.02  & 15.25 &  & ~\textcolor{red}{\textbf{7.12}} & ~~15.86 & 15.27 & \textcolor{blue}{\textbf{11.70}} \\
&                            & (1.42) & (1.13)  & (2.43) & (0.96) & (0.86)  &   & (0.98)  & (1.71) & (1.11) & (1.43) & (1.61) & (1.46) & (1.17)  & (1.20) &  & (1.38) & ~~(4.70) & (2.43) & (1.20) \\
\cline{2-21}
& \multirow{2}{*}{$\delta\hspace{-0.02in}=\hspace{-0.02in}0.5$}   & 17.85 & 17.43  & 19.13 & 17.63 & 18.04 &  & 17.42  &  \textcolor{green}{\textbf{15.92}} & 19.26 & 17.35 & 17.54 & 17.30 & 17.58  & 17.31 &  & \textcolor{red}{\textbf{13.03}} & ~~16.89 & 17.59 & \textcolor{blue}{\textbf{14.34}} \\
&                            & (2.07) & (1.46)  & (2.14) & (1.83) & (1.71)  &   & (1.73)  & (1.75) & (1.65) & (1.41) & (1.76) & (1.38) & (1.42)  & (1.67) &  & (2.14) & ~~(4.07) & (1.58) & (1.14) \\
\bottomrule[1.5pt]
\end{tabular}}
\end{center}
\end{table*}


\begin{figure*}[t]
    \centering
    \subfigcapskip=-5pt
    \subfigure[$\delta=0$]{\includegraphics[width=5.9cm,height=3.2cm]{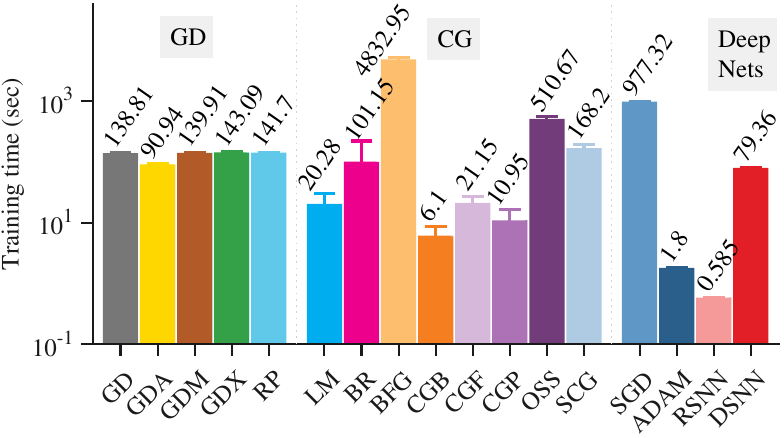}}
    \subfigure[$\delta=0.01$]{\includegraphics[width=5.9cm,height=3.2cm]{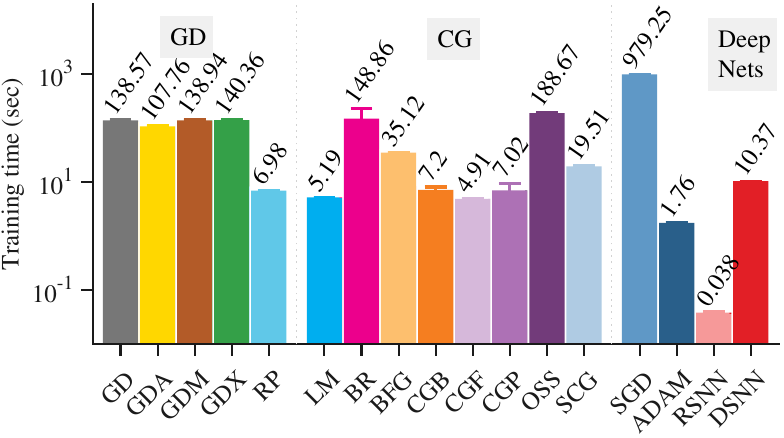}}
    \subfigure[$\delta=0.1$]{\includegraphics[width=5.9cm,height=3.2cm]{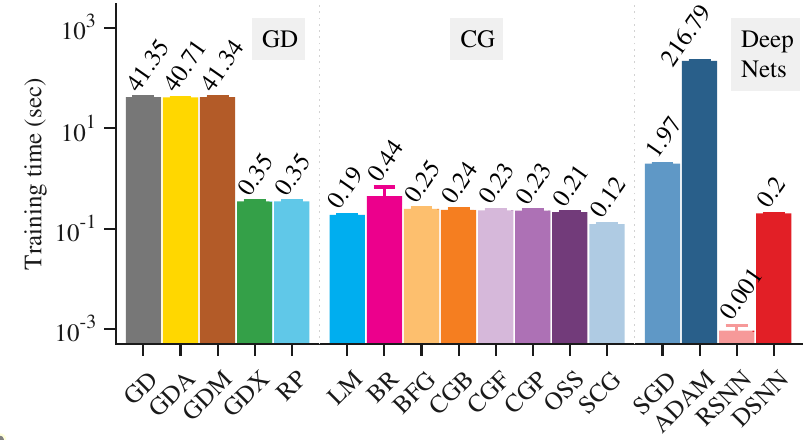}}\\
    \subfigure[$\delta=0.3$]{\includegraphics[width=5.9cm,height=3.2cm]{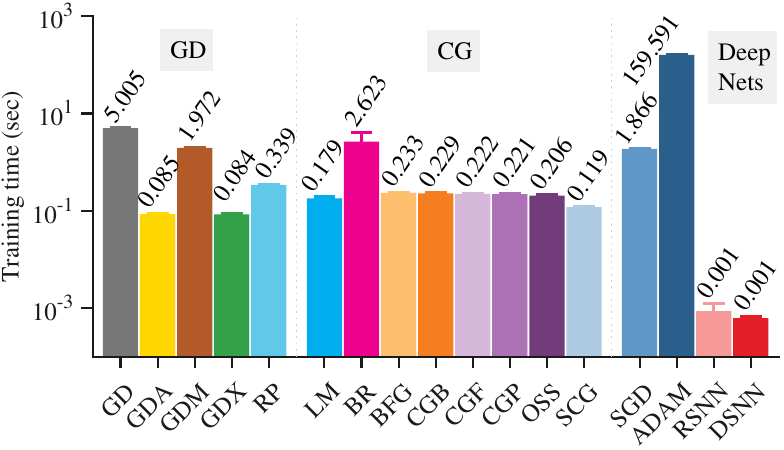}}\hspace{0.05in}
    \subfigure[$\delta=0.5$]{\includegraphics[width=5.9cm,height=3.2cm]{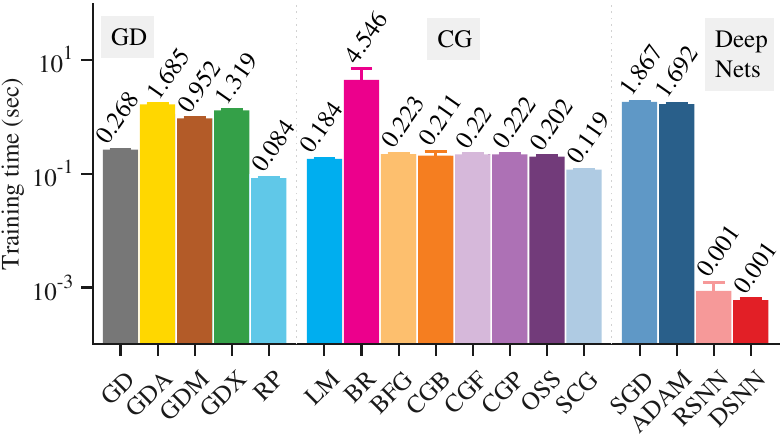}}
    \caption{Comparisons of the training time for the seventeen methods on data generated by $f_1$ with different noise levels}\label{fig_trtime_poly}
    \vspace{-0.1in}
\end{figure*}

\begin{figure*}[t]
    \centering
    \subfigcapskip=-5pt
    \subfigure[$\delta=0$]{\includegraphics[width=5.9cm,height=3.2cm]{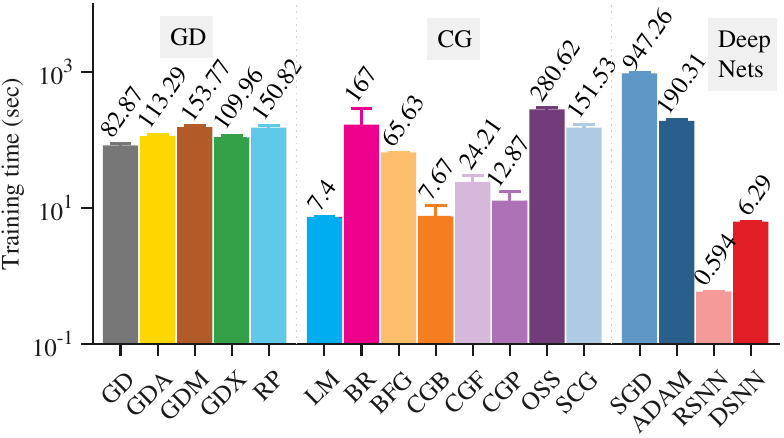}}
    \subfigure[$\delta=0.01$]{\includegraphics[width=5.9cm,height=3.2cm]{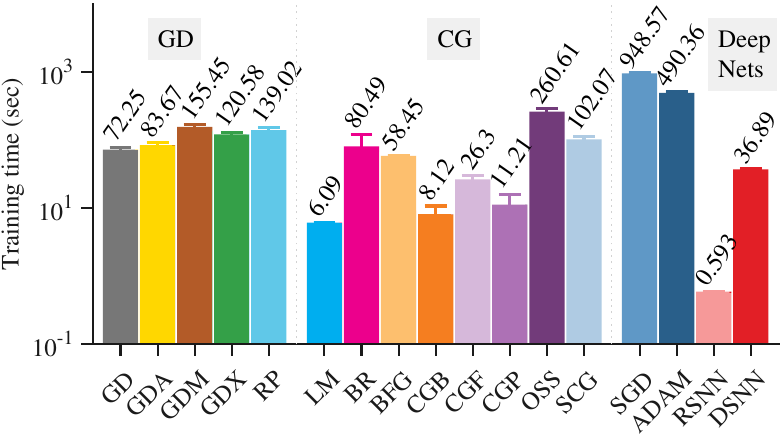}}
    \subfigure[$\delta=0.1$]{\includegraphics[width=5.9cm,height=3.2cm]{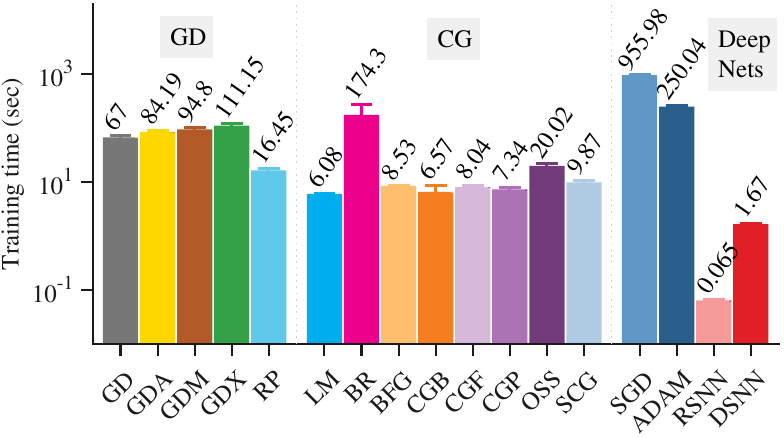}}\\
    \subfigure[$\delta=0.3$]{\includegraphics[width=5.9cm,height=3.2cm]{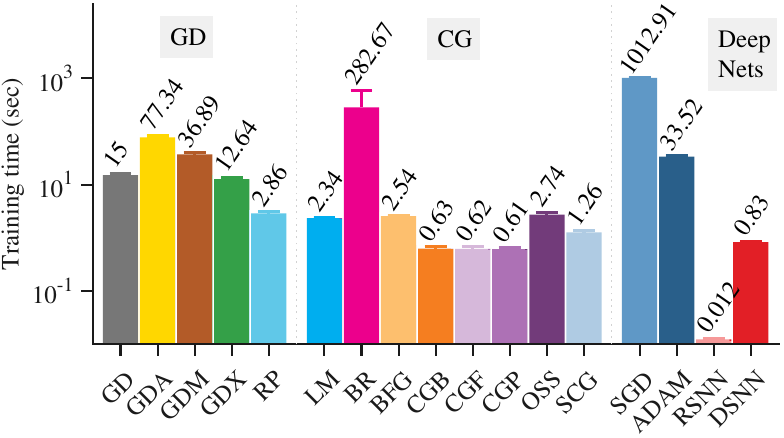}}\hspace{0.05in}
    \subfigure[$\delta=0.5$]{\includegraphics[width=5.9cm,height=3.2cm]{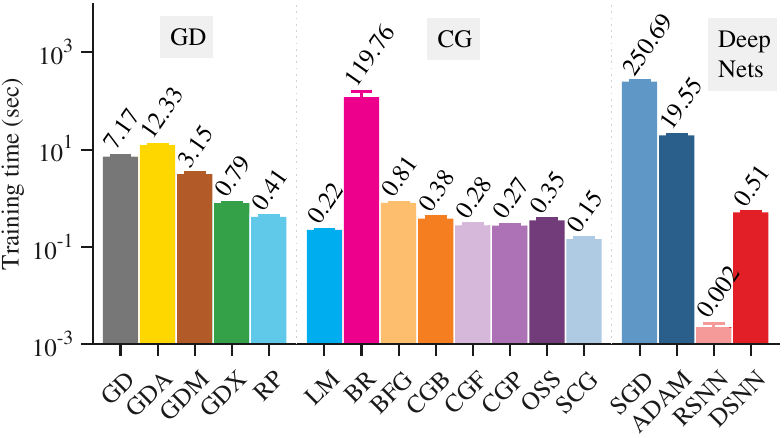}}
    \caption{Comparisons of the training time for the seventeen methods on data generated by $f_2$ with different noise levels}\label{fig_trtime_f21}
    \vspace{-0.1in}
\end{figure*}

\begin{figure*}[t]
    \centering
    \subfigcapskip=-5pt
    \subfigure{\includegraphics[width=5.85cm,height=3.85cm]{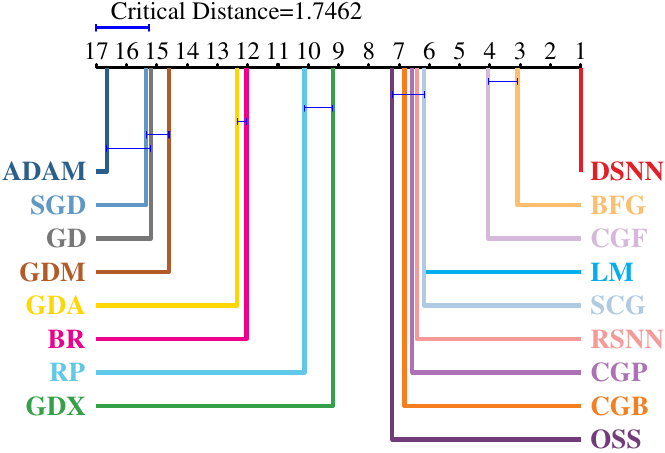}}\hspace{0.06in}
    \subfigure{\includegraphics[width=5.85cm,height=3.85cm]{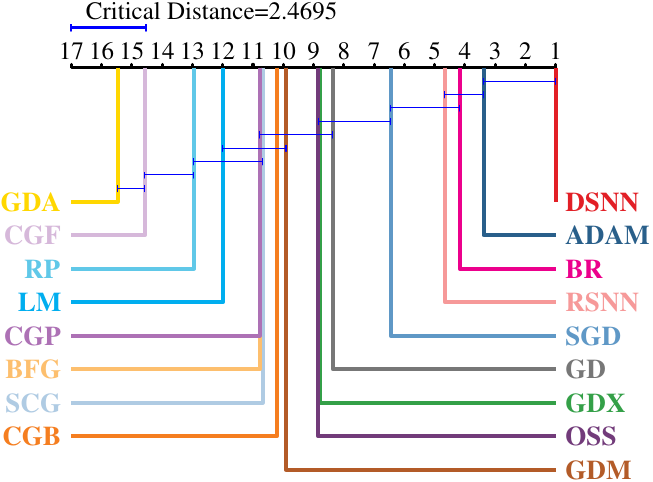}}\hspace{0.06in}
    \subfigure{\includegraphics[width=5.85cm,height=3.85cm]{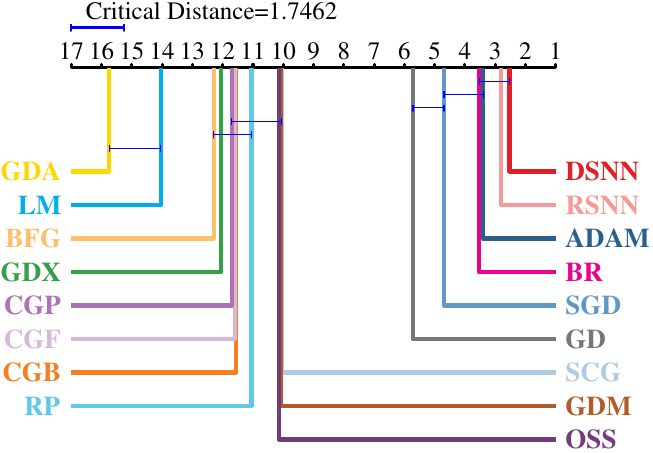}}\\
    \setcounter{subfigure}{0}
    \subfigure[Low noise data ($\delta\leq 0.01$)]{\includegraphics[width=5.85cm,height=3.85cm]{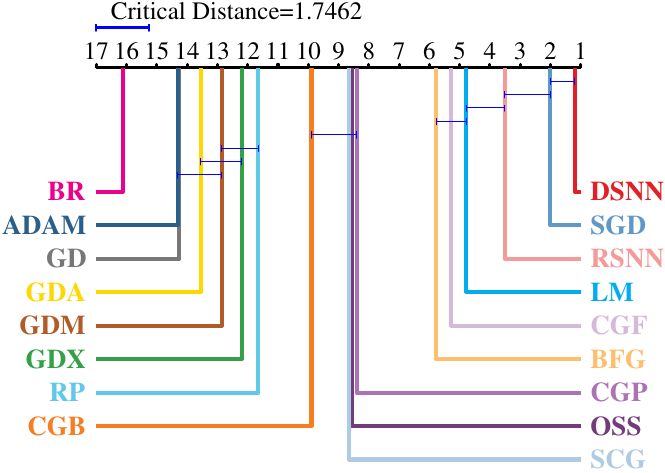}}\hspace{0.06in}
    \subfigure[Middle noise data ($\delta=0.1$)]{\includegraphics[width=5.85cm,height=3.85cm]{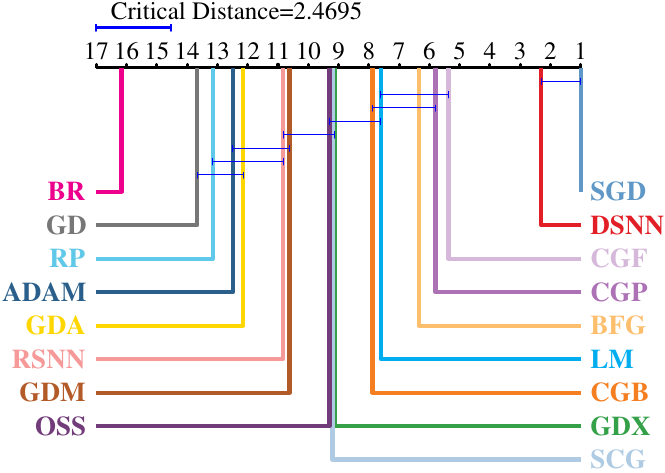}}\hspace{0.06in}
    \subfigure[high noise data ($\delta\geq 0.3$)]{\includegraphics[width=5.85cm,height=3.85cm]{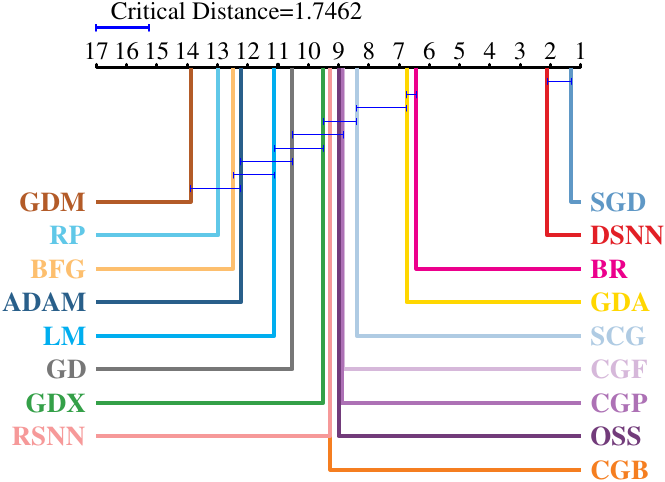}}
    \caption{Post hoc statistical tests of the seventeen methods on the two synthetic datasets}\label{fig_cdplot_syn}
    \vspace{-0.1in}
\end{figure*}

\begin{table}[t]
    \renewcommand\arraystretch{1.0}
 \begin{center}
		\caption{\small Overall description of the real-world datasets}\label{tab_data_des}
		\small
	\begin{tabular}{*{1}{p{2.1cm}<{ }}
		            *{1}{p{1.2cm}<{\centering}}	
                    *{1}{p{1.5cm}<{\centering}}	
                    *{1}{p{1.3cm}<{\centering}}	
	               }
		\toprule[1.25pt]
       Datasets  &Attributes    &Training $\#$ &Testing $\#$ \\
       \hline
       \hline
		\texttt{WineRed}          &11   & 1279  & 320 \\
		\texttt{WineWhite}        &11 &3918   & 980  \\
		\texttt{Abalone}           &8   &3342 &835 \\
            \texttt{CCPP}             &4   &7654   &1914 \\
            \texttt{Energy}            &17   &17520   &17520 \\
            \texttt{Quality}             &65   &14592   &14592 \\
            \bottomrule[1.25pt]
	\end{tabular}
 \end{center}
 \vspace{-0.1in}
\end{table}

\begin{table*}[t]
\caption{Comparisons of the testing RMSE for the seventeen methods on real-world datasets\label{Comparison_RMSE_real}}
\renewcommand\arraystretch{1.1}
{\footnotesize
\begin{tabular}{
                        *{1}{p{1.2cm}<{\centering}}
                        *{1}{||p{0.47cm}<{\centering}}
 				    *{4}{p{0.47cm}<{\centering}}
                        *{1}{p{0.001cm}<{\centering}}
                        *{1}{p{0.48cm}<{\centering}}
 				    *{7}{p{0.48cm}<{\centering}}
                        *{1}{p{0.001cm}<{\centering}}
                        *{1}{p{0.5cm}<{\centering}}
                        *{1}{p{0.65cm}<{\centering}}
 				    *{2}{p{0.55cm}<{\centering}}
			           }
\toprule[1.5pt]
\multirow{2}{*}{Data}  & \multicolumn{5}{c}{Gradient descent-related algorithms} &  & \multicolumn{8}{c}{Conjugate gradient descent-related algorithms} &  & \multicolumn{4}{c}{Deep nets}\\
\cline{2-6} \cline{8-15} \cline{17-20}
& GD & GDA & GDM & GDX & RP & & LM & BR & BFG & CGB & CGF & CGP & OSS & SCG & & SGD & ADAM & RSNN & DSNN\\
\hline
\hline
\multirow{2}{*}{\shortstack{\texttt{WineRed}\\({$\times10^{-1}$})}}
   & ~6.73 & ~6.52  & ~7.13 & ~6.64 & ~6.55 &  & ~6.65  &  ~\textcolor{green}{\textbf{6.37}}  & ~7.47 & ~6.49 & ~6.52 & ~6.48 & ~6.55  & ~6.50 &  & ~\textcolor{red}{\textbf{6.33}} & ~6.39 & ~6.67 &  ~\textcolor{blue}{\textbf{6.36}} \\
   & (0.63) & (0.40)  & (0.95) & (0.41) &  (0.31) &   & (0.29)  & (0.28) & (2.50) & (0.31) & (0.25) & (0.26) &  (0.35) & (0.32) & & (0.28) & (0.26) &  (0.55) &  (0.27) \\
\hline
\multirow{2}{*}{\shortstack{\texttt{WineWhite}\\({$\times10^{-1}$})}}
   & \multirow{2}{*}{---} &  ~7.15 & \multirow{2}{*}{---} & ~7.34 &  ~7.23 &   & ~7.33  & ~\textcolor{green}{\textbf{7.02}} & ~7.76 & ~7.26 & ~7.55 & ~7.17 & ~7.46  & ~7.71 & & ~\textcolor{red}{\textbf{6.86}} & ~7.12 & ~7.66 & ~\textcolor{blue}{\textbf{6.98}}  \\
   &  & (0.39)  & & (0.53)  & (0.27)  &   & (0.53)  & (0.23) & (1.31) & (0.36) & (0.91) & (0.29) & (0.50)  & (1.10) & & (0.22) & (0.16) & (0.34) & (0.24)  \\
\hline
\multirow{2}{*}{\texttt{Abalone}}
    & \multirow{2}{*}{---}  & ~\textcolor{blue}{\textbf{2.12}}  & \multirow{2}{*}{---} & ~2.19 &  ~2.15 &   & ~2.14  & ~\textcolor{blue}{\textbf{2.12}} & ~2.26 & ~2.27 & ~2.15 & ~\textcolor{green}{\textbf{2.13}} &  ~2.15 & ~2.14 & & ~\textcolor{blue}{\textbf{2.12}} & ~2.15 & ~2.26 & ~\textcolor{red}{\textbf{2.11}}  \\
    &  & (0.06)  & & (0.14)  & (0.09)  &   & (0.09)  & (0.07) &  (0.33) & (0.24) & (0.10) & (0.09) & (0.09) & (0.05)  &  & (0.08) & (0.07) & (0.20) & (0.06)   \\
\hline
\multirow{2}{*}{\texttt{CCPP}}
   & \multirow{2}{*}{---} & ~3.96  & \multirow{2}{*}{---} & ~3.96 & ~3.92  &   & ~\textcolor{red}{\textbf{3.81}}  & ~3.96 & ~3.98 & ~3.93 & ~\textcolor{green}{\textbf{3.84}} & ~3.91 & ~3.85  & ~3.87 & & \multirow{2}{*}{---}  & ~3.87 & ~4.16 & ~\textcolor{blue}{\textbf{3.82}}  \\
   &  & (0.15) &  & (0.14)  & (0.16)  &  & (0.15) & (0.14) & (0.17) & (0.16) & (0.17)  & (0.16)  & (0.14) & (0.13) &  &  & (0.13) & (1.91) & (0.15)\\
\hline
\multirow{2}{*}{\shortstack{\texttt{Energy}\\({$\times10^{-2}$})}}
   & 72.92 & 10.96  & 51.28 & 10.90 &  ~7.65 &   & ~\textcolor{green}{\textbf{3.41}}  & ~\textcolor{red}{\textbf{2.65}} & 19.28 & ~7.57 & ~6.38 & ~9.35 & ~4.97 & ~5.01 & & ~3.94 & ~5.78 & ~6.42 & ~\textcolor{blue}{\textbf{2.66}}  \\
   & (46.08) & (1.25)  & (53.46) & (2.34) & (0.83)  &   & (0.28)  & (0.20) & (38.85) & (1.20) & (0.32) & (1.86) & (0.59)  & (0.83) & & (0.14) & (0.62) & (1.47) &  (0.08) \\
\hline
\multirow{2}{*}{\shortstack{\texttt{Quality}\\({$\times10^{-2}$})}}
    & ~9.84 & ~9.51  & ~9.86 & ~9.50 & ~9.51 &  & ~9.60 & ~\textcolor{red}{\textbf{9.41}}  & ~9.49 & ~9.54 & ~\textcolor{green}{\textbf{9.47}} & ~9.55 & ~9.49 & ~9.50  &  & ~9.59 & ~9.64 & ~11.85 & ~\textcolor{blue}{\textbf{9.43}}   \\
    & (0.09) & (0.07)  & (0.06) & (0.09) &  (0.06) &  &  (0.07)  & (0.06) & (0.07) & (0.13) & (0.06) & (0.12) & (0.08)  & (0.10) & & (0.06) & (0.06) & (4.07) & (0.06)  \\
\bottomrule[1.5pt]
\end{tabular}}
\end{table*}

{\bf Simulation 3.} To explore the importance of incorporating locality information from the frequency domain, we assign the frequency domain parameter $J$ with values ranging from $\{1, 2, 3, 4, 5\}$, while maintaining the same range of values for other parameters as in Simulation 2. Note that a $J$ value exceeding 1 signifies that the model integrates frequency domain locality information, whereas a $J$ value of 1 indicates the absence of such integration within the model. Figure \ref{fig_freq_role} demonstrates the influence of the parameter $\tau$ and $J$ on accuracy when parameters $N$ and $n$ are set to their optimal values under different noise levels. Additionally, Table \ref{tab_freq_role} presents a comparative analysis, highlighting the influence of varying $J$ on accuracy, given that $N$, $n$, and $\tau$ have all been tuned to their optimal values. In Table \ref{tab_freq_role}, the best testing RMSE for each noise level is in bold, and the values in brackets are standard deviations.
Based on the results, we have the following observations: 1) At a given noise level, the relationship between the generalization error and the value of the parameter $J$ exhibits a consistent similarity across different $\tau$ values, significantly indicating that the frequency parameter $J$ plays a similar role under varying values of $\tau$. Specifically, an increase in the value of the parameter $J$ conspicuously enhances the model's accuracy in most cases, thereby underscoring the crucial importance of integrating local frequency domain information. 2) When the noise level of data generated by $g_1$ is high (i.e., $\delta\geq 0.3$), the RMSE tends to increase as the value of $J$ grows. In contrast, this phenomenon is not present in the data generated by $g_2$. This disparity is due to the fact that function $g_1$ possesses superior smoothness compared to function $g_2$. In scenarios with high noise, a higher value of $J$ increases the complexity of the model, fitting the noise information to some extent, which ultimately results in overfitting. This observation simultaneously highlights the necessity of introducing frequency domain locality information when fitting data with complex distributions.

{\bf Simulation 4.} This simulation offers a comparative analysis of the proposed method against the aforementioned approaches, evaluating both accuracy and training time.
For fully connected neural networks, the settings and ranges of parameters are identical to those employed in Simulation 1, with the exception of the maximum number of epochs, which is set to 50,000.
For both RSNN and the proposed method, the parameters $n$, $N$, and $\tau$ are selected from the sets $\{1,2,\dots,10\}$, $\{10,20,\dots,400\}$, and $\{0.001, 0.01, 0.1, 0.3, 0.5\}$, respectively; the depth of RSNN and the parameter $J$ in the proposed method are selected from the set $\{1,2,3,4,5\}$.

The testing RMSEs of the seventeen methods with the best parameters, for varying noise levels, are presented in Table \ref{Comparison_RMSE_syn}, where the bold red, blue, and green numbers respectively represent the top three average accuracies under the same noise level. A comparative analysis of their respective training times is illustrated in Figures \ref{fig_trtime_poly}-\ref{fig_trtime_f21}, where the symbols ``GD'' and ``CG'' represent gradient descent-related algorithms and conjugate gradient descent-related algorithms for shallow fully connected neural networks, respectively,
and ``Deep Nets'' include deep fully connected neural networks, RSNN, and our proposed method. Furthermore, a nonparametric posthoc statistical test \cite{Demsar2006,Yu2017} is conducted to investigate the comprehensive generalization performance of these methods. Here, a critical difference (CD) diagram is utilized to demonstrate the overall statistical comparisons. The results are shown in Figure \ref{fig_cdplot_syn}, where the data is grouped into low-noise (i.e., $\delta\leq0.01$), medium-noise (i.e., $\delta=0.1$), and high-noise (i.e., $\delta\geq0.3$) categories, and the statistical tests are performed independently in each group. Notably, the right side of the axis in Figure \ref{fig_cdplot_syn} indicates the highest rank, which corresponds to the best performance.

From the above results, we can draw the following observations:
1) Deep neural networks typically outperform shallow networks in terms of accuracy, particularly when dealing with complex data distributions, such as those generated by $f_2$. This is attributed to their ability to more effectively uncover the intrinsic nonlinear structural relationships within such data.
2) Among shallow networks, the conjugate gradient-based approach often surpasses gradient-based methods in both computational accuracy and training time.
3) In deep fully connected networks, the Adam algorithm often suffers from suboptimal performance due to gradient vanishing issues during optimization. In contrast, the SGD algorithm exhibits stable and superior generalization performance across various noisy datasets, yet it incurs the longest computational time among these methods.
4) RSNN utilizes sketching technology to transform the neural network training problem into a least-squares problem that is independent of data dimensions and quantity, thus avoiding the complex nonlinear optimization process of neural networks. Consequently, RSNN possesses the shortest training time in almost all cases. However, its generalization performance tends to be relatively poor when handling data with more complex distributions, such as data generated by $f_2$ with high levels of noise.
5) Compared to RSNN, our proposed method employs a component-based sketching technique instead of random sketching, resulting in a more stable generalization performance with a reduced standard deviation. Furthermore, our method incorporates locality information in the frequency domain, which, while slightly increasing training complexity, demonstrates superior generalization performance when handling data with complex distributions. In summary, our proposed method achieves the best generalization performance in most cases, while its training time is comparable to or even more efficient than that of shallow networks.

\begin{figure*}[t]
    \centering
    \subfigcapskip=-5pt
    \subfigure[\texttt{WineRed}]{\includegraphics[width=5.9cm,height=3.2cm]{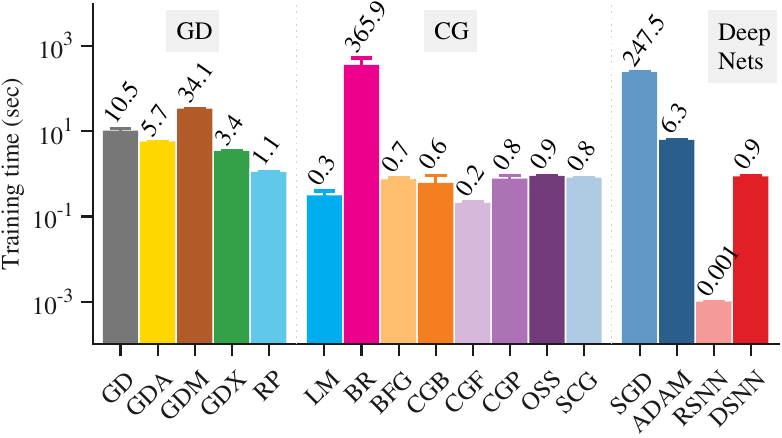}}
    \subfigure[\texttt{WineWhite}]{\includegraphics[width=5.9cm,height=3.2cm]{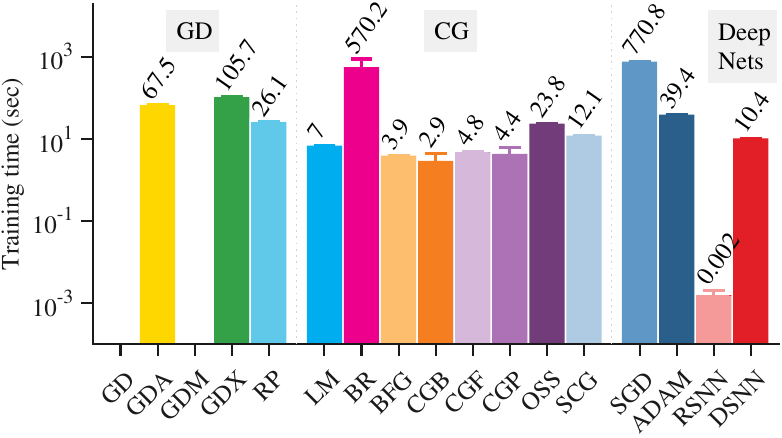}}
    \subfigure[\texttt{Abalone}]{\includegraphics[width=5.9cm,height=3.2cm]{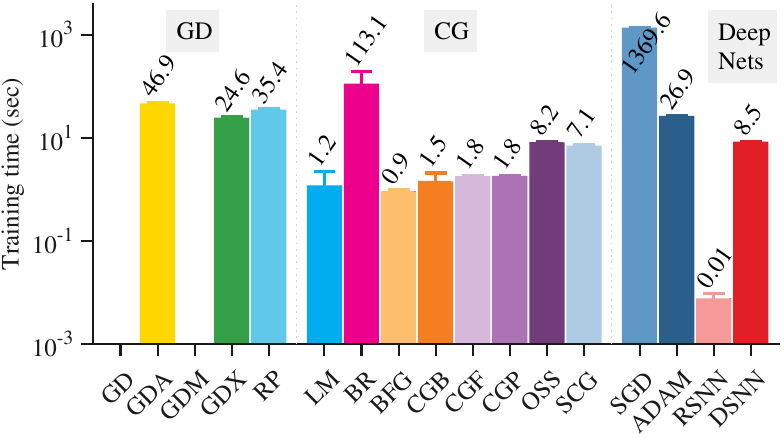}}\\
    \subfigure[\texttt{CCPP}]{\includegraphics[width=5.9cm,height=3.2cm]{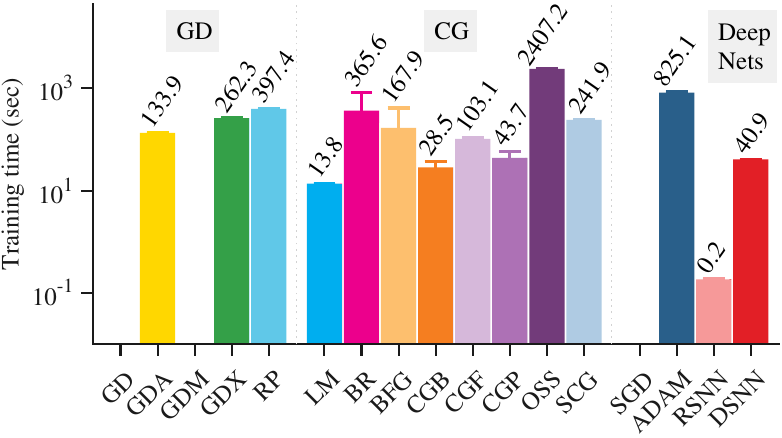}}
    \subfigure[\texttt{Energy}]{\includegraphics[width=5.9cm,height=3.2cm]{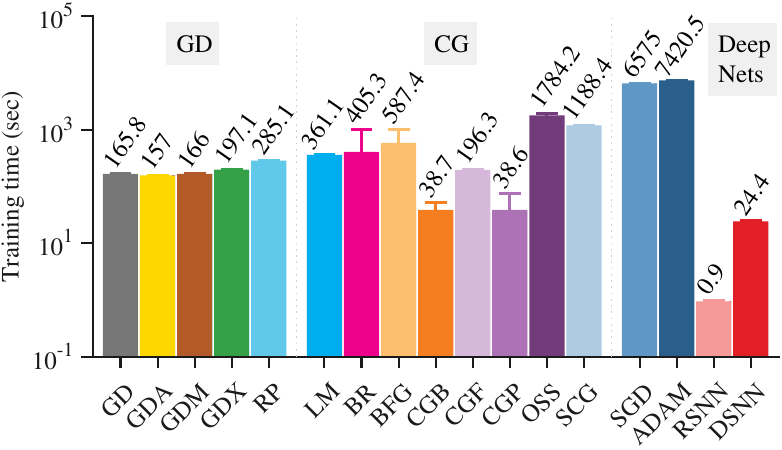}}
    \subfigure[\texttt{Quality}]{\includegraphics[width=5.9cm,height=3.2cm]{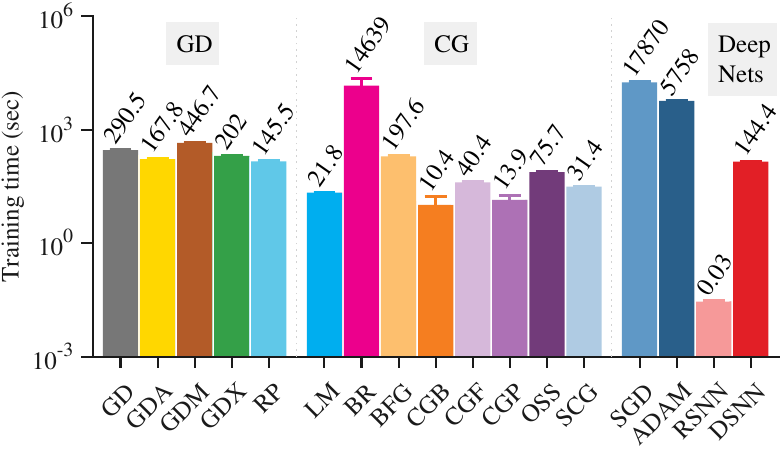}}
    \caption{Comparisons of the training time for the seventeen methods on real-world datasets}\label{fig_trtime_realdata}
    \vspace{-0.1in}
\end{figure*}

\subsection{{ Real-world} Benchmark Datasets Experiments}

In this part, we carry out experiments on a family of real-world benchmark datasets to demonstrate the effectiveness of the proposed method. The general description of these datasets is presented in Table \ref{tab_data_des}, in which ``$\#$'' represents the number of samples. The datasets are as follows:

$\bullet$ \emph{Wine Quality\footnote{\url{https://archive.ics.uci.edu/dataset/186/wine+quality}}:} The Wine Quality dataset comprises two subsets, corresponding to red and white wine samples from northern Portugal, denoted as \texttt{WineRed} and \texttt{WineWhite}. The purpose of this dataset is to model the relationship between wine quality and 11 physicochemical indicators, including fixed acidity, volatile acidity, citric acid, pH, alcohol, etc. Wine quality scores are determined by the median score of at least three sensory evaluators through blind tasting, ranging from 0 (very poor) to 10 (excellent).

$\bullet$ \emph{Abalone\footnote{\url{https://archive.ics.uci.edu/dataset/1/abalone}}:} This dataset contains 4,177 samples of abalones, with each sample providing detailed physical characteristics, such as sex, length, diameter, height, number of growth rings, etc. Typically, the age of an abalone is calculated by adding 1.5 to the number of growth rings (in years). The utilization of this dataset aims to model the relationship between the age of abalones and their physical characteristics, thus providing scientific support for abalone aquaculture and fisheries management.

$\bullet$ \emph{Combined Cycle Power Plant\footnote{\url{https://archive.ics.uci.edu/dataset/294/combined+cycle+power+plant}}:} The dataset contains 9,568 samples collected from a combined cycle power plant operating at full load from 2006 to 2011, denoted as \texttt{CCPP}. Each sample includes hourly average ambient variables such as temperature, ambient pressure, relative humidity, and exhaust vacuum, as well as the target variable, net hourly electrical energy output. The primary objective is to utilize these data to establish a relationship between the power plant's hourly net electrical power output and these ambient variables, with the ultimate goal of optimizing energy production and distribution.

$\bullet$ \emph{Steel Industry Energy Consumption\footnote{\url{https://archive.ics.uci.edu/dataset/851/steel+industry+energy+consumption}}:} This dataset encompasses energy consumption data collected by DAEWOO Steel Co. Ltd in Gwangyang, South Korea, at a 15-minute interval over a span of 365 days,
denoted as \texttt{Energy}.
The features include lagging and leading current reactive power, lagging and leading current power factor, carbon dioxide emissions, and load types. The aim of this dataset is to analyze the relationship between energy consumption and these features, thereby providing robust support for the formulation and optimization of effective energy consumption policies.

$\bullet$ \emph{Production Quality\footnote{\url{https://www.kaggle.com/datasets/podsyp/production-quality}}:} This dataset collects data from a roasting machine that spans from January 2015 to May 2018, denoted as \texttt{Quality}. The roasting machine contains five chambers of equal size, and each chamber is equipped with three temperature sensors positioned at different locations. Additionally, the roasting machine is equipped with layer height and humidity sensors to measure the layer height and humidity when raw materials enter the machine. These sensors collect data every minute. Product quality is assessed through hourly sampling in the laboratory, reflecting the roasting quality of the previous hour. Given the inconsistency in sampling frequencies between product quality data and roasting sensor data, we divide the roasting sensor data within an hour into four temporal windows for averaging, and then combine the averaged data from these four windows to represent the roasting features for that hour. The goal of this dataset is to establish a relationship between the product quality produced by the roasting machine and the features obtained from sensor data,
aiming to effectively enhance and optimize the machine's production performance.

Before discussing the experiments, it is important to clarify some implementation details: 1) For each dataset, we first apply min-max normalization to each dimension, excluding the target variable, and then map the normalized data to the ball centered at the origin with a radius of 1/2. 2) For the \texttt{Energy} and \texttt{Quality} datasets, we perform a logarithmic transformation on the target variable $y$, with the specific transformation formula being $\hat{y}=\log(1+y)$.
3) For the \texttt{Energy} and \texttt{Quality} datasets, we randomly split the data into equal halves for training and testing. For the other datasets, we randomly allocate 80\% of the samples for training and the remaining 20\% for testing. 10 independent sets of training and testing samples are generated by performing 10 random splits on each dataset.
4) For both RSNN and the proposed method, the parameters $N$ and $\tau$ are selected from the sets $\{10,20,\dots,200\}$ and $\{0.01,0.1,0.3,0.5\}$ respectively, and the ranges of other parameters remain the same as those in Simulation 4 of Section \ref{Syn_ex}.

\begin{figure}[t]
    \centering
  \includegraphics[width=6.5cm,height=4.2cm]{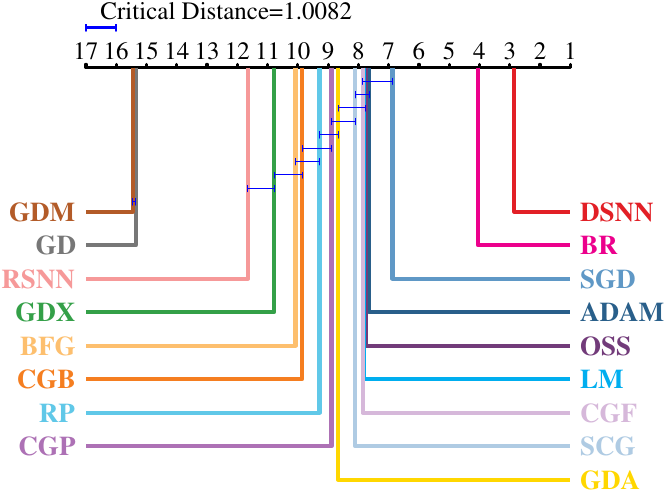}
    \caption{Post hoc statistical tests of the seventeen methods on real-world datasets}\label{fig_cdplot_real}
    \vspace{-0.1in}
\end{figure}

The average RMSE along with the standard deviation for the compared methods is reported in Table \ref{Comparison_RMSE_real}, wherein the bold red, blue, and green numbers indicate the three lowest RMSEs. Figure \ref{fig_trtime_realdata} depicts the corresponding average training time. It is worth noting that the methods GD, GDM, and SGD have no records of RMSE and training time for some datasets, such as \texttt{WineWhite}, \texttt{Abalone}, and \texttt{CCPP}, due to the occurrence of the exploding gradients problem. To further illustrate the advantages of the proposed method, we additionally provide a CD diagram for the six datasets, as shown graphically in Figure \ref{fig_cdplot_real}. From the results, we can conclude
the following assertions: 1) Deep neural networks, due to their proficiency in precisely capturing complex nonlinear relationships within data, demonstrate stable and outstanding performance on real-world datasets, compared to shallow neural networks. 2) Among shallow networks, BR stands out as particularly excellent, surpassing even deep networks in accuracy on some datasets and ranking second in the posthoc statistical test. This may be attributed to the fact that real-world data frequently contains substantial noise, thereby hindering the full exploitation of the benefits of deep neural networks. This finding is consistent with the results of Simulation 4, where BR tends to exhibit good performance in noisy datasets, as demonstrated in subfigure (c) of the final column in Figure \ref{fig_cdplot_syn}. However, the BR method typically requires a large number of neurons and training epochs, resulting in training time that is comparable to, or even exceeds, that of deep neural networks. 3) Although RSNN maintains the shortest training time among all methods, its generalization performance is relatively poor, ranking third from the bottom in posthoc statistical test. This could be attributed to the complex data distribution relationships encountered in real-world data, which frequently lack clear functional correlations and are accompanied by considerable noise. The architecture of RSNN is not capable of accurately capturing such complex data relationships. This also aligns with the results from Simulation 3, where RSNN performs poorly on data generated by $f_2$  with high-level noise, as shown in subfigures (b) and (c) in the bottom row of Figure \ref{fig_cdplot_syn}. 4) Grounded in component-based sketching techniques, the proposed method innovatively integrates locality information from the frequency domain. This not only avoids the burdensome optimization process of deep neural networks, but also accurately captures the complex distribution relationships of real-world data. Consequently, it achieves excellent generalization performance in a training time comparable to, or even shorter than, that required for shallow networks.

All these simulations and real-world experiments demonstrate the effectiveness and efficiency of the proposed method when compared to the aforementioned approaches.


\section{Proofs}\label{Sec.Proofs}
The main tool in proving Theorem \ref{Theorem:approximation} is the following lemma, a dimension leverage technique, which can be easily deduced from \cite[Theorem 8.2]{petrushev1998approximation}. We refer readers to  \cite{fang2020learning,wang2020random} for detailed proofs.
\begin{lemma}\label{Lemma:petrushev}
Let $r>0$, $\mathcal L_n$ be a linear space defined on $\mathbb J$ of dimension $n$, and $\{\xi_\ell\}_{\ell=1}^N$ be a set of minimal $\mu$-energy points with $\mu>d$.
If $N\sim n^{d-1}$, for each $g\in W^r(L^2(\mathbb J))$, there exists a  $g_n\in \mathcal L_n$  and some $r'\geq r$ such that
\begin{equation}\label{approx-tool-ass}
    \|g-g_n\|_{L^2(\mathbb J)}\leq  C_1n^{-r'}\|g\|_{W^{r'}(L^2(\mathbb J))},
\end{equation}
then for any $f\in W^{r+(d-1)/2}(L^2(\mathbb B^d_{1/2}))$,
there exist functions $g_\ell$ ($\ell=1,\dots,N$) in $\mathcal{L}_n$ such that
\begin{equation}\label{Def.R}
    R_{n,N}(x):=\sum_{\ell=1}^N g_\ell(\xi_\ell\cdot x)
\end{equation}
satisfies
\begin{equation}\label{apprx-tool-res}
    \left\|f-R_{n,N}\right\|_{L^2(\mathbb B^d_{1/2})}\leq C_2n^{-r-(d-1)/2}\|f\|_{ W^{r+(d-1)/2}(L^2(\mathbb B^d_{1/2}))},
\end{equation}
 where $C_1$ and $C_2$ are constants depending only on $d,r$.
\end{lemma}

Set  $\mathcal L_n$ as
$$
     \mathcal H^{uni}_{J,n,m,\tau}:=\left\{\sum_{j=0}^{J-1}\sum_{k=1}^na_{jk}\tilde{\times}_{J,j,k,m,\tau}(t):a_{jk}\in\mathbb R\right\}
$$
with some fixed $0<\tau\leq n^{-4J-1}$ and $m\in\mathbb N$, where $\tilde{\times}_{J,j,k,m,\tau}$ is given in \eqref{basis-for-simpl}.
It follows from the above lemma that the approximation capability of the hypothesis space $\mathcal H_{j,n,N,m,\tau}$ defined in (\ref{Hypothesis-space-deep}) is determined by the approximation performance of $\mathcal H^{uni}_{J,n,\tau,\varepsilon}$. Therefore, it suffices to quantify the distance between $W^r(L^2(\mathbb J))$ and
$\mathcal H^{uni}_{J,n,\tau,\varepsilon}$.

For $s\in\mathbb N$ and a compact set  $\mathbb A$, denote by $\mathcal P_s(\mathbb A)$ the set of algebraic polynomials of degree at most $s$ defined on $\mathbb A$.
The well-known Jackson inequality  \cite{devore1993constructive} shows that for any $g\in W^r(L^2(\mathbb J))$, there exists a $p\in \mathcal P_s(\mathbb J)$ such that
\begin{equation}\label{Jackson-uni}
    \|g-p\|_{L^2(\mathbb J)}\leq C_3 s^{-r}\|g\|_{W^r(L^2(\mathbb J))}, \ \mbox{and}\ \|p\|_{W^r(L^2(\mathbb J))}\leq \|g\|_{W^r(L^2(\mathbb J))},
\end{equation}
where $C_3$ is a constant depending only on $r$.
Let $0\leq j\leq J-1$, $1\leq k\leq n$. For any $p\in \mathcal P_s(\mathbb J)$ with $s\in\mathbb N$, define
\begin{equation}\label{Taylor-poly}
    \psi_{j,t_{k-1},p}(t):=p( t_{k-1})+\frac{p'(t_{k-1})}{1!}
            (t-t_{k-1})+\cdots+\frac{p^{(j)}(t_{k-1})}{j!}( t-t_{k-1})^j
\end{equation}
as the Taylor polynomial approximation  of the polynomial $p$, where  we use the denotation $0!=1$.
It follows from the Taylor's formula that
\begin{equation}\label{Taylor formula1}
           p(t)=\psi_{j,t_{k-1},p}(t)+\frac{1}{j!}\int_{t_{k-1}}^tp^{(j+1)}(u)(t-u)^{j}du.
\end{equation}
Define further
\begin{equation}\label{N-one dim}
       N_{n,j}(t):=\sum_{k=1}^{n}\psi_{j,t_{k-1},p}(t)T_{\tau,t_{k-1},t_{k}}(t).
\end{equation}
We then have the following lemma that describes the approximation error of $N_{n,j}(t)$.
\begin{lemma}\label{Lemma:distance}
Let $n,s\in\mathbb N$ satisfy $n\sim s$. If  $0<\tau\leq n^{-4r-5}$, then for any $p\in\mathcal P_s(\mathbb J)$, there holds
\begin{eqnarray*}
        \|p-N_{n,j}\|_{L^2(\mathbb J)}^2  \leq C_4n^{-2j-2}\|p\|^2_{W^{j+1}(L^2(\mathbb J))},
\end{eqnarray*}
where $C_4$ is a constant depending only on $d,j$.
\end{lemma}

\begin{IEEEproof} Denote $A_k=[t_{k-1},t_{k})$ for $k=1,\dots,n-1$ and
$A_{n}=[t_{n-1},t_{n}]$. Then
\begin{eqnarray}\label{error decomposition}
   &&\|p-N_{n,j}\|_{L^2(\mathbb J)}^2
   =\int_{\mathbb
   J}\left|p(t)-\sum_{k=1}^{n}\psi_{j,t_{k-1},p}(t)T_{\tau,t_{k-1},t_{k}}(t)\right|^2dt\nonumber\\
   &\leq&
   2\sum_{k'=1}^{n}\int_{A_{k'}}|p(t)-\psi_{j,t_{k'-1},p}(t)T_{\tau,t_{k'-1},t_{k'}}(t)|^2dt \nonumber\\
   &+&
   2\sum_{k'=1}^{n}\int_{A_{k'}} \left|\sum_{k\neq
   k'}\psi_{j,t_{k-1},p}(t)T_{\tau,t_{k-1},t_{k}}(t)\right|^2dt\nonumber\\
   &=:&
   \mathcal B_1+\mathcal B_2.
\end{eqnarray}
It follows from (\ref{Deteailed trapezoid}),
 (\ref{Taylor formula1}),  and H\"{o}lder's inequality that
\begin{eqnarray*}
     \mathcal B_1&=&
     2\sum_{k'=1}^{n}\int_{A_{k'}}|p(t)-\psi_{j,t_{k'-1},p}(t)|^2dt\\
      &=&
     2\sum_{k'=1}^{n}\int_{A_{k'}}\left|\frac{1}{j!}\int_{t_{k'-1}}^tp^{(j+1)}(u)(t-u)^{j}du\right|^2dt \nonumber\\
     &\leq&
     \frac{2}{(j!)^2}\sum_{k'=1}^{n}\int_{A_{k'}}(t-t_{k'-1})^{2j+1} \int_{t_{k'-1}}^t|p^{(j+1)}(u)|^2dudt\\
     &\leq&
     \frac2{(j!)^2n^{2j+1}}\sum_{k'=1}^{n}\int_{A_{k'}}\int_{t_{k'-1}}^{t_{k'}}|p^{(j+1)}(u)|^2dudt \nonumber\\
     &=&
     \frac2{(j!)^2n^{2j+2}}\sum_{k'=1}^{n}\int_{t_{k'-1}}^{t_{k'}}|p^{(j+1)}(u)|^2du\\
     &=&
     \frac2{(j!)^2n^{2j+2}}\|p^{(j+1)}\|_{L^2(\mathbb J)}^2.
\end{eqnarray*}
This implies
\begin{equation}\label{Bound B1}
          \mathcal B_1 \leq \frac{2}{(j!)^2n^{2j+1}}\|p^{(j+1)}\|_{L^2(\mathbb J)}^2.
\end{equation}
Furthermore, (\ref{Deteailed trapezoid}) and $\tau\leq 1/n$ yields
\begin{eqnarray}\label{intial estimate for B2}
   &&\mathcal B_2
    \leq
   4\sum_{k'=2}^{n-1}\int_{A_{k'}}
   |\psi_{j,t_{k'-2},p}(t)T_{\tau,t_{k'-2},t_{k'-1}}(t)|^2dt \nonumber\\
   &+&4\sum_{k'=2}^{n-1}\int_{A_{k'}}|\psi_{j,t_{k'},p}(t)T_{\tau,t_{k'},t_{k'+1}}(t)|^2dt  \nonumber\\
   &+&2\int_{A_1}|\psi_{j,t_1,p}(t)T_{\tau,t_{1},t_{2}}(t)|
   ^2 dt\nonumber\\
   &+&2\int_{A_{n}}|\psi_{j,t_{n-2},p}(t)T_{\tau,t_{n-2},t_{n-1}}(t)|
   ^2dt\nonumber\\
   &\leq&
   4\sum_{k'=2}^{n-1}\int_{t_{k'-1-\tau}-\tau}^{t_{k'-1}}|\psi_{j,t_{k'-2},p}(t)|^2dt
   +
   4\sum_{k'=2}^{n-1}\int_{t_{k'}}^{t_{k'}+\tau}|\psi_{j,t_{k'},p}(t)|^2dt
   \nonumber\\
   &+&
   2\int_{t_1}^{t_1+\tau}|\psi_{j,t_0,p}(t)|
   ^2 dt
   +2\int_{t_{n}-\tau}^{t_{n}}|\psi_{j,t_{n-1},p}(t)|
   ^2dt.
\end{eqnarray}
For $t\in[t_{k-1}-\tau,t_{k+1}+\tau]$, it follows from the well-known Markov inequality and
 Nikolskii inequality for algebraic polynomials
\cite{borwein2012polynomials} that
 \begin{eqnarray*}
    &&\|\psi_{j,t_k,p}\|_{L^\infty(\mathbb J)}
    \leq
   \sum_{j'=0}^{j}  \frac{\|p^{(j')}\|_{L^\infty(\mathbb J)}}{j'!}(\tau+1/n)^j \\
   &\leq &
   C_5\sum_{j'=0}^{j}  \frac{s^{2j'}\|p\|_{L^\infty(\mathbb J)}}{j'!}(\tau+1/n)^j\\
   &\leq&
   C_6\sum_{j'=0}^{j}  \frac{s^{2j'+1}\|p\|_{L^2(\mathbb J)}}{j'!}(\tau+1/n)^{j'}\\
   &\leq&
   C_62^{(d-1)/2}\sum_{j'=0}^{j}  \frac{s^{2j'+1}\|p\|_{L^2(\mathbb J)}}{j'!}(\tau+1/n)^{j'},
\end{eqnarray*}
where $C_5$ and $C_6$ are absolute constants.
Plugging the above estimate into (\ref{intial estimate for B2}), we get from $n\sim s$ and $0<\tau\leq n^{-4j-5}$ that
\begin{equation}\label{Bound B2}
   \mathcal B_2
   \leq
   C_6^22^{d+1} n   \tau\|p\|^2_{L^2(\mathbb J)}\left(\sum_{j'=0}^{j}  \frac{s^{2j'+1}2^{j'}}{j'!n^{j'}}\right)^2 \leq C_4  n^{-2j-2}\|p\|^2_{L^2(\mathbb J)},
\end{equation}
where $C_4$ is a constant depending only on $j$ and $d$.
 Inserting (\ref{Bound B1}) and
(\ref{Bound B2}) into (\ref{error decomposition}), we  have
\begin{eqnarray*}
        \|p-N_n\|_{L^2(\mathbb J)}^2  \leq
    C_4 n^{-2j-2}\|p\|^2_{W^{j+1}(L^2(\mathbb J))}.
\end{eqnarray*}
This completes the proof of Lemma \ref{Lemma:distance}.
 \end{IEEEproof}

From Lemma \ref{Lemma:distance} and Proposition \ref{Prop:product-component-bound}, we can derive the following lemma.

\begin{lemma}\label{Lemma:distance-univariate}
Let  $r>0$, $J,n,s\in\mathbb N$ satisfy $n\sim s$. If  $J\geq r$, $0<\tau\leq n^{-4J-1}$, and $m \sim  \log n $, then for any $p\in\mathcal P_s(\mathbb J)$, there is a $g_n\in \mathcal H^{uni}_{J,n,m,\tau}$ such that
$$
    \|p-g_n\|_{L^2(\mathbb J)}
    \leq  C_8\left(n^{-J}\|p\|_{W^{r}(L^2(\mathbb J))}+n^{-J}\|p\|_{W^{r}(L^2(\mathbb J))}\right),
$$
where $C_8$ is a constant depending only on $d, r$.
\end{lemma}

\begin{IEEEproof}
For any $j\leq J-1$, it follows from
  (\ref{N-one dim}) and Proposition \ref{Prop:product-component-bound} that  there exists an element $N_{J,j n, m,\tau}\in\mathcal H^{uni}_{J,n,\tau,\varepsilon}$ defined by
$$
    N_{J,j, n, m,\tau}(t):=\sum_{j=0}^{J-1}\sum_{k=0}^{n } a_{j,k}
    PG_{J,m}(T_{\tau,t_{k-1},t_{k}}(t),\overbrace{t,\dots,t}^{j},\overbrace{1,\dots,1}^{J-1-j})
$$
such that
$$
    |N_{J,j, n, m,\tau}(t)-N_{n,j}(t)|\leq c_12^{-2m}.
$$
This together with Lemma \ref{Lemma:distance} with $j=J-1$ yields
\begin{eqnarray*}
    &&\|N_{J,j, n, m,\tau}-p\|_{L^2(\mathbb J)}^2\leq    C_4 n^{-2J}\|p\|^2_{W^{J}(L^2(\mathbb J))}+c_1^22^{-4m}.
\end{eqnarray*}
Setting  $m\sim  \log n$ such that
$2^{-4m}\leq n^{-2r}\|p\|^2_{W^{r}(L^2(\mathbb J))}$, and consequently,
$$
    \|N_{J,j, n, m,\tau}-p\|_{L^2(\mathbb J)}
    \leq  C_8\left(n^{-J}\|p\|_{W^{r}(L^2(\mathbb J))}+n^{-J}\|p\|_{W^{r}(L^2(\mathbb J))}\right),
$$
where $C_8$ is a constant depending only on $r, d$. This completes the proof of Lemma \ref{Lemma:distance-univariate}.
\end{IEEEproof}

With the above foundations, we can prove Theorem \ref{Theorem:approximation} as follows.

\begin{IEEEproof}[Proof of Theorem \ref{Theorem:approximation}]
Due to Lemma \ref{Lemma:distance-univariate} and  (\ref{Jackson-uni}) with $s\sim n$ and $r\leq J$, there exists a $g_n\in \mathcal H^{uni}_{J,n,m,\tau}$ such that
\begin{eqnarray*}
   &&\|g-g_n\|_{L^2(\mathbb J)}
   \leq
   \|g-p\|_{L^2(\mathbb J)}+\|p-g_n\|_{L^2(\mathbb J)}\\
   &\leq&
   C_3n^{-r}\|g\|_{W^r(L^2(\mathbb J))}+C_3C_4n^{-J}\|g\|_{W^J(L^2(\mathbb J))}.
\end{eqnarray*}
This verifies (\ref{approx-tool-ass}).
Then it follows from Lemma \ref{Lemma:petrushev}, by setting $r'$ to $r$ and $J$ in separate cases, that for any $f\in W^{r+(d-1)/2}(L^2(\mathbb B^d_{1/2}))$, we have
$$
   \mbox{dist}(f,\mathcal H_{J,n,N,m,\tau})\leq Cn^{-r-(d-1)/2}\|f\|_{ W^{r+(d-1)/2}(L^2(\mathbb B^d_{1/2}))},
$$
where $C$ is a constant depending only on $r, d$. The lower bound can be deduced by the well-known width theory \cite{pinkus2012n}. We refer the readers to \cite[Theorem 1]{maiorov1999degree} for a detailed proof.
 This completes the proof of Theorem \ref{Theorem:approximation}.
\end{IEEEproof}

To prove Theorem \ref{Theorem:gene-deep}, we need the following lemma \cite[Theorem 11.3]{gyorfi2002distribution}.

\begin{lemma}\label{Lemma:ORACLE}
Let $H_u$ be a  $u$-dimensional linear space. Define the estimate
$f_{D,u}$ by
$$
         f_{D,u}:=\pi_M f^*_{D,u}\quad\mbox{where}\quad
         f^*_{D,u}=\arg\min_{f\in
         H_u}\frac1{|D|}\sum_{i=1}^{|D|}|f(x_i)-y_i|^2.
$$
Then
$$
          \mathbf E\left\{\|f_{D,u}-f_\rho\|_\rho^2\right\}\leq
          \tilde{C}M^2\frac{u\log |D|}{|D|}+8\inf_{f\in H_u}\|f_\rho-f\|_\rho^2,
$$
for some universal constant $\tilde{C}$.
\end{lemma}

We then proceed the proof of Theorem \ref{Theorem:gene-deep} as follows.

\begin{proof}[Proof of Theorem \ref{Theorem:gene-deep}]
For the upper bound, noting that the hypothesis space $  \mathcal{H}_{J,n,N,m,\tau}$ is a linear space of dimension $JnN$ for fixed $m$ and $\tau$, it suffices to derive the approximation error $\inf_{f\in \mathcal{H}_{J,n,N,m,\tau}}\|f_\rho-f\|_\rho^2$. Due to the definition of $D_{\rho_X}$ and Theorem \ref{Theorem:approximation},  for $f_\rho\in W^{r+(d-1)/2}(L^2(\mathbb B^d_{1/2}))$, there holds
\begin{eqnarray*}
           &&\inf_{f\in \mathcal{H}_{J,n,N,m,\tau}}\|f_\rho-f\|_\rho^2
          \leq D_{\rho_X}^2 \inf_{f\in \mathcal{H}_{J,n,N,m,\tau}}\|f_\rho-f\|_{L^2(\mathbb B_{1/2}^d)}^2\\
          &\leq&
          C^2D_{\rho_X}^2n^{-2r-d+1}\|f_\rho\|^2_{ W^{r+(d-1)/2}(L^2(\mathbb B^d_{1/2}))},
\end{eqnarray*}
where $J,n,N\in \mathbb N$ satisfy $N\sim n^{d-1}$ and $J\geq r$,
$0<\tau\leq n^{-4J-1}$, and $m\sim \log n$.
Therefore, Lemma \ref{Lemma:ORACLE} yields
\begin{eqnarray*}
          \mathbf E \left\{\|f_{D,J,n,N,m,\tau}-f_\rho\|_\rho^2\right\}
          &\leq&
          \tilde{C}M^2\frac{JnN\log {|D|}}{|D|}\\
          &+&
          8 C^2D_{\rho_X}^2n^{-2r-d+1}\|f_\rho\|^2_{ W^{r+(d-1)/2}(L^2(\mathbb B^d_{1/2}))}.
\end{eqnarray*}
Noting further $n\sim |D|^{1/(2r+2d-1)}$ and $N\sim n^{d-1}\sim |D|^{(d-1)/(2r+2d-1)}$, we obtain
$$
 \mathbf E \left\{\|f_{D,J,n,N,m,\tau}-f_\rho\|_\rho^2\right\}
 \leq \tilde{C}_1 (J+D_{\rho_X}^2) |D|^{ \frac{-2r-d+1}{2r+2d-1}}\log |D|,
$$
where $\tilde{C}_1$ is a constant depending only on $d,M,f_\rho$, and $r$.
 The lower bound   can be easily deduced
  from \cite[Theorem 3.2]{gyorfi2002distribution} or \cite{maiorov2006approximation}. This completes the proof of Theorem \ref{Theorem:gene-deep}.
\end{proof}

To prove Theorem \ref{THEOREM: adaption selection}, we need the
following lemma, which can be found in \cite[Proposition 11]{caponnetto2010cross}.

\begin{lemma}\label{Lemma: cross-validation}
Let $\{\eta_i\}_{i=1}^{N'}$ be a set of real valued i.i.d. random
variables with mean $\mu$, $|\eta_i|\leq B$, and $
\mathbf E[(\eta_i-\mu)^2]\leq\sigma^2$, for all $i\in\{1,2,\dots,N'\}$. Then,
for arbitrary $a>0$, $\epsilon>0$, there hold
$$
        \mathbf P\left[\frac1{N'}\sum_{i=1}^{N'}\eta_i-\mu\geq
        a\sigma^2+\epsilon\right]\leq e^{-6N'a\epsilon}(3+4aB),
$$
and
$$
        \mathbf P\left[\mu-\frac1{N'}\sum_{i=1}^{N'}\eta_i\geq
        a\sigma^2+\epsilon\right]\leq e^{-6N'a\epsilon}(3+4aB).
$$
\end{lemma}

\begin{proof}[Proof of Theorem \ref{THEOREM: adaption selection}]
Let
$$
       n^*=\arg\min_{n\in\Xi}\int_Z(f_{D,J,n,N,m,\tau}(x)-y)^2d\rho
$$
with $N^*\sim (n^*)^{d-1}$. Due to \eqref{equality},
we get
$$
       n^*=\arg\min_{n\in\Xi}\|f_{D,J,n,N,m,\tau }-f_\rho\|^2_\rho
$$
and  $N^*\sim (n^*)^{d-1}$. According to the
definition of $f_{\tilde{D}_1,J,n^*,N^*,\tau,\varepsilon}$ and Theorem
\ref{Theorem:gene-deep},  we have
\begin{eqnarray}\label{cv0}
                     \mathbf E(\|f_{\tilde{D}_1,J,n^*,N^*,m,\tau}-f_\rho\|_\rho^2)
                     \leq
                    \tilde{C}_2  (J+D_{\rho_X}^2) |D|^{\frac{-2r-d+1}{2r+2d-1}}\log |D|,
\end{eqnarray}
where $\tilde{C}_2$ is a constant depending only on $M,d,r$, and $f_\rho$.
For $(x_i,y_i)\in \tilde{D}_2$, let us define the random variables
$$
         \eta_i^n=(f_{\tilde{D}_1,J,n,N,m,\tau}(x_i)-y_i)^2-(f_\rho(x_i)-y_i)^2
$$
with $N\sim n^{d-1}$. Clearly, $|\eta_i^n|\leq 8M^2$ almost surely,
$$
      \mathbf E[\eta_i^n]=\|f_{\tilde{D}_1,J,n,N,m,\tau}-
      f_\rho\|_\rho^2,
$$
and
\begin{eqnarray*}
     &&\mathbf E[(\eta_i^n)^2]=\int_Z(f_{\tilde{D}_1,J,n,N,m,\tau}(x)
    -f_\rho(x))^2\\
    &\times&
    (f_{\tilde{D}_1,J,n,N,m,\tau}(x)+f_\rho(x)-2y)^2d\rho\leq
    16M^2
           \mathbf E[\eta_i^n].
\end{eqnarray*}
Hence, using Lemma \ref{Lemma: cross-validation} with $a=1$,
$\eta_i=\eta_i^n$, $\mu= \mathbf  E[\eta_i^n]$, $B=8M^2$, and $\sigma^2\leq
16M^2\mu$, we obtain for all $n\in\Xi$, with probability greater
than
$$
          1-|\Xi|\exp\{-C|\tilde{D}_2|\epsilon\}=1-\exp\left\{-C|\tilde{D}_2|\epsilon+(d-1)\log |D|\right\},
$$
there holds
\begin{equation}\label{cv1}
       \frac1{|\tilde{D}_2|}\sum_{i=1}^{|\tilde{D}_2|}\eta_i^n\leq C \mathbf  E[\eta_i^n]+
       \epsilon
\end{equation}
and
\begin{equation}\label{cv2}
         \mathbf  E[\eta_i^n]\leq
          C\frac1{|\tilde{D}_2|}\sum_{i=1}^{|\tilde{D}_2|}\eta_i^n+\epsilon.
\end{equation}
Therefore, for arbitrary   $\epsilon\geq
\mathbf E\left[\|f_{\tilde{D}_1,J,n^*,N^*,m,\tau}- f_\rho\|_\rho^2\right]$, with confidence
at least
$$
         1-5\exp\left\{-C|D|\epsilon\right\},
$$
there holds
\begin{eqnarray*}
      &&  \mathbf E\left[\|f_{\tilde{D}_1,J,\hat{n},\hat{N},m,\tau}-
      f_\rho\|_\rho^2\right]
      =
          \mathbf E\left[
        E[\xi_i^{\hat{n}}]\right]\\
      &\leq&
       \mathbf  E\left[\frac{C}{m_2}\sum_{i=1}^{m_2}
      \left(\left(f_{\tilde{D}_1,J,\hat{n},\hat{N},m,\tau}(x_i)-y_i\right)^2\right.\right.\\
      &-&\left.\left.
      \left(f_\rho(x_i)-y_i\right)^2\right)\right]+\epsilon\\
           &\leq&
              \mathbf E\left[\frac{C}{|\tilde{D}_2|}\sum_{i=1}^{|\tilde{D}_2|}\left(\left(f_{\tilde{D}_1,J,n^*,N^*,m,\tau}(x_i)-y_i\right)^2\right.\right.\\
            &-&\left.\left.\left(f_\rho(x_i)-y_i\right)^2\right)\right]+\epsilon
           \leq
           C \mathbf  E[\eta_i^{n^*}]+2\epsilon\\
           &=&
             \mathbf E\left[\|f_{\tilde{D}_1,J,n^*,N^*,m,\tau}- f_\rho\|_\rho^2\right]+2\epsilon
            \leq 3\epsilon,
\end{eqnarray*}
where the first inequality is deduced from (\ref{cv2}), the second
inequality is according to the definition of $\hat{n}$, and the third one
is based on (\ref{cv1}). Therefore
\begin{eqnarray*}
        && \mathbf E\left[\|f_{\tilde{D}_1,J,\hat{n},\hat{N},m,\tau}- f_\rho\|_\rho^2\right]\\
         &=&
       \int_0^\infty\mathbf
       P\left\{ \mathbf  E \left[\|f_{\tilde{D}_1,J,\hat{n},\hat{N},m,\tau}- f_\rho\|_\rho^2\right]>\epsilon\right\} d\epsilon\\
            &\leq&
             \mathbf  E\left[\|f_{\tilde{D}_1,J,n^*,N^*,m,\tau}- f_\rho\|_\rho^2\right]
            +
            5\int_0^\infty \exp\{-C_1|D|\epsilon\}d\epsilon\\
            &\leq&
            \tilde{C}_3  (J+D_{\rho_X}^2) |D|^{-\frac{-2r-d+1}{2r+2d-1}}\log |D|,
\end{eqnarray*}
where the last inequality is
based on (\ref{cv0}).
 This finishes the proof of Theorem \ref{THEOREM: adaption selection}.
\end{proof}


\bibliographystyle{IEEEtran}
\bibliography{deep-net}

\begin{thebibliography}{10}
\providecommand{\url}[1]{#1}
\csname url@samestyle\endcsname
\providecommand{\newblock}{\relax}
\providecommand{\bibinfo}[2]{#2}
\providecommand{\BIBentrySTDinterwordspacing}{\spaceskip=0pt\relax}
\providecommand{\BIBentryALTinterwordstretchfactor}{4}
\providecommand{\BIBentryALTinterwordspacing}{\spaceskip=\fontdimen2\font plus
\BIBentryALTinterwordstretchfactor\fontdimen3\font minus
  \fontdimen4\font\relax}
\providecommand{\BIBforeignlanguage}[2]{{%
\expandafter\ifx\csname l@#1\endcsname\relax
\typeout{** WARNING: IEEEtran.bst: No hyphenation pattern has been}%
\typeout{** loaded for the language `#1'. Using the pattern for}%
\typeout{** the default language instead.}%
\else
\language=\csname l@#1\endcsname
\fi
#2}}
\providecommand{\BIBdecl}{\relax}
\BIBdecl

\bibitem{lecun2015deep}
Y.~LeCun, Y.~Bengio, and G.~Hinton, ``Deep learning,'' \emph{Nature}, vol. 521,
  no. 7553, pp. 436--444, 2015.

\bibitem{krizhevsky2017imagenet}
A.~Krizhevsky, I.~Sutskever, and G.~E. Hinton, ``Imagenet classification with
  deep convolutional neural networks,'' \emph{Communications of the ACM},
  vol.~60, no.~6, pp. 84--90, 2017.

\bibitem{silver2016mastering}
D.~Silver, A.~Huang, C.~J. Maddison, A.~Guez, L.~Sifre, G.~Van Den~Driessche,
  J.~Schrittwieser, I.~Antonoglou, V.~Panneershelvam, M.~Lanctot \emph{et~al.},
  ``Mastering the game of go with deep neural networks and tree search,''
  \emph{nature}, vol. 529, no. 7587, pp. 484--489, 2016.

\bibitem{kamath2019deep}
U.~Kamath, J.~Liu, and J.~Whitaker, \emph{Deep learning for NLP and speech
  recognition}.\hskip 1em plus 0.5em minus 0.4em\relax Springer, 2019, vol.~84.

\bibitem{zhou2020universality}
D.-X. Zhou, ``Universality of deep convolutional neural networks,''
  \emph{Applied and Computational Harmonic Analysis}, vol.~48, no.~2, pp.
  787--794, 2020.

\bibitem{kingma2014adam}
D.~P. Kingma and J.~Ba, ``Adam: A method for stochastic optimization,''
  \emph{arXiv preprint arXiv:1412.6980}, 2014.

\bibitem{choudhary2022recent}
K.~Choudhary, B.~DeCost, C.~Chen, A.~Jain, F.~Tavazza, R.~Cohn, C.~W. Park,
  A.~Choudhary, A.~Agrawal, S.~J. Billinge \emph{et~al.}, ``Recent advances and
  applications of deep learning methods in materials science,'' \emph{npj
  Computational Materials}, vol.~8, no.~1, p.~59, 2022.

\bibitem{chui2020realization}
C.~K. Chui, S.-B. Lin, B.~Zhang, and D.-X. Zhou, ``Realization of spatial
  sparseness by deep relu nets with massive data,'' \emph{IEEE Transactions on
  Neural Networks and Learning Systems}, 2020.

\bibitem{lin2018generalization}
S.-B. Lin, ``Generalization and expressivity for deep nets,'' \emph{IEEE
  Transactions on Neural Networks and Learning Systems}, vol.~30, no.~5, pp.
  1392--1406, 2018.

\bibitem{schwab2019deep}
C.~Schwab and J.~Zech, ``Deep learning in high dimension: Neural network
  expression rates for generalized polynomial chaos expansions in uq,''
  \emph{Analysis and Applications}, vol.~17, no.~01, pp. 19--55, 2019.

\bibitem{chui2019deep}
C.~K. Chui, S.-B. Lin, and D.-X. Zhou, ``Deep neural networks for
  rotation-invariance approximation and learning,'' \emph{Analysis and
  Applications}, vol.~17, no.~05, pp. 737--772, 2019.

\bibitem{shaham2018provable}
U.~Shaham, A.~Cloninger, and R.~R. Coifman, ``Provable approximation properties
  for deep neural networks,'' \emph{Applied and Computational Harmonic
  Analysis}, vol.~44, no.~3, pp. 537--557, 2018.

\bibitem{kohler2016nonparametric}
M.~Kohler and A.~Krzy{\.z}ak, ``Nonparametric regression based on hierarchical
  interaction models,'' \emph{IEEE Transactions on Information Theory},
  vol.~63, no.~3, pp. 1620--1630, 2016.

\bibitem{imaizumi2019deep}
M.~Imaizumi and K.~Fukumizu, ``Deep neural networks learn non-smooth functions
  effectively,'' in \emph{The 22nd International Conference on Artificial
  Intelligence and Statistics}.\hskip 1em plus 0.5em minus 0.4em\relax PMLR,
  2019, pp. 869--878.

\bibitem{schmidt2020nonparametric}
J.~Schmidt-Hieber, ``Nonparametric regression using deep neural networks with
  relu activation function,'' \emph{Annals of Statistics}, vol.~48, no.~4, pp.
  1875--1897, 2020.

\bibitem{han2020depth}
Z.~Han, S.~Yu, S.-B. Lin, and D.-X. Zhou, ``Depth selection for deep relu nets
  in feature extraction and generalization,'' \emph{IEEE Transactions on
  Pattern Analysis and Machine Intelligence}, vol.~44, no.~4, pp. 1853--1868,
  2022.

\bibitem{han2023learning}
Z.~Han, B.~Liu, S.-B. Lin, and D.-X. Zhou, ``Deep convolutional neural networks
  with zero-padding: Feature extraction and learning,'' \emph{arXiv preprint
  arXiv:2307.16203}, 2023.

\bibitem{anthony1999neural}
M.~Anthony and P.~L. Bartlett, \emph{Neural Network Learning: Theoretical
  Foundations}.\hskip 1em plus 0.5em minus 0.4em\relax Cambridge University
  Press, 1999, vol.~9.

\bibitem{shawe2004kernel}
J.~Shawe-Taylor, N.~Cristianini \emph{et~al.}, \emph{Kernel methods for pattern
  analysis}.\hskip 1em plus 0.5em minus 0.4em\relax Cambridge University Press,
  2004.

\bibitem{allen2019convergence}
Z.~Allen-Zhu, Y.~Li, and Z.~Song, ``A convergence theory for deep learning via
  over-parameterization,'' in \emph{International Conference on Machine
  Learning}.\hskip 1em plus 0.5em minus 0.4em\relax PMLR, 2019, pp. 242--252.

\bibitem{zeng2019global}
J.~Zeng, T.~T.-K. Lau, S.~Lin, and Y.~Yao, ``Global convergence of block
  coordinate descent in deep learning,'' in \emph{International conference on
  machine learning}.\hskip 1em plus 0.5em minus 0.4em\relax PMLR, 2019, pp.
  7313--7323.

\bibitem{zeng2021admm}
J.~Zeng, S.-B. Lin, Y.~Yao, and D.-X. Zhou, ``On admm in deep learning:
  Convergence and saturation-avoidance,'' \emph{Journal of Machine Learning
  Research}, vol.~22, no. 199, pp. 1--67, 2021.

\bibitem{sun2020global}
R.~Sun, D.~Li, S.~Liang, T.~Ding, and R.~Srikant, ``The global landscape of
  neural networks: An overview,'' \emph{IEEE Signal Processing Magazine},
  vol.~37, no.~5, pp. 95--108, 2020.

\bibitem{gyorfi2002distribution}
L.~Gy{\"o}rfi, M.~Kohler, A.~Krzy{\.z}ak, and H.~Walk, \emph{A
  Distribution-free Theory of Nonparametric Regression}.\hskip 1em plus 0.5em
  minus 0.4em\relax Springer, 2002, vol.~1.

\bibitem{li2022benefit}
D.~Li, T.~Ding, and R.~Sun, ``On the benefit of width for neural networks:
  disappearance of basins,'' \emph{SIAM Journal on Optimization}, vol.~32,
  no.~3, pp. 1728--1758, 2022.

\bibitem{zhang2021understanding}
C.~Zhang, S.~Bengio, M.~Hardt, B.~Recht, and O.~Vinyals, ``Understanding deep
  learning (still) requires rethinking generalization,'' \emph{Communications
  of the ACM}, vol.~64, no.~3, pp. 107--115, 2021.

\bibitem{cao2020generalization}
Y.~Cao and Q.~Gu, ``Generalization error bounds of gradient descent for
  learning over-parameterized deep relu networks,'' in \emph{Proceedings of the
  AAAI Conference on Artificial Intelligence}, vol.~34, no.~04, 2020, pp.
  3349--3356.

\bibitem{chen2020much}
Z.~Chen, Y.~Cao, D.~Zou, and Q.~Gu, ``How much over-parameterization is
  sufficient to learn deep relu networks?'' in \emph{International Conference
  on Learning Representations}, 2020.

\bibitem{lin2021generalization}
S.-B. Lin, Y.~Wang, and D.-X. Zhou, ``Generalization performance of empirical
  risk minimization on over-parameterized deep relu nets,'' \emph{IEEE
  Transations on Information Theory, Under Revision (arXiv:2111.14039)}, 2021.

\bibitem{zhou2024learning}
T.-Y. Zhou and X.~Huo, ``Learning ability of interpolating deep convolutional
  neural networks,'' \emph{Applied and Computational Harmonic Analysis},
  vol.~68, p. 101582, 2024.

\bibitem{herrmann2022constructive}
L.~Herrmann, J.~A. Opschoor, and C.~Schwab, ``Constructive deep relu neural
  network approximation,'' \emph{Journal of Scientific Computing}, vol.~90,
  no.~2, p.~75, 2022.

\bibitem{lin2018constructive}
S.~Lin, J.~Zeng, and X.~Zhang, ``Constructive neural network learning,''
  \emph{IEEE Transactions on Cybernetics}, vol.~49, no.~1, pp. 221--232, 2018.

\bibitem{liu2022construction}
X.~Liu, D.~Wang, and S.-B. Lin, ``Construction of deep relu nets for spatially
  sparse learning,'' \emph{IEEE Transactions on Neural Networks and Learning
  Systems}, 2022.

\bibitem{fang2020learning}
J.~Fang, S.~Lin, and Z.~Xu, ``Learning through deterministic assignment of
  hidden parameters,'' \emph{IEEE Transactions on Cybernetics}, vol.~50, no.~5,
  pp. 2321--2334, 2020.

\bibitem{wang2020random}
D.~Wang, J.~Zeng, and S.-B. Lin, ``Random sketching for neural networks with
  relu,'' \emph{IEEE Transactions on Neural Networks and Learning Systems},
  vol.~32, no.~2, pp. 748--762, 2020.

\bibitem{yarotsky2017error}
D.~Yarotsky, ``Error bounds for approximations with deep relu networks,''
  \emph{Neural Networks}, vol.~94, pp. 103--114, 2017.

\bibitem{petersen2018optimal}
P.~Petersen and F.~Voigtlaender, ``Optimal approximation of piecewise smooth
  functions using deep relu neural networks,'' \emph{Neural Networks}, vol.
  108, pp. 296--330, 2018.

\bibitem{petrushev1998approximation}
P.~P. Petrushev, ``Approximation by ridge functions and neural networks,''
  \emph{SIAM Journal on Mathematical Analysis}, vol.~30, no.~1, pp. 155--189,
  1998.

\bibitem{zhang2024classification}
Z.~Zhang, L.~Shi, and D.-X. Zhou, ``Classification with deep neural networks
  and logistic loss,'' \emph{Journal of Machine Learning Research}, vol.~25,
  no. 125, pp. 1--117, 2024.

\bibitem{lin2022universal}
S.-B. Lin, K.~Wang, Y.~Wang, and D.-X. Zhou, ``Universal consistency of deep
  convolutional neural networks,'' \emph{IEEE Transactions on Information
  Theory}, 2022.

\bibitem{guo2019realizing}
Z.-C. Guo, L.~Shi, and S.-B. Lin, ``Realizing data features by deep nets,''
  \emph{IEEE Transactions on Neural Networks and Learning Systems}, vol.~31,
  no.~10, pp. 4036--4048, 2019.

\bibitem{cucker2007learning}
F.~Cucker and D.~X. Zhou, \emph{Learning Theory: an Approximation Theory
  Viewpoint}.\hskip 1em plus 0.5em minus 0.4em\relax Cambridge University
  Press, 2007, vol.~24.

\bibitem{lin2023optimal}
S.-B. Lin, ``Optimal approximation and learning rates for deep convolutional
  neural networks,'' \emph{arXiv preprint arXiv:2308.03259}, 2023.

\bibitem{lin2017does}
H.~W. Lin, M.~Tegmark, and D.~Rolnick, ``Why does deep and cheap learning work
  so well?'' \emph{Journal of Statistical Physics}, vol. 168, no.~6, pp.
  1223--1247, 2017.

\bibitem{brauchart2015distributing}
J.~S. Brauchart and P.~J. Grabner, ``Distributing many points on spheres:
  minimal energy and designs,'' \emph{Journal of Complexity}, vol.~31, no.~3,
  pp. 293--326, 2015.

\bibitem{leopardi2006partition}
P.~Leopardi, ``A partition of the unit sphere into regions of equal area and
  small diameter,'' \emph{Electronic Transactions on Numerical Analysis},
  vol.~25, no.~12, pp. 309--327, 2006.

\bibitem{goodfellow2016deep}
I.~Goodfellow, Y.~Bengio, and A.~Courville, \emph{Deep Learning}.\hskip 1em
  plus 0.5em minus 0.4em\relax MIT Press, 2016.

\bibitem{bartlett2019nearly}
P.~L. Bartlett, N.~Harvey, C.~Liaw, and A.~Mehrabian, ``Nearly-tight
  vc-dimension and pseudodimension bounds for piecewise linear neural
  networks,'' \emph{Journal of Machine Learning Research}, vol.~20, no.~1, pp.
  2285--2301, 2019.

\bibitem{mendelson2003entropy}
S.~Mendelson and R.~Vershynin, ``Entropy and the combinatorial dimension,''
  \emph{Inventiones Mathematicae}, vol. 152, no.~1, pp. 37--55, 2003.

\bibitem{devore1993constructive}
R.~A. DeVore and G.~G. Lorentz, \emph{Constructive Approximation}.\hskip 1em
  plus 0.5em minus 0.4em\relax Springer Science \& Business Media, 1993, vol.
  303.

\bibitem{shi2011concentration}
L.~Shi, Y.-L. Feng, and D.-X. Zhou, ``Concentration estimates for learning with
  $\ell_1$-regularizer and data dependent hypothesis spaces,'' \emph{Applied
  and Computational Harmonic Analysis}, vol.~31, no.~2, pp. 286--302, 2011.

\bibitem{MathWorks2020}
B.~M, H.~M, and D.~H., \emph{Deep Learning Toolbox$^{\text{\tiny{TM}}}$ Getting
  Started Guide}.\hskip 1em plus 0.5em minus 0.4em\relax Natick, MA, USA: The
  MathWorks, 2020.

\bibitem{Ruder2016}
S.~Ruder, ``An overview of gradient descent optimization algorithms,''
  \emph{arXiv preprint arXiv:1609.04747}, 2016.

\bibitem{Demsar2006}
J.~Demsar, ``Statistical comparisons of classifiers over multiple data sets,''
  \emph{Journal of Machine Learning Research}, vol.~7, no.~1, pp. 1--30, JAN
  2006.

\bibitem{Yu2017}
Z.~Yu, Z.~Wang, J.~You, J.~Zhang, J.~Liu, H.~Wong, and G.~Han, ``A new kind of
  nonparametric test for statistical comparison of multiple classifiers over
  multiple datasets,'' \emph{IEEE Transactions on Cybernetics}, vol.~47,
  no.~12, pp. 4418--4431, DEC 2017.

\bibitem{borwein2012polynomials}
P.~Borwein and T.~Erd{\'e}lyi, \emph{Polynomials and Polynomial
  Inequalities}.\hskip 1em plus 0.5em minus 0.4em\relax Springer Science \&
  Business Media, 2012, vol. 161.

\bibitem{pinkus2012n}
A.~Pinkus, \emph{N-widths in Approximation Theory}.\hskip 1em plus 0.5em minus
  0.4em\relax Springer Science \& Business Media, 2012, vol.~7.

\bibitem{maiorov1999degree}
V.~Maiorov and J.~Ratsaby, ``On the degree of approximation by manifolds of
  finite pseudo-dimension,'' \emph{Constructive approximation}, vol.~15, no.~2,
  pp. 291--300, 1999.

\bibitem{maiorov2006approximation}
V.~Maiorov, ``Approximation by neural networks and learning theory,''
  \emph{Journal of Complexity}, vol.~22, no.~1, pp. 102--117, 2006.

\bibitem{caponnetto2010cross}
A.~Caponnetto and Y.~Yao, ``Cross-validation based adaptation for
  regularization operators in learning theory,'' \emph{Analysis and
  Applications}, vol.~8, no.~02, pp. 161--183, 2010.

\end{thebibliography}

\end{document}